\documentclass{article}
\usepackage{amsfonts,amsmath,amssymb,amsthm,bm}
\usepackage{algorithm}
\usepackage{algorithmic}
\usepackage{mathrsfs}
\usepackage[colorlinks,
            linkcolor=red,
            anchorcolor=blue,
            citecolor=blue
            ]{hyperref}
\usepackage{natbib}
\usepackage{tabularx}
\usepackage{multirow}
\usepackage{threeparttable}
\usepackage{makecell}
\usepackage{booktabs} 
\usepackage{fullpage}
\usepackage{appendix}
\usepackage{graphicx}
\usepackage{wrapfig}
\usepackage{subcaption}
\usepackage{caption}
\usepackage{appendix}

\newtheorem{lemma}{Lemma}
\newtheorem{theorem}{Theorem}

\newtheorem{example}{Example}
\newtheorem{corollary}[theorem]{Corollary}
\newtheorem{remark}{Remark}

\ifx\assumption\undefined
\newtheorem{assumption}{Assumption}
\fi

\newcommand{\rR}{{\mathbb{R}}}
\newcommand{\E}{{\mathbb{E}}}

\newcommand{\tx}{\mathrm{tr}\{ \varSigma \}}

\newcommand{\adv}{\mathrm{adv}}

\newcommand{\ood}{\mathrm{ood}}
\renewcommand\bm[1]{{#1}}
\renewcommand\parallel{\|}

\usepackage{xcolor}         

\newcommand\T{T} 
\newcommand\srisk{\mathcal{R}^{\operatorname{std}}} 
\newcommand\ib{\mathcal{B}^{\operatorname{id}}}
\newcommand\iv{\mathcal{V}^{\operatorname{id}}}
\newcommand\ob{\mathcal{B}^{\operatorname{ood}}}
\newcommand\ov{\mathcal{V}^{\operatorname{ood}}}

\newcommand{\trace}{\operatorname{tr}}

\newcommand{\inprobto}{\overset{\text{pr.}}{\to}}

\def\bbE{\mathbb{E}}

\def\bbP{\mathbb{P}}

\def\cC{\mathcal{C}}
\def\cD{\mathcal{D}}

\title{On the Benefits of Over-parameterization for \\ Out-of-Distribution Generalization}

\author{
Yifan Hao\thanks{The first two authors contributed equally.}~\thanks{The Hong Kong University of Science and Technology. Email: \texttt{yhaoah@connect.ust.hk}} \qquad
Yong Lin\footnotemark[1] \thanks{The Hong Kong University of Science and Technology. Email: \texttt{ylindf@connect.ust.hk}} \qquad
Difan Zou\thanks{The University of Hong Kong. Email: \texttt{dzou@cs.hku.hk}} \qquad
Tong Zhang\thanks{University of Illinois Urbana-Champaign. Email: \texttt{tongzhang@tongzhang-ml.org}}
}

\date{}

\begin{document}

\maketitle
\begin{abstract}%
In recent years, machine learning models have achieved success based on the independently and identically distributed (IID) assumption. However, this assumption can be easily violated in real-world applications, leading to the Out-of-Distribution (OOD) problem. Understanding how modern over-parameterized DNNs behave under non-trivial natural distributional shifts is essential, as current theoretical understanding is insufficient. Existing theoretical works often provide meaningless results for over-parameterized models in OOD scenarios or even contradict empirical findings.

To this end, we are investigating the performance of the over-parameterized model in terms of OOD generalization under the commonly-adopted ``benign overfitting'' conditions \citep{bartlett2020benign}. Our analysis focuses on a ReLU-based random feature model and examines non-trivial natural distributional shifts, where the benign overfitting estimators demonstrate a constant excess OOD loss, despite achieving zero excess in-distribution (ID) loss. We demonstrate that in this scenario, further increasing the model's parameterization can significantly reduce the OOD loss. Intuitively, the variance term of the ID testing loss usually remains minimal because of the orthogonal characteristics of the long-tail features in each sample. This implies that incorporating these features to over-fit training data noise generally doesn't notably raise the testing loss. However, in OOD situations, the variance component grows due to the distributional shift. Thankfully, the inherent shift is unrelated to individual $x$, maintaining the orthogonality of long-tail features despite this change. Expanding the hidden dimension can additionally improve this orthogonality by mapping the features into higher-dimensional spaces, thereby reducing the variance component.

We further show that model ensembles can also enhance the OOD testing loss, achieving a similar effect to increasing the model capacity. These results offer theoretical insights into the intriguing empirical phenomenon of significantly improved OOD generalization through model ensemble. We also provide supportive simulations which are consistent with theoretical results.  

\end{abstract}
\section{Introduction}
\label{sect:intro}

In recent years, machine learning with modern deep neural networks (DNN) architecture has achieved notable success and has found widespread applications in various domains like computer vision \citep{brown2020language}, natural language processing \citep{radford2021learning}, and autonomous driving \citep{yurtsever2020survey}. A common fundamental assumption of machine learning is  the identically and independently (IID) assumption, which assumes that the testing data are driven from the same distribution with the training data. However, IID assumption can easily fail in real-world application. For example, a autonomous driving car trained on the data collected from city roads could also be required to navigate in the country-side roads. This is also referred as the Out-of-Distribution (OOD) generalization problem. Existing empirical works show that the performance of machine learning models can drop significantly due to distributional shifts. 

Though the OOD problem is of vital importance, the theoretical understanding of how machine learning models (especially highly over-parameterized DNNs) perform under distributional shifts is mainly lacking. What's more, the prevailing generalization theory under distributional shifts even contradicts some crucial empirical observations. Particularly intriguing is that practitioners repeatedly report that enlarging the DNNs can consistently improve the OOD performance under non-trivial distributional shifts. However, none of the existing theories can explain this phenomenon. Specifically, \cite{ben2006analysis} presents a generalization bound of models in OOD scenarios, where the upper bound of OOD losses increases vacuously with the VC-dimension of the model. In contrast, \citet{wald2022malign,sagawa2020investigation, zhou2022sparse} offer contradictory analyses, showing that larger models may more easily rely on unstable features (also referred to as spurious features \citep{arjovsky2019invariant}), and such reliance could lead to model failure under distributional shifts. Another mysterious empirical phenomenon in OOD generalization is that model ensembles consistently improve the OOD performance \citep{wortsman2022robust, wortsman2022model, rame2022diverse, cha2021swad, arpit2022ensemble, tian2023trainable, kumar2022calibrated,lin2023spurious}, repeatedly pushing the State-of-the-Art (SOTA) performance of various large scale OOD benchmarks such as DomainBed \citep{gulrajani2020search} and ImageNet variants \citep{wortsman2022robust}.

In this paper, we investigate the impact of over-parameterization for OOD generalization by considering a natural shift. Let $x \in \mathbb{R}^p$ and $y \in \mathbb{R}$ denote the input and output. We want to learn a function $f$ to regress over $y$ by $f(x)$. Let $\cD$ denote the training distribution of $(x, y)$ where $x \sim \cD_x$ and $y=g(x) + \epsilon$  by some unknown non-linear function $g$ and random a Gaussian noise $\epsilon$. We consider a  distributional shift parameterized by $\delta$ as follows: 
\begin{align}
    \label{eqn:formulation}
    \mathcal{L}_{\ood}(f) = & \max_{\Delta \in {\Xi}_{\mathrm{ood}}}\mathbb{E}_{(x, \delta, y) \sim \cD(\Delta)} \left[\ell (f(x), y)\right], \\
    \mathrm{s.t. } \quad \cD(\Delta) = &  \{ (x, \delta, y)| x \sim \cD_x, \epsilon \sim \cD_\epsilon, \delta \sim \Delta, y = g(x + \delta)+\epsilon, \delta \mbox{ is independent of } (x, \epsilon)\}, \nonumber 
\end{align}
where $\Xi_{\mathrm{ood}}$ specifies the extent of allowed distributional shifts. The independence of $\delta$ from $x$ and $\epsilon$ is assumed, and further details and discussions can be found in Section~\ref{sect:performance_measure}. Notably, as elaborated in Section~\ref{sect:performance_measure}, we adopt a relatively mild constraint on $\Xi_{\mathrm{ood}}$ by allowing for a compatible scaling of $\delta$ in comparison to $x$. Specific, in this model, as we will show in the later part, when we achieve optimal prediction loss (no excess prediction loss) on the training domain, the excess OOD loss still remains at a constant level.  Consistent with the standard assumption in covariate shift \citep{ben2006analysis}, we assume that the true function generating $y$, denoted as $g(\cdot)$, remains unchanged in the testing domain.

We consider a the predictor defined as $f(x) = \phi(x^\top W) \theta$ where $\phi(x^\top W)$ is a random feature model (i.e., $W \in \mathbb{R}^{p \times m}$ is a random feature map) with an element-wise ReLU activation function (i.e., $\phi(a) = \mathrm{ReLU}(a) = \max\{0, a\}$, $\forall a \in \mathbb{R}$), and our focus is on investigating the behavior of "ridgeless" estimators of $\theta$ in a region where $m \gg n$. We consider the eigenvalues of $x$ follow the ``benign overfitting'' conditions \citep{bartlett2020benign}, where the over-parameterized models achieve good ID performance. Specifically, we observe that while the ID excess risk diminishes as the sample size $n$ grows, following the phenomenon of ``benign overfitting'', the OOD excess risk remains consistently high, at a constant level. Furthermore, we find that we can reduce such OOD excess risk by increasing the number of model parameters. Moreover, we demonstrate that constructing an ensemble of multiple independently initialized and trained models can also be effective in reducing the OOD risk within this scenario, which is consistent with the empirical findings \citep{wortsman2022model}. Intuitively, the variance term of the ID testing loss is typically small due to the orthogonal nature of the long-tail features in each sample, which means that fitting the training data noise with these features does not significantly increase the testing loss \citep{shamir2022implicit}. However, in OOD scenarios, the variance term is increased because of the distributional shift. Fortunately, the natural shift is independent of $x$, so the orthogonality of long-tail features is preserved under this shift. Increasing the hidden dimension can further enhance this orthogonality by projecting the features into higher-dimensional spaces, thus reducing the variance term. Additional simulation results can be found in Section~\ref{sec:re}.


Our main results can be summarized as follows:
\begin{itemize}
\item We offer a precise non-asymptotic analysis of ID excess risk and OOD excess risk, providing both upper and lower bounds, for a random feature model with ReLU activation. Within the benign overfitting regime, as the min-norm estimator is asymptotically optimal in ID situation, it behaves unsatisfactorily in OOD situations. In the aforementioned setting, the OOD excess risk exhibits non-trivial reduction when the number of model parameters increases.
\item Furthermore, we demonstrate that constructing an ensemble of multiple independently initialized and trained models can also effectively reduce the OOD excess risk, achieving similar effects as those seen with enlarging model capacity. This serves to explain the intriguing empirical findings that ensemble models improve OOD generalization.
\end{itemize}

Our theoretical result is distinct from existing theories for several reasons:
\begin{itemize}
    \item Our result differs from several existing theoretical viewpoints on over-parameterization for OOD generalization with conjecture that overparameterization may also lead to instability under distributional shifts \citep{ben2006analysis,sagawa2020investigation, zhou2022sparse, wald2022malign}. 
    \item The behavior of benign-overfitting estimators under natural shifts differs markedly from their behavior under adversarial attacks. While it has theoretically verified that increased over-parameterization exacerbates adversarial vulnerability ~\citep{ hao2024surprising}, our work demonstrates, in the context of natural shifts, overparameterization with benign overfitting can actually be advantageous for reducing OOD loss.
\end{itemize}

The paper is structured as follows: we review related works in Section~\ref{sec:liter}, present the model setting and performance measurement methods in Section~\ref{sec:pre}, showcase our main results in Section~\ref{sec:re}, provide sketches of proofs in Section~\ref{sec:pf}, and conclude with future discussions in Section~\ref{sec:conclu}.

\section{Related work}\label{sec:liter}

There exists a substantial body of work on implicit bias, benign overfitting, model ensemble, and distribution shifts. In this section, we review the most relevant works to ours.

\subsection{Learning Theory} 

\paragraph{Implicit bias and benign overfitting.} Several recent works have delved into the generalization capabilities of large overparameterized models, particularly in the context of fitting noisy data~\citep{neyshabur2014search,wyner2017explaining,belkin2018understand,belkin2019reconciling,liang2020just,zhang2021understanding,cao2023implicit}. 
These studies have uncovered that the implicit biases of various algorithms can contribute to their favorable generalization properties, prompting a deeper investigation into their successes~\citep{telgarsky2013margins,neyshabur2015path,keskar2016large,neyshabur2017exploring,wilson2017marginal}. 
\citet{soudry2018implicit} and \citet{ji2019implicit} investigated the implicit bias of gradient descent in classification tasks, while \citet{gunasekar2018characterizing} explored the implicit bias of various optimization methods in linear regression. Additionally, \citet{ji2018gradient} delved into the implicit bias of deep neural networks, and \citet{gunasekar2017implicit} and \citet{arora2019implicit} analyzed the implicit bias in matrix factorization problems.
When these implicit bias are accounted for, a series of works have emerged focusing on the phenomenon of ``benign overfitting'', in both regression problems~ \citep{bartlett2020benign,belkin2020two,muthukumar2020harmless,liang2020just, zou2021benign,zou2021benefits,shamir2022implicit,tsigler2020benign, simon2023more, hao2024surprising} and classification problems~\citep{chatterji2021finite, muthukumar2021classification, wang2021benign, wang2022binary, cao2022benign, chen2023benign}. Our work is partly inspired by the setup and analysis presented in \citet{bartlett2020benign}, \citet{tsigler2020benign} and \citet{hao2024surprising}. However, while \citet{bartlett2020benign} and \citet{tsigler2020benign} primarily explore the consistency of estimators in ``benign overfitting'' regime, and \citet{hao2024surprising} verified the adversarial sensitivity of such estimators, our work stands out as the first explicit exploration within the context of distribution shifts.

\paragraph{Model ensemble.} Our analysis demonstrates that an ensemble of several independently trained models can achieve similar improvement on OOD generalization performance with increased parameterization.
 Model ensemble has been a popular technique to enhances generalization performance, as documented in the existing literature ~\citep{hansen1990neural, krogh1994neural, perrone1995networks, opitz1999popular, dietterich2000ensemble, zhou2002ensembling, polikar2006ensemble, rokach2010ensemble}. Empirical works have extensively explore the remarkable efficacy  of model ensemble~\citep{wortsman2022robust, wortsman2022model, rame2022diverse, cha2021swad, arpit2022ensemble,   tian2023trainable, kumar2022calibrated,lin2023spurious}.
 Another line of research focuses on developing boosting algorithms, which also rely on ensemble-based methods \citep{freedman1981bootstrapping, breiman1996bagging, freund1997decision, friedman2001greedy, zhang2005boosting, rodriguez2006rotation,  kolter2007dynamic, galar2011review, kuncheva2014combining, bolon2019ensembles}, 
and these works are orthogonal to ours findings. 

There is a limited number of works that attempt to theoretically explain the effectiveness of ensemble methods. \cite{brown2005managing} decomposes the prediction error of ensemble models into bias, variance and a covariance term between individual models, proposing algorithms to encourage the diversity of individual models to reduce the covariance term. \citet{allen2020towards} proposes a multi-view theory to explain the effectiveness of ensemble of deep models trained with gradient descent from different initialization, whereas, their analysis relies on a very specific data structure, assuming limited number (e.g., less than 10) of latent features, each data point containing a subset of these latent features. A more recent study by \citet{lin2023spurious} illustrates that when two models utilize distinct sets of latent features, their ensemble can harness a broader range of latent features, referred to as feature diversification. This diversification can lead to improvements in OOD generalization, indicating that the enhancement in OOD performance may be attributed to feature diversification. Our model, compared with \citet{allen2020towards} and \citet{lin2023spurious}, is more general as we do not impose specific structural assumptions on the latent features or the number of latent features learned by each individual model.




 


\subsection{Generalization Under Non-IID Distributions}
\label{sect:non-iid-related-works}
Typically, robustness under non-IID distribution is characterized as follows:
\begin{align*}
    \sup_{\cD \in \mathfrak{D}} \bbE_{(x, y) \sim \cD} [\ell(f(x, y))],
\end{align*}
where  $\mathfrak{D}$ is a set of distributions. The set $\mathfrak{D}$ outlines the potential distribution perturbations to which the estimator should exhibit robustness. If $\mathfrak{D}$ encompasses arbitrary distributions, it could lead to inconclusive results. Therefore, it is customary to impose constraints on the potential shifts included in $\mathfrak{D}$.

\paragraph{Distribution Robust Optimization and Adversarial Attacks.} The field of Distributional Robust Optimization (DRO) focuses on the uncertainty set $\mathfrak{D}(\cD_0, \epsilon) = \{\cD: M(\cD, \cD_0) \leq \epsilon\}$, wherein $M$ represents a distance measurement. Typically, the value of $\epsilon$ is small. Examples of distance measurements include the Wasserstein distance \citep{kuhn2019wasserstein,mohajerin2018data,esfahani2015data, luo2017decomposition}, $\phi$-divergence \citep{hu2013kullback,namkoong2016stochastic,levy2020large,duchi2021learning,staib2019distributionally}, and others.  \citet{keskar2016large, namkoong2016stochastic, qi2021online,mehta2023distributionally, zhang2024stochastic} propose efficient optimization algorithms for solving the DRO problem. Meanwhile, \citet{sinha2017certifying, duchi2021learning, an2021generalization, kuhn2019wasserstein} explore the generalization ability of DRO estimators. Notably, \citet{sinha2017certifying} establishes a connection between Wasserstein distance DRO and adversarial examples. Adversarial examples involve making slight perturbations to the input $x$ with the goal of maximizing the performance drop for a given model $f$,
\begin{align}
\label{eqn:adversarial}
\mathcal{L}_{\adv}(f) = \mathbb{E}_{(x, y) \sim \cD} \left[\max_{\delta \in \Delta_\adv}\ell (f(x + \delta, y))\right].
\end{align}
\cite{hao2024surprising} shows that ``benign overfitting" estimators are overly sensitive to adversarial attacks. Specifically, as the sample size $n$ grows, the adversarial loss $\mathcal{L}_{\adv}(f)$ would diverge to infinity, even when the estimator is benign on generalization performance and the ground truth model is adversarially robust.

\paragraph{Natural Distributional Shifts.} The natural distributional shifts is closely aligned with the set of distributional shifts defined by causal graphs \citep{gong2016domain,huang2020causal,peters2016causal,arjovsky2019invariant,lin2022zin,heinze2017conditional}. Each node in the causal graph represents a (potentially latent) feature or an outcome \cite{pearl2009causality}. Distributional shifts arise from conducting do-interventions on the causal graphs \cite{pearl2009causality, peters2016causal, arjovsky2019invariant}. It is believed that the reliance on certain nodes  in the causal graph may lead to models being sensitive to distributional shifts.
To address this, significant research has focused on developing models that depend on a resilient subset of causal nodes \citep{arjovsky2019invariant,  gong2016domain, lin2022zin, ganin2016domain,lin2022bayesian,ahuja2020invariant}. Our work is orthogonal to these works by examining how over-parameterization affects generalization under distributional shifts. The behavior of natural distributional shifts \citep{moayeri2022explicit} differs from that in adversarial examples which introduce synthetic perturbations on the input.  Moreover, the work of \citet{moayeri2022explicit} reveals that a trade-off exists between the robustness against adversarial attacks and the ability to handle natural shifts \citep{moayeri2022explicit}.  
Though shifts in the causal graph can lead to changes in both $P(x)$ and $P(y|x)$. However, in our study, we align with the common assumption in covariate shift research, which only considers changes in $P(x)$ while maintaining $P(y|x)$ unchanged~\citep{sugiyama2007covariate,gretton2008covariate,bickel2009discriminative,wu2022power}.

\section{Preliminary and Settings}\label{sec:pre}

\textbf{Notation.} For any matrix $A$, we use $\| A \|_2$ to denote its $\ell_2$ operator norm, use $\mathrm{tr}\{ A \}$ to denote its trace, use $\| A \|_F$ to denote its Frobenius norm, and use $A^\dag$ denotes its Moore-Penrose pseudoinverse. The $j-$th row of $A$ is denoted as $A_{j \cdot}$, and the $j-$th column of $A$ is denoted as $A_{\cdot j}$. The $i-$th largest eigenvalue of $A$ is denoted as $\mu_i(A)$. The transposed matrix of $A$ is denoted as $A^T$. And the inverse matrix of $A$ is denoted as $A^{-1}$. For any set $\mathcal{C}$, we use $| \mathcal{C} |$ to denote the number of components within $\mathcal{C}$. The notation $a = o(b)$ means that $a/b \to 0$; similarly, $a = \omega(b)$ means that $a/b \to \infty$. For a sequence of random variables $\{ v_s \}$, $v_s = o_p(1)$ refers to $v_s \inprobto 0$ as $s \to \infty$, and the notation $\gamma_s v_s = o_p(1)$ is equivalent to $v_s = o_p(1/\gamma_s)$; $v_s = O_p(1)$ refers to $\lim_{M \to \infty} \sup_s \mathbb{P}(| v_s | \ge M) = 0$, similarly, $\gamma_s v_s = O_p(1)$ is equivalent to $v_s = O_p(1/\gamma_s)$.

\subsection{Data Settings}
We consider regression tasks where $n$ i.i.d. training examples $(x_1, y_1), \dots, (x_n, y_n)$ from distribution $\cD$ take values in $\mathbb{R}^p \times \mathbb{R}$. Here, $y$ is generated based on $x$ with an unknown non-linear function $g(\cdot): \mathbb{R}^p \to \mathbb{R}$:
\begin{equation}\label{eq:truth}
    y := g(x) + \epsilon.
\end{equation}
We adopt the following assumptions on $\{(x_i, y_i)\}_{i=1}^n$:
\begin{enumerate}
\item $x_i = \varSigma^{1/2} \eta_i$, where $\varSigma := \E[x_ix_i^\T] = \text{diag}[\lambda_1, \dots, \lambda_p]$, and the components of $\eta_i$ are independent $\sigma_x$-subgaussian random variables with mean zero and unit variance;
\item $\mathbb{E}[y_i \mid x_i] = g(x_i)$ (as already stated in \eqref{eq:truth});
\item $\mathbb{E}[y_i - g(x_i) | x_i]^2 = \mathbb{E} [\epsilon_i]^2 = \sigma^2 > 0$.
\end{enumerate}
Without loss of generality, we assume $\lambda_1 \ge \lambda_2 \ge \dots \ge \lambda_p > 0$ on the covariance matrix $\varSigma$. Then similar to the definition in \citet{bartlett2020benign}, the effective rank could be defined for each non-negative integer $k$:
\begin{equation}\label{eq:rank}
    r_k := \frac{\sum_{i>k} \lambda_i}{\lambda_{k+1}} ,
\end{equation}
where the critical index for a given $b > 0$ is
\begin{equation}\label{eq:defk}
     k^*(b) := \inf \{ k \geq 0 : r_k \geq bn \} .
\end{equation}

\subsection{Distributional Shift}
\label{sect:performance_measure}

For the OOD situation, we consider an addictive covariate shift $\delta$ on $x$, meaning the input variable changes from $x$ to $x + \delta$. We assume that the conditional probability $\bbP[y|x]$ remains unchanged in OOD scenarios, i.e., $y = g(x+\delta) + \epsilon$ in OOD. We consider that 

\begin{assumption}[Independence of the Shift]\label{ass:ood1}
$\delta$ is independent of observer $x$ and noise $\epsilon$, i.e, 
\begin{equation*}
   \mathbb{E} \delta x^T = 0, \quad \mathbb{E} \epsilon \delta = 0.
\end{equation*}
\end{assumption}
\paragraph{Discussion on the Data Generalization Process.}
It is common to only consider changes in $\mathbb{P}(x)$ while maintaining $\mathbb{P}(y | x)$ unchanged \citep{ben2006analysis}. For example, consider our task is to predict the weight (i.e., $y$) of a man based his height (i.e., $x$).  The training data is collected in  country $A$ (the training domain) while the testing data is from country $B$ (the testing domain).  It should be noted that the distribution of height in the testing domain differs from that in the training domain. For instance, men from country $B$ are generally taller than those from country $A$. This difference is quantified by $\delta$, and the distribution of $\delta$ in our context is dependent on the domain but independent of $x$ (as illustrated in the causal graph given by Figure~\ref{fig:causal_graph}). Following the common assumption in covariate shift literature, we consider the underlying function $g$ that relates height to weight, and this function remains the same in both the testing and training domains~\citep{sugiyama2007covariate,gretton2008covariate,bickel2009discriminative}.

While we denote $\delta \sim \Delta \in \Xi_{\mathrm{ood}}$, it is crucial to consider a feasible distribution set $\Xi_{\mathrm{ood}}$ for the potential shift. If $\Xi_{\mathrm{ood}}$ is too small, the OOD loss will not significantly differ from the ID loss. Conversely, if we consider an excessively large $\Xi_{\mathrm{ood}}$ encompassing arbitrary shifts, it becomes impossible to derive meaningful conclusions. In this paper, we consider a shift strength that allows for a shift $\delta$ comparable to the input $x$ in terms of their eigenvalues. This is due to the fact that natural shifts typically result in perceptual differences in the samples from ID and OOD domains. Considering the conditions of the data in \citet{bartlett2020benign}, where they separately examine ``small'' and ``large'' eigenvalues of the input $x$, we also impose assumptions on $\delta$ for the ``small'' and ``large'' eigenvalues, respectively. Specifically, let $b>0$ be a constant, denote the indexes of ``large'' eigenvalues as $\cC_1 = \{i : \lambda_i > \tx / (bn)\}$, the corresponding eigenvectors could concat a matrix $\varSigma_{\cC_1} \in \mathbb{R}^{|\cC_1 | \times p}$, then the projection of matrix $\bm{A} \in \mathbb{R}^{p \times p}$ on the subspace spanning on $\varSigma_{\cC_1}$ could be denoted as $\bm{\Pi}_{\cC_1} \bm{A} := \varSigma_{\cC_1}^T (\varSigma_{\cC_1} \varSigma_{\cC_1}^T)^{-1} \varSigma_{\cC_1} \bm{A}$; similarly, denote the indexes of ``small'' eigenvalues as  $\cC_2 = \{i : \lambda_i \le \tx / (bn)\}$, we also have the projection $\bm{\Pi}_{\cC_2} \bm{A}$ on the subspace spanning on their corresponding eigenvectors. Notably, recalling Eq.~\eqref{eq:defk} and the benign overfitting condition in \citet{bartlett2020benign}, we have 
\begin{equation*}
   \lambda_{k^*+1} \le \frac{\sum_{j > k^*} \lambda_j}{bn} \Longrightarrow \lambda_{k^* + 1} \le \frac{l}{bn} \Longrightarrow |\cC_1| \le k^* \ll n.
\end{equation*}
With $\cC_1$ and $\cC_2$ defined above, we assume
\begin{assumption}[The Strength of  the Shift]\label{ass:ood2}
    Denoting $\varSigma_{\delta} = \mathbb{E} \delta \delta^T \in \mathbb{R}^{p \times p}$ with eigenvalues $\{ \alpha_1, \dots, \alpha_p \}$,  we have $\| \varSigma_{\delta} \|_2 \le \tau$ and the following constraints:
\begin{align*}
 & \mbox{shifts on the directions of ``large'' eigenvalues: }  \mathrm{tr}\{ \bm{\Pi}_{\cC_1} (\varSigma_{\delta} \varSigma^{-1}) \} \le \tau ,\\
 & \mbox{shifts on the directions of ``small'' eigenvalues: }  \| \bm{\Pi}_{\cC_2} \varSigma_{\delta} \|_2 \le \tau \| \bm{\Pi}_{\cC_2} \varSigma \|_2  . 
\end{align*}
with some constant $\tau > 0$.
\end{assumption}

For simplifying the analysis, here we consider $\delta$ is zero-mean and $\varSigma_{\delta}$ as a diagonal matrix, and similar results could be shown with milder constraints.

\begin{figure}
    \centering
    \includegraphics[width=0.2\linewidth]{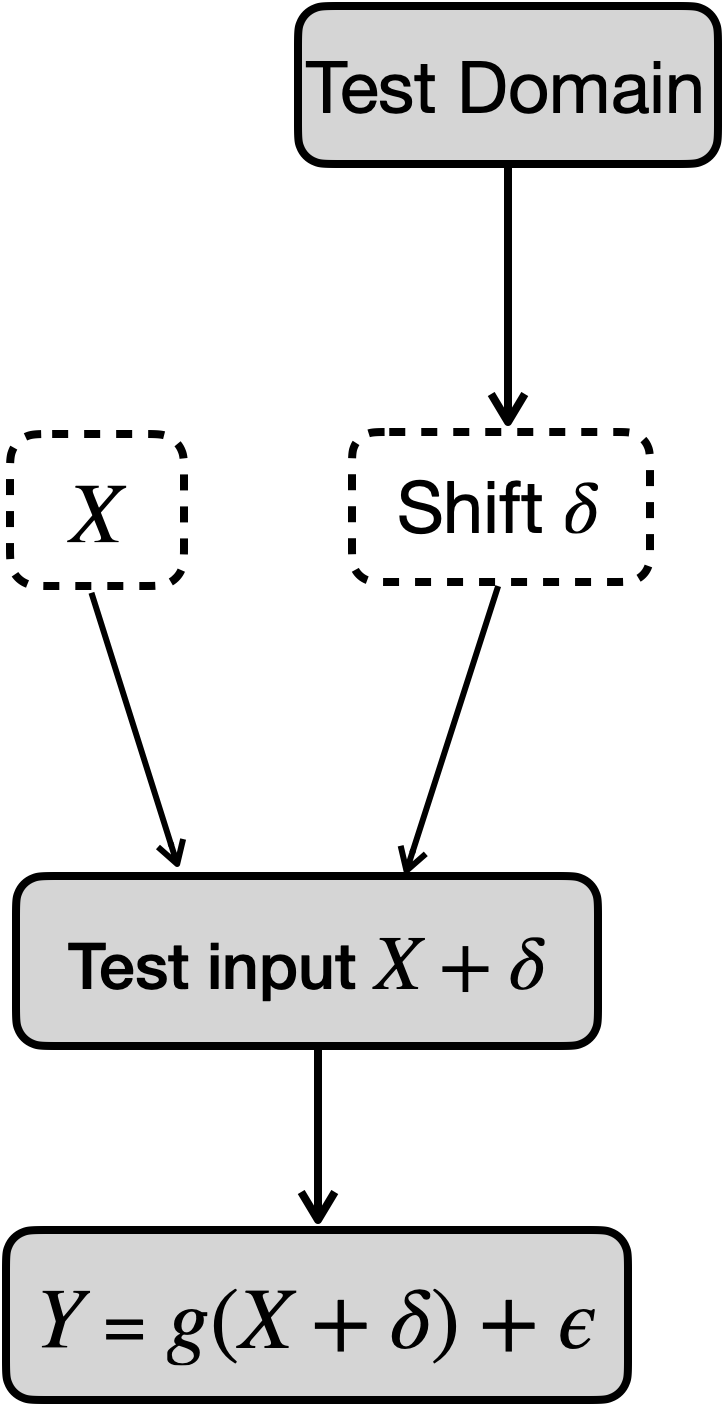}
    \caption{Illustration of the causal graph.}
    \label{fig:causal_graph}
\end{figure}

\paragraph{Discussion on the Data Model}
We consider a relatively large set of $\Xi_{\mathrm{ood}}$ where the eigenvalues of  the shift $\delta$ are comparable with those of the input $x$. Specifically, the constraints on $\varSigma_{\delta}$ could be explained as follows: 
\begin{itemize}
    \item [(1)] $\|\Sigma_\delta\|_2 \leq \tau$,  the maximum shift on each direction of eigenvalues should be within a constant level.
    \item [(2)] $\mathrm{tr}\{ \bm{\Pi}_{\cC_1} (\varSigma_{\delta} \varSigma^{-1}) \} \le \tau n$. We allow for a relatively large shift on the large eigenvalues. If we consider diagonal matrices for both $\Sigma_\delta$ and $\Sigma$, we have $\sum_{ i \in \cC_1} \alpha_i / \lambda_i \leq \tau n$. The average value of $\alpha_i / \lambda_i$ on the directions of ``large" eigenvalues should be smaller than $\tau n / |\cC_1|$. As is mentioned above, the number of large eigenvalues ($\lambda_i \ge \tx / (bn)$), i.e., $|\cC_1|$,  is smaller than $k^*$. Furthermore, following \citet{bartlett2020benign} which considers  $k^*/ n \to 0$ in the benign overfitting region, we then conclude that the number of large eigenvalues is  significantly smaller than both sample size $n$ and data dimension $p$. So we allow the average value of $\alpha_i / \lambda_i$ in $\mathcal{C}_1$ up to $\tau n/|\cC_1|$ which goes to $\infty$. 
    \item [(3)] $\| \bm{\Pi}_{\cC_2} \varSigma_{\delta} \|_2 \le \tau \| \bm{\Pi}_{\cC_2} \varSigma \|_2$. The spectral norm of $\delta$ on the directions of small eigenvalues is within a constant level compared with that of $x$. 
\end{itemize}

\paragraph{Empirical Observations} In order to investigate natural distribution shifts, we conducted an analysis using observations from the DomainNet datasets \citep{peng2019moment}. DomainNet comprises multiple domains, each of which contains images from specific distributions. For instance, the ``real domain'' encompasses photos captured in real-world settings, while the ``quickdraw'' domain contains drawings created by players of the global game ``Quick Draw!''(see Figure~\ref{fig:domain_diff}). For each domain, we extract features from the observed images using a pre-trained ResNet18 \citep{he2016deep} and then calculate the eigenvalues of the corresponding covariance matrix. To simplify, we designate the "real" domain as in-domain and denote the eigenvalues of the covariance matrix in the ``real'' domain as $[\lambda_1, \dots, \lambda_p]$. As an illustration, let's consider the ``quickdraw'' domain. We first obtain the eigenvalues in the ``quickdraw'' domain as $[\lambda'_1, \dots, \lambda'_p]$ and then quantify the difference between the ``quickdraw'' and ``real'' domain using the eigenvalue ratio: $[\lambda'_1/\lambda_1, \dots, \lambda'_p/\lambda_p]$. Similarly, we also calculate the eigenvalue ratios for the ``clipart'', ``sketch'', and ``infograph'' domains using the ``real'' domain as the base.  We observe these differences across four domains, with detailed results presented in Figure~\ref{fig:domain_diff}. Our observations reveal that when distribution shifts occur, the discrepancies in eigenvalues are significant, particularly with notable changes in several ``large'' eigenvalues. These findings are consistent with Assumption~\ref{ass:ood2}.

\begin{figure}
    \centering
    \includegraphics[width=0.5\linewidth]{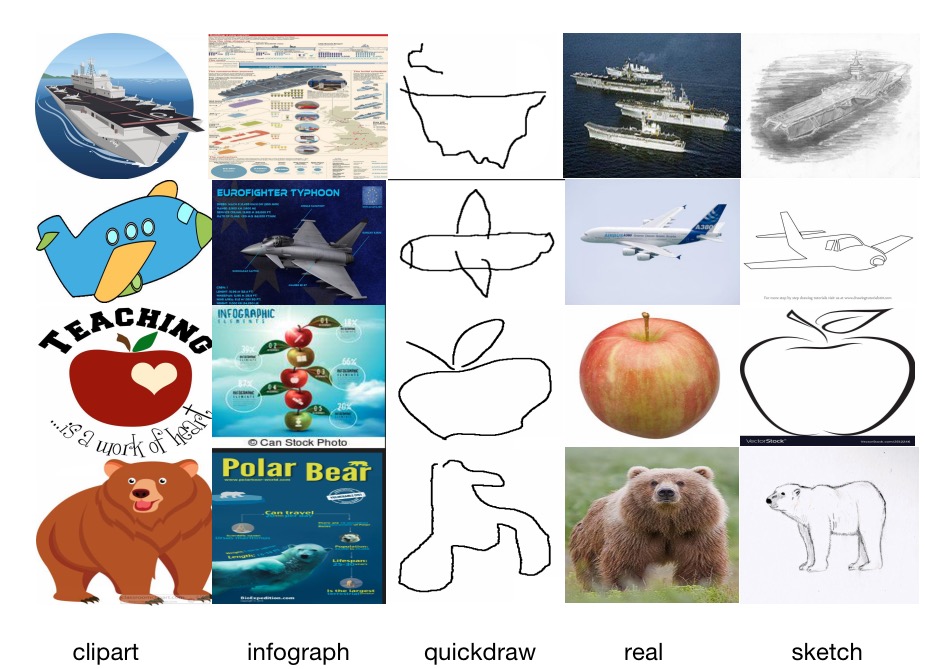}
    \caption{Examples of image data in different domains.}
    \label{fig:domainpic}
\end{figure}

\begin{figure}[h]
     \centering
     \begin{subfigure}[b]{0.48\linewidth}
         \centering         
         \includegraphics[width=\linewidth]{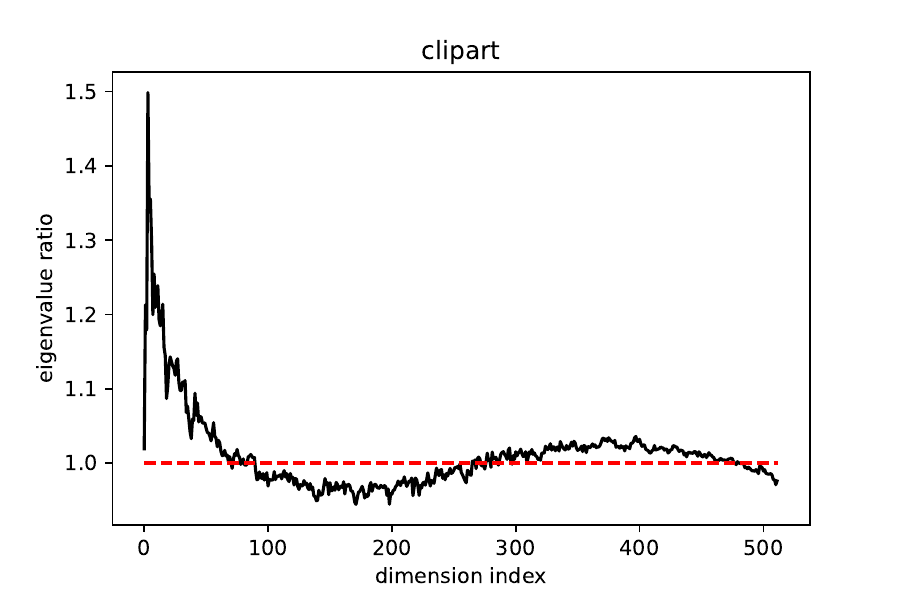}
     \end{subfigure}
     \hspace{1em}
     \begin{subfigure}[b]{0.48\linewidth}
         \centering         
         \includegraphics[width=\linewidth]{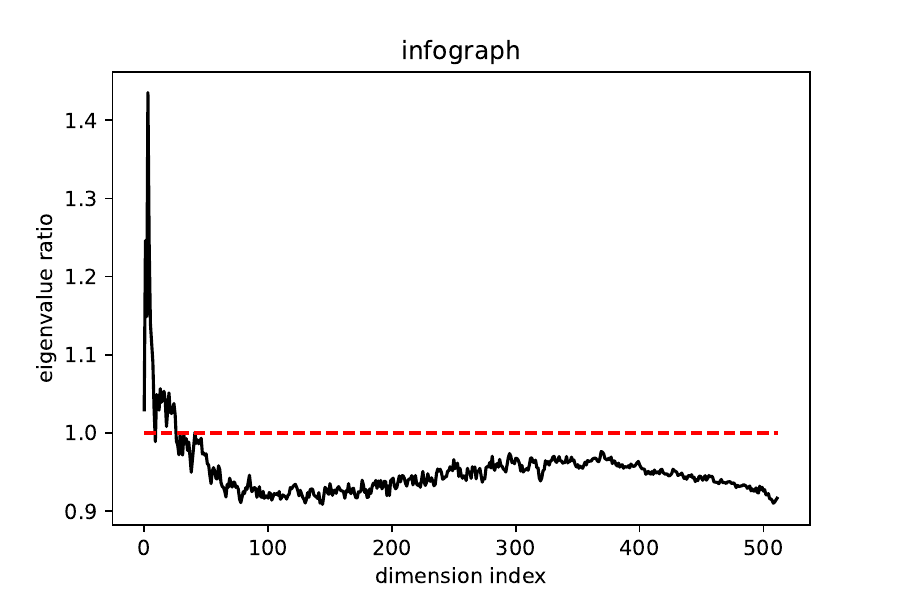}
     \end{subfigure}
     \hspace{1em}
     \begin{subfigure}[b]{0.48\linewidth}
         \centering         
         \includegraphics[width=\linewidth]{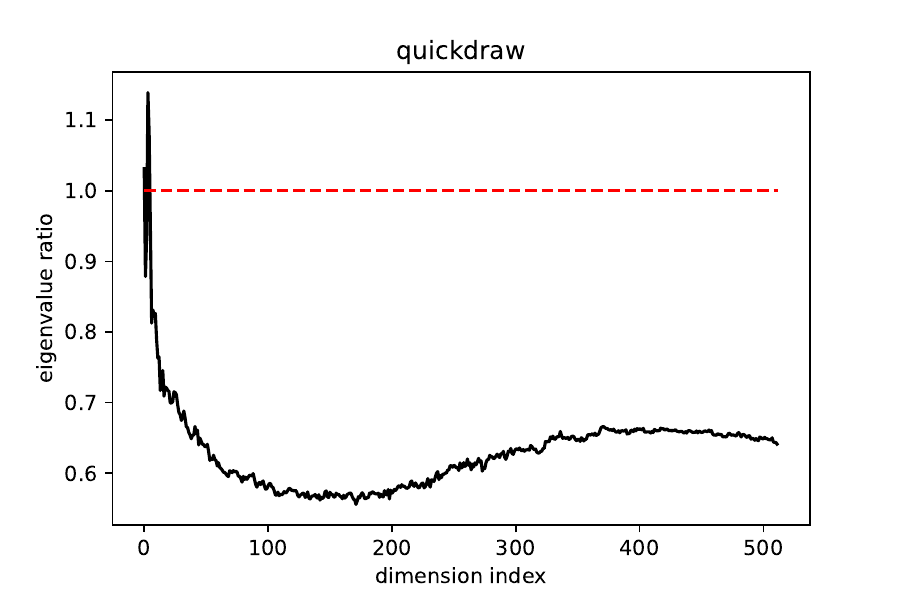}
     \end{subfigure}
     \hspace{1em}
     \begin{subfigure}[b]{0.48\linewidth}
         \centering         
         \includegraphics[width=\linewidth]{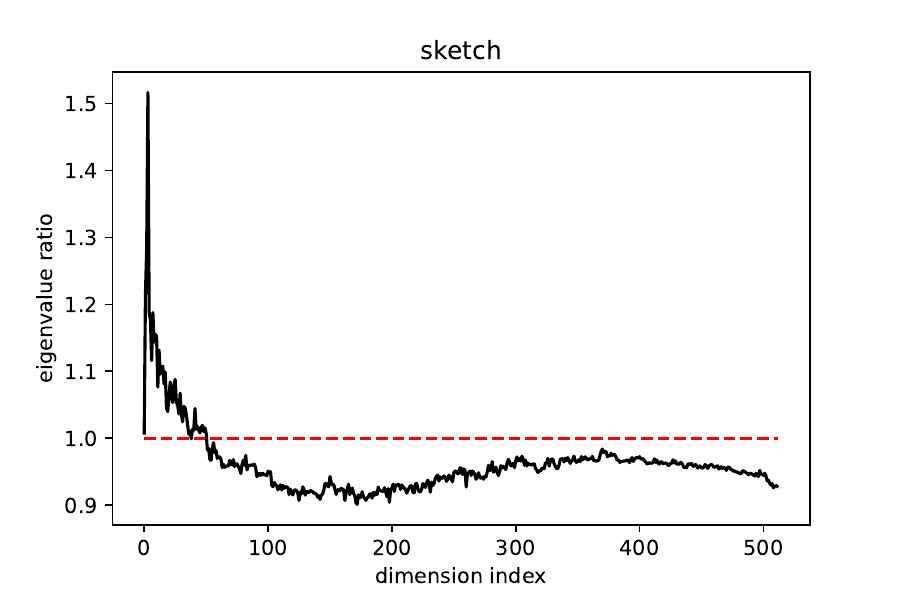}
     \end{subfigure}

     \vspace{5pt}
    \raggedright
    \footnotesize
    \emph{Notes.} Here we analyze the eigenvalues across four distinct domains, i.e, ``clipart'', ''infograph'', ``quickdraw'' and ``sketch'', in comparison to the eigenvalues in the ``real'' domain.  The solid black line illustrates the eigenvalue ratios for each dimension index, while the dashed red line represents the baseline eigenvalue ratio of $1$.
    \caption{Comparison of covariance matrix eigenvalues in different domain.}
     \label{fig:domain_diff}
\end{figure}

\paragraph{The Difference Between Natural Shifts and Adversarial Attacks} Comparing the natural shift formulation in Equation~\eqref{eqn:formulation} and adversarial attacks in Equation~\eqref{eqn:adversarial}, a clear distinction arises in which the perturbation $\delta$ of adversarial examples, as defined in Equation~\eqref{eqn:adversarial}, is dependent on each example $x$ in an adversarial manner. It is important to note that for over-fitting to be considered ``benign", the long tail feature of each sample must be orthogonal. Over-fitting the label noise in the training data using these features would not have a substantially negative effect on the testing data \citep{shamir2022implicit}. However, the adversarial perturbation $\delta$ could disrupt the orthogonality, resulting in a significant increase in testing loss. For instance, in a simple linear regression problem with $\mathbb{E}[y |x ] = x^T \theta^*$,  where we have a benign estimator $\hat \theta$, the adversarial perturbation $\delta$ that solves $\max_\delta  ((x+\delta)^\top \hat{\theta} - \theta^{*T} x)^2$ is given by $\delta(x) = (\hat{\theta} \hat{\theta}^T)^{\dag} (\hat{\theta} \theta^{*T} - \hat{\theta} \hat{\theta}^T) x$. 
Consequently, $\delta(x)$ encompasses the information of training data (from the estimator $\hat{\theta}$), thus breaking the orthogonality of the long tail features. \citet{hao2024surprising} also demonstrate that over-parameterization leads to sensitivity to adversarial examples. In contrast, the orthogonality of long tail features is preserved in OOD situation, as natural distributional shifts do not adversarially explore the long tail features of each sample. As we will demonstrate later on, even when considering a large scaling distributional shift $\delta$ comparable to the input $x$, increased hidden dimensions can still benefit the OOD loss. 

\subsection{Model and Performance Measurement}
To estimate the target $y$, we study a random feature model with parameter $W \in \mathbb{R}^{p \times m}$ in this work:
\begin{align*}
& f_{W} (\theta, x) = \frac{1}{\sqrt{m}} {\phi(x^T W)} \theta,
\end{align*}
where $\phi(z) = \max \{ 0, z\}$ is an element-wise ReLU function, and all of the elements on $W$ are i.i.d. sampled from Gaussian distribution $\mathcal{N}(0, 1/p)$.
Accordingly, $\theta^*(W)$ is denoted the minimizer of expected in-distribution (ID) mean square error (MSE):
\begin{align*}
& \theta^*(W) = \arg \min_{\theta \in \mathbb{R}^m} \mathbb{E}_{x, y} [y - f_W(\theta, x)]^2 = \arg \min_{\theta \in \mathbb{R}^m} \mathbb{E}_{x, y} [y - 
\frac{1}{\sqrt{m}} \phi( x^T W) \theta]^2,
\end{align*}
where we use $W$ in $\theta^*(W)$ to explicitly denote the dependency of $\theta^*(W)$ on $W$. 
We assume the estimation ability of random feature models is strong enough to satisfy that
\begin{equation}\label{eq:est_ability}
     \sup_{x \in \mathbb{R}^p} | g(x) - f_W(\theta^*(W), x) | \le \varrho = o(1).
\end{equation}
Given samples $\{ (x_i, y_i) \}_{i=1}^n$, we obtain the min-norm estimator $\hat{\theta}(W)$ as
\begin{align*}
 & \hat{\theta}(W) = \Phi_W^T (\Phi_W \Phi_W^T)^{-1}y,
\end{align*}
where $\Phi_W = \frac{1}{\sqrt{m}} [\phi(W x_1), \dots, \phi(W  x_n))]^T \in \mathbb{R}^{n \times m}$.
The ID performance is measured by excess mean square error 
\begin{equation}\label{eq:iderror}
    \mathcal{L}_{\mathrm{id}}(f_W(\hat{\theta}(W), x)) = \mathbb{E}_{x, y} \left[f_W(\hat{\theta}(W), x) - f_W(\theta^*(W), x) \right]^2 = \mathbb{E}_{x, y} \left[\frac{1}{\sqrt{m}} \phi(x^T W) \left( \hat{\theta}(W) - \theta^*(W) \right) \right]^2.
\end{equation}
For OOD situation, let $\theta^*_\delta(W)$ denotes the optimal $\theta$ given a specific $\delta \sim \Delta \in \Xi_{\mathrm{ood}}$, i.e., $\theta^*_\delta(W) = \arg \min_{\theta \in \mathbb{R}^m} \mathbb{E}_{x, \delta, y} [y - 
\frac{1}{\sqrt{m}} \phi( (x + \delta)^T W) \theta]^2$, 
we are interested in the maximum OOD excess risk of $\hat \theta(W)$ with respect to the optimal $\theta^*_\delta(W)$:
\begin{equation*}
    \max_{\Delta \in \Xi_{\mathrm{ood}}} \mathbb{E}_{x, \delta, y} \left[f_W(\hat{\theta}(W), x + \delta) - f_W(\theta^*_{\delta}(W), x+\delta) \right]^2.
\end{equation*}
Furthermore, taking Eq.~\eqref{eq:est_ability} into consideration, we can see that $f_W(\theta^*_\delta(W), x)$ is close to the optimal estimator $f_W(\theta^*(W), x)$ in ID domain since the mapping function $g$ between $x$ and $y$ remains unchanged in ID and OOD domains, i.e, for any $\delta$, we have
\begin{align*}
& \quad \mathbb{E}_{x, y, \delta} [ \frac{1}{\sqrt{m}}\phi((x+\delta)^T W)\theta^*(W) - \frac{1}{\sqrt{m}} \phi((x+\delta)^T W)\theta^*(W,\delta)]^2 \\
&= \mathbb{E}_{x, y, \delta} [\frac{1}{\sqrt{m}} \phi((x+\delta)^T W)\theta^*(W) - g(x+\delta) + g(x+\delta) - 
\frac{1}{\sqrt{m}} \phi((x+\delta)^T W)\theta^*(W,\delta)]^2\\
&\le 2 \mathbb{E}_{x, y, \delta} [\frac{1}{\sqrt{m}} \phi((x+\delta)^T W)\theta^*(W) - g(x+\delta)]^2 + 2 \mathbb{E}_{x, y, \delta} [ g(x+\delta) - \frac{1}{\sqrt{m}} \phi((x+\delta)^T W)\theta^*(W,\delta)]^2 \\
&\le 4 \mathbb{E}_{x, y, \delta} [\frac{1}{\sqrt{m}} \phi((x+\delta)^T W)\theta^*(W) - g(x+\delta)]^2 \\
&\le 4  \mathbb{E}_{x, y, \delta} \sup | \frac{1}{\sqrt{m}} \phi((x+\delta)^T W)\theta^*(W) - g(x+\delta) |^2 \le 4 \varrho^2 = o(1),
\end{align*}
so we cold measure the OOD performance by
\begin{equation}\label{eq:ooderror}
\begin{aligned}
 \mathcal{L}_{\mathrm{ood}} (f_W(\hat{\theta}(W), x)) &= \max_{\Delta \in \Xi_{\mathrm{ood}}} \mathbb{E}_{x, \delta, y} \left[f_W(\hat{\theta}(W), x + \delta) - f_W(\theta^*(W), x+\delta) \right]^2\\
 &=\max_{\Delta \in \Xi_{\mathrm{ood}}} \mathbb{E}_{x, \delta, y} \left[ \frac{1}{\sqrt{m}} \phi((x + \delta)^T W) (\hat{\theta}(W) -\theta^*(W)) \right]^2.   
\end{aligned}  
\end{equation}

\subsection{Model Ensemble}
As introduced in Section~\ref{sect:intro}, there has been repeated empirical observations showing that ensemble models could achieve superior OOD performance~\citep{wortsman2022model, wortsman2022robust, lin2023spurious, rame2022diverse}. So we are also interested in investigating the effectiveness of model ensemble for OOD generalization. 
The ensemble model is defined as the average of $K$ outputs related to $K$ ``ridgeless'' estimators in independently trained single models with corresponding parameters $\{W_1, \dots, W_K \}$. To be specific, for single models
\begin{equation*}
    f_{W_r}(\theta_r, x) = \frac{1}{\sqrt{m}} \phi(x^T W_r) \theta_r, \forall r = 1, \dots, K,
\end{equation*}
the ensemble model is defined as
\begin{equation*}
    f_{\mathrm{ens}}(\theta_1, \dots, \theta_K, x) = \frac{1}{K} \sum_{r=1}^K f_{W_r}(\theta_r, x) = \frac{1}{K \sqrt{m}} \sum_{r=1}^K  \phi(x^T W_r) \theta_r.
\end{equation*}
Recalling the min-norm estimators $\{ \hat{\theta}(W_1), \dots, \hat{\theta}(W_K) \}$ on each single model, we explore the generalization performance of
\begin{equation*}
    f_{\mathrm{ens}}(\hat \theta(W_1), \dots, \hat \theta(W_K), x) = \frac{1}{K} \sum_{r=1}^K f_{W_r}(\hat \theta(W_r), x) = \frac{1}{K \sqrt{m}} \sum_{r=1}^K  \phi(x^T W_r) \hat \theta(W_r),
\end{equation*}
then according to \eqref{eq:iderror} and \eqref{eq:ooderror}, the ID and OOD performance are measured respectively by
\begin{align*}
& \mathcal{L}_{\mathrm{id}}(f_{\mathrm{ens}}(\hat{\theta}(W_1), \dots, \hat{\theta}(W_K),x)) = \mathbb{E}_{x,y} \left[\frac{1}{K \sqrt{m}} \sum_{r=1}^K  \phi(x^T W_r) (\hat{\theta}(W_r) - \theta^*(W_r))  \right]^2,\\
& \mathcal{L}_{\mathrm{ood}}(f_{\mathrm{ens}}(\hat{\theta}(W_1), \dots, \hat{\theta}(W_K), x+\delta)) =\max_{\Delta \in \Xi_{\mathrm{ood}}} \mathbb{E}_{x,y,\delta} \left[\frac{1}{K \sqrt{m}} \sum_{r=1}^K  \phi((x + \delta)^T W_r) ( \hat{\theta}(W_r) - \theta^*(W_r))  \right]^2.
\end{align*}
Without loss of generality, we focus our analysis on the case where $K=2$, and extending our findings to other choices of $K$ is straightforward.

\section{Main Results}\label{sec:re}
Following \citet{bartlett2020benign},
we focus on the ``benign overfitting'' phase with the assumptions bellow:
\begin{assumption}[benign matrix]\label{cond:benign}
There exist some constants $\xi > 0$ and $b > 0$ such that for $k^* = k^*(b)$,
\begin{equation*}
    \lim_{n \to \infty} \frac{r_0(\varSigma)}{n} = \lim_{n \to \infty} \frac{k^*}{n} = \lim_{n \to \infty} \frac{n^{1 + \xi}\sum_i \lambda_i^2 }{(\sum_i \lambda_i)^2} = 0.
\end{equation*}    
\end{assumption}
Assumption~\ref{cond:benign} is compatible with the assumption in \citet{bartlett2020benign}, which characterizes the slow decreasing rate on covariance eigenvalues $\{ \lambda_i \}$. Moreover, Assumption~\ref{cond:high-dim} is also required in further analysis:
\begin{assumption}[High-dimension condition]\label{cond:high-dim}
Here we consider the relationships among $n, p, m, l$ are as follows:
\begin{equation*}
    n \le p^{1/4}, \quad \tx \gg n^{3/4}, \quad n \gg \ln m, \quad m \ge p.
\end{equation*}    
\end{assumption}
Assumption~\ref{cond:high-dim} describes the high-dimension setup on data $x$, as well as the overparameterized structure of random feature models.
To show the compatibility of Assumption~\ref{cond:benign} and \ref{cond:high-dim}, we verify them on two examples from \citet{bartlett2020benign}.

\begin{example}\label{eg1}
 Suppose the eigenvalues as
\begin{equation*}
\lambda_k = \left\{
\begin{aligned}
& 1, k = 1, \\
& \frac{1}{n^{21/5}} \frac{1 + s^2 - 2 s \cos(k \pi / (p + 1))}{1 + s^2 - 2 s \cos(\pi / (p + 1))}, \quad 2 \le k \le p, \\
& 0, \text{otherwise},
\end{aligned}
\right.
\end{equation*}
where $p = n^5$. We could obtain $k^* = 1$ and $0 < \xi < 4$.
\end{example}

\begin{example}\label{eg2}
 Suppose the eigenvalues as
 \begin{equation*}
     \lambda_k = k^{- 5/6}, \quad 1 \le k < p,
 \end{equation*}
where $p = n^5$. We could obtain $k^* = n^{1/5}$ and $0 < \xi < 2/3$.
\end{example}
The detailed calculations are in Appendix~\ref{pf:eg}, and it is easy to design other similar examples satisfying Assumption~\ref{cond:benign} and Assumption~\ref{cond:high-dim}.

Our first main result could be stated below, shows that while the distribution shift on $x$ is significant, even there is a near-optimal ID performance in ``benign overfitting'' regime, the ``benign'' estimator always leads to a non-converging OOD excess risk (the detailed proof is in Appendix~\ref{pf:idood}).
\begin{theorem}\label{thm_idood}
 For any $\sigma_x,b, \xi, \varrho > 0$, there exist $\xi' = \min\{ 1/2, \xi/2 \}$ and constants $C_1, C_2 > 0$ depending only on $\sigma_x, b, \xi, \varrho$, such that the following holds.
Assume Assumption~\ref{ass:ood1}, \ref{ass:ood2}, \ref{cond:benign}  and \ref{cond:high-dim} are satisfied,  there exists a constant $c>1$ such that for $\delta \in (0,1)$ and $\ln(1 / \delta) < n^{\xi'} / c$,
with probability at least $1 - \delta$ over $X, W_1, W_2$,
\begin{align*}
& \mathcal{L}_{\mathrm{id}}(f_{W_r}(\hat{\theta}(W_r), x))   \le C_1 \left\{ \frac{\tx}{p} \frac{\| \theta^*(W_r) \|_2^2 }{ n^{1/4}}  + \sigma^2 \left( \frac{1}{n^{1/8}} + \frac{k^*}{n} + \frac{n \sum_{j > k^*} \lambda_j^2}{\tx^2} \right) \right\}, \quad r = 1, 2,
\end{align*}
and
\begin{align*}
& \mathcal{L}_{\mathrm{ood}}(f_{W_r}(\hat{\theta}(W_r), x)) \ge C_2 \left\{\sigma^2 \tau  \frac{p}{m}  + \sigma^2 \tau \frac{ \sum_{\lambda_j \le \tx / (b n)} \lambda_j}{\tx} \right\}, \quad r = 1, 2.
\end{align*}
\end{theorem}
Given a target $y$ with constant scaling and $\theta^*(W_r)$ satisfying \eqref{eq:est_ability}, we have $\| \theta^*(W_r) \|_2 = O(p / \tx)$, which implies the bias term $\tx \cdot \| \theta^*(W_r) \|_2^2 / (p n^{1/4})$ in ID excess risk has a convergence rate $O(n^{- 1/4})$. Combining this observation with the converged performance of variance term induced by  Assumption~\ref{cond:benign}, we can directly derive the following Corollary~\ref{coro:idood}:

\begin{corollary}\label{coro:idood}
For any $\sigma_x,b, \xi, \varrho > 0$, there exist $\xi' = \min\{ 1/2, \xi/2 \}$ and some constant $C_3 > 0$ depending only on $\sigma_x, b, \xi, \varrho$, such that the following holds.
Assume Assumption~\ref{ass:ood1}, \ref{ass:ood2}, \ref{cond:benign}  and \ref{cond:high-dim} are satisfied and $p/m = O(1)$,  we have
\begin{align*}
 & \lim_{n \to \infty} \mathcal{L}_{\mathrm{id}}(f_{W_r}(\hat{\theta}(W_r), x)) = 0, \quad r = 1, 2,\\
 & \lim_{n \to \infty} \mathcal{L}_{\mathrm{ood}}(f_{W_r}(\hat{\theta}(W_r), x))  \ge C_3 \left\{ \sigma^2 \tau \frac{p}{m} + \sigma^2 \tau \frac{ \sum_{\lambda_j \le \tx / (b n)} \lambda_j}{\tx} \right\} = O(1) , \quad r = 1, 2.
\end{align*}
\end{corollary}


As the adversarial risk may escalate with increasing model capacity, one might inquire whether the behavior of the OOD risk is similar as the hidden dimension $m$ grows or the model is ensembled. The answer is negative, which is induced from the following Theorem~\ref{thm_relu}. Before delving into the results, we introduce a notation to denote the improvement of ensemble models on OOD risk:
\begin{equation*}
    \mathcal{R}_K := \frac{ \sum_{r=1}^K \mathcal{L}_{\mathrm{ood}} (f_{W_r}(\hat{\theta}(W_r), x)) /K - \mathcal{L}_{\mathrm{ood}} (f_{\mathrm{ens}} (\hat{\theta}(W_1), \dots, \hat{\theta}(W_K), x))}{\sum_{r=1}^K \mathcal{L}_{\mathrm{ood}} (f_{W_r}(\hat{\theta}(W_r), x)) /K}.
\end{equation*}
It can be readily seen that larger $R_K$ implies a more significant improvement of the $K$-ensemble model over single models. Then, the OOD performance of the ensemble model is stated as follows:
\begin{theorem}\label{thm_relu}
For any $\sigma_x,b, \xi, \varrho > 0$, there exist $\xi' = \min\{ 1/2, \xi/2 \}$ and constants $C_4, C_5 > 0$ depending only on $\sigma_x, b, \xi, \varrho$, such that the following holds.
Assume Assumption~\ref{ass:ood1}, \ref{ass:ood2}, \ref{cond:benign}  and \ref{cond:high-dim} are satisfied,  there exists a constant $c>1$ such that for $\delta \in (0,1)$ and $\ln(1 / \delta) < n^{\xi'} / c$,
with probability at least $1 - \delta$ over $X, W_1, W_2$,
\begin{align*}
& \mathcal{L}_{\mathrm{ood}}(f_{W_r}(\hat{\theta}(W_r), x))  \le C_4 \left\{ \tau \mathbb{E}_x \| \nabla_x g(x)^T \|_2^2 + \sigma^2 \tau \left( \frac{p}{m} + 1 \right) + \sigma^2 \tau \frac{\sum_{\lambda_j \le \tx/ (b n)} \lambda_j }{\tx} \right\}, \quad r = 1, 2,
\end{align*}
and
\begin{align*}
& \mathcal{R}_2 
\ge  \frac{C_5}{2} \frac{\sigma^2 \tau \cdot p / m }{\tau \mathbb{E}_{x} \| \nabla_x g(x)\|_2^2 + \sigma^2 \tau (p/m + 1) + \sigma^2 \tau \sum_{\lambda_j \le \tx/ (b n)} \lambda_j / \tx } .
\end{align*}
\end{theorem}

The detailed proof is in Appendix~\ref{pf:relu}. This result immediately implies the following consequence: although having a model capacity of $m = p$ is sufficient for achieving near-optimal performance ($\mathcal{L}_{\mathrm{id}}(f_{W_r}(\hat{\theta}(W_r), x)) \inprobto 0$), it is not enough for achieving good OOD performance;
as we increase the hidden dimension $m$ directly, or enlarge the model capacity by ensemble procedure, the increases in the number of parameters could benefit OOD risk. 

\begin{remark}
The decrease on OOD excess risk is related to $p/m$. To be specific, increasing the hidden dimension $m$ results in a decrease in OOD excess risk for the single models, but the corresponding improvement in OOD performance for ensemble methods is more modest. In an extreme scenario where $p/m \to 0$, ensembling does not lead to a reduction in OOD excess risk.
\end{remark}

\begin{remark}
 If we consider ensemble on $K$ single models, the improvement proportion in Theorem~\ref{thm_relu} should be
 \begin{align*}
& \mathcal{R}_K \ge C_5 \left( 1  - \frac{1}{K} \right) \frac{\sigma^2 \tau \cdot p / m }{\tau \mathbb{E}_x\|\nabla_x g(x)^T \|_2^2 + \sigma^2 \tau (p/m + 1) + \sigma^2 \tau \sum_{\lambda_j \le \tx/ (b n)} \lambda_j / \tx } ,
\end{align*}
which suggests that by ensembling more models, we can expand the model capacity further, resulting in a greater decrease in OOD excess risk.
\end{remark}

\paragraph{Simulations.} We utilize multiple numerical simulations to demonstrate the advantages of enhanced hidden dimensions and ensemble methods for OOD generalization, as depicted in Figure~\ref{fig:loss_decrease}. For clarity, we conduct four simulations, each with $40$ training samples and $1000$ test samples. The data dimension is set to $p = 40$. In these simulations, we consider two types of distribution on $x$, i.e, $\mathcal{N}(0, \varSigma_1)$ and $\mathcal{N}(0, \varSigma_2)$, where $\varSigma_1$ has eigenvalues as $\lambda_1 = 1$ and $\lambda_2 = \cdots = \lambda_p = 0.25$, and $\Sigma_2$ has eigenvalues as $\lambda_i = i^{-5/12}$. The ground truth models are defined as $g_1(x) = \beta^T x$ and $g_2(x) = \log(1 + e^{\beta^T x})$, where $\| \beta \|_2 = 1$. We introduce data noise $\epsilon \sim \mathcal{N}(0, 0.005^2)$ and OOD perturbation $\delta \sim \mathcal{N}(0, 4)$ into the simulations. In each simulation, corresponding to various feature numbers $m$, we iterate the experiment $500$ times and compute the average $\mathcal{L}_2$ loss. The results presented in Figure~\ref{fig:loss_decrease} show that: (i). as the ID loss reaches a satisfactory level, the associated OOD loss tends to be large; (ii). increasing the hidden dimension or employing ensemble models leads to a reduction in OOD loss; (iii). with the escalation of hidden dimension $m$, the enhancement in OOD performance from a single model to an ensemble model becomes less pronounced, moreover, when $m$ is sufficiently large, this enhancement becomes marginal. These observations are consistent with the findings outlined in Theorem~\ref{thm_relu}.


\begin{figure}[h]
     \centering
     \begin{subfigure}[b]{0.48\linewidth}
         \centering         
         \includegraphics[width=\linewidth]{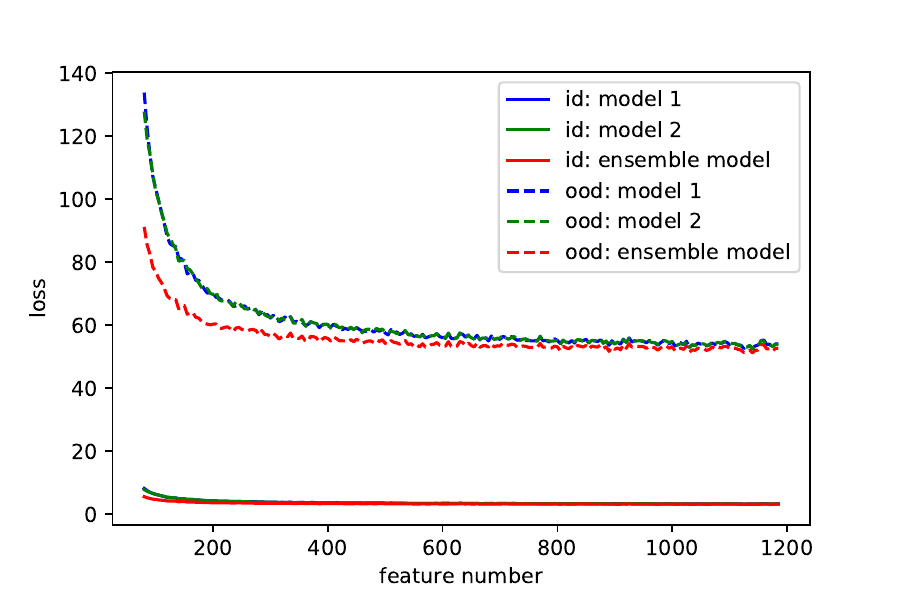}
         \caption{$x \sim \mathcal{N}(0, \varSigma_1), y = \beta^Tx + \epsilon$.}
     \end{subfigure}
     \hspace{1em}
     \begin{subfigure}[b]{0.48\linewidth}
         \centering         
         \includegraphics[width=\linewidth]{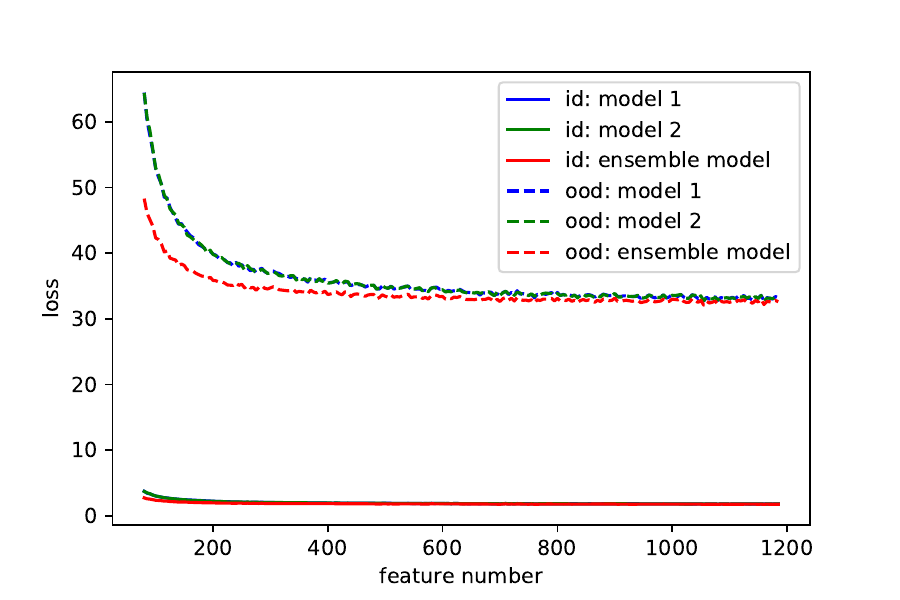}
         \caption{$x \sim \mathcal{N}(0, \varSigma_1), y = \log(1 + e^{\beta^T x}) + \epsilon$.}
     \end{subfigure}
     \hspace{1em}
     \begin{subfigure}[b]{0.48\linewidth}
         \centering         
         \includegraphics[width=\linewidth]{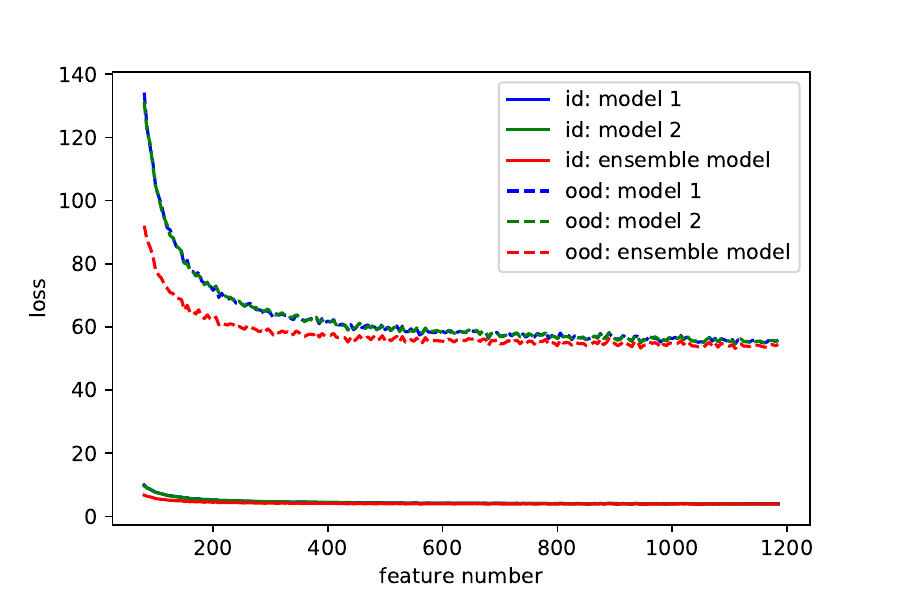}
         \caption{$x \sim \mathcal{N}(0, \varSigma_2), y = \beta^T x + \epsilon$.}
     \end{subfigure}
     \hspace{1em}
     \begin{subfigure}[b]{0.48\linewidth}
         \centering         
         \includegraphics[width=\linewidth]{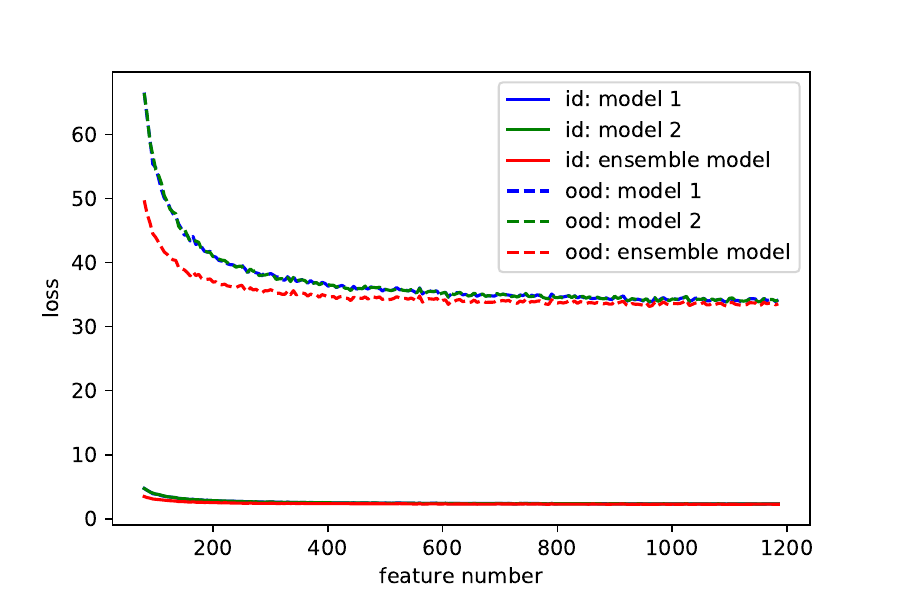}
         \caption{$x \sim \mathcal{N}(0, \varSigma_2), y = \log(1 + e^{\beta^T x}) + \epsilon$.}
     \end{subfigure}

     \vspace{5pt}
    \raggedright
    \footnotesize
    \emph{Notes.} Here solid lines represent ID losses, while dashed lines represent OOD losses. The blue and green lines correspond to results from two individual models, whereas the red lines pertain to results from the ensemble model.
    \caption{Loss decreasing.}
     \label{fig:loss_decrease}
\end{figure}    

\section{Overview of Proof Technique}\label{sec:pf}

The proof sketches for Theorem~\ref{thm_idood} and Theorem~\ref{thm_relu} are summarized in this section. For simplify, we use $c_i$ and $c'_i$ to denote positive constants that only depend on $\sigma_x, b , \xi, \varrho$. 

First, we recall the decomposition $\varSigma = \sum_i \lambda_i e_i e_i^T$ and obtain that
\begin{equation*}
    XX^T = \sum_i \lambda_i z_i z_i^T, \quad X \varSigma X^T = \sum_i \lambda_i^2 z_i z_i^T,
\end{equation*}
in which
\begin{equation*}
  z_i := \frac1{\sqrt{\lambda_i}} X e_i
\end{equation*}
are independent $\sigma_x$-subgaussian random vectors in $\rR^n$ with mean $0$ and covariance $I$. Then we will take the following notations in following analysis: 
\begin{equation*}
    A = XX^T, \quad A_k = \sum_{i > k} \lambda_i z_i z_i^T, \quad A_{-k} = \sum_{i \ne k} \lambda_i z_i z_i^T.
\end{equation*}

\subsection{Technical Lemmas}
Before presenting the proof sketches for the main theorems, we outline the key technical lemmas that are employed in our analysis. The proofs are in Appendix~\ref{lem}. 

\begin{lemma}[Refinement of Theorem 2.1 in \citealp{el2010spectrum}]\label{lemf:spectrum}
Let we assume that we observe n i.i.d.~random vectors, $x_i \in \rR^p$. Let us consider the kernel matrix $K$ with entries
\begin{equation*}
    K_{i,j} = f(\frac{x_i^T x_j}{\tx}).
\end{equation*}
We assume that:
\begin{enumerate}
\item  $n,p, \tx$ satisfy Assumption~\ref{cond:benign} and \ref{cond:high-dim};
\item  $\varSigma$ is a positive-define $p \times p$ matrix, and $\| \varSigma \|_2$ remains bounded;
\item $x_i = \varSigma^{1/2} \eta_i$, in which $\eta_i, i = 1, \dots, n$ are $\sigma_x$-subgaussian i.i.d.~random vectors with $\mathbb{E} \eta_i = 0$ and $\mathbb{E} \eta_i \eta_i^T = I_p$;
\item $f$ is a $C^1$ function in a neighborhood of $1$ and a $C^3$ function in a neighborhood of $0$.
\end{enumerate}
Under these assumptions, the kernel matrix $K$ can in probability be approximated consistently in operator norm, when $p$ and $n$ tend to $\infty$, by the kernel $\tilde{k}$, where
\begin{align*}
& \tilde{K} = \left( f(0) + f''(0) \frac{\mathrm{tr}(\varSigma^2)}{2 \tx^2} \right) 1 1^T + \frac{f'(0)}{\tx} XX^T + v_p I_n,\\
& v_p = f(1) - f(0) - f'(0).
\end{align*}
In other words, with probability at least $1 - 4 n^2 e^{- n^{1/8} / 2}$,
\begin{equation*}
    \| K - \tilde{K} \|_2 \le o(n^{- 1 / 16}). 
\end{equation*}
\end{lemma}

\begin{lemma}\label{lemf:feature_ortho}
Assume $w_1, \dots, w_m$ are sampled i.i.d. from $\mathcal{N}(1, 1/p I_p)$, then with probability at least $1 - 2 e^{- n^{\xi/2} / 4}$, we have
\begin{equation*}
    \mathbb{P} \left( \mid w_i^T \varSigma w_j - \mathbb{E}(w_i^T \varSigma w_j) \mid \right) \le  \frac{\tx}{p n^{(2 + \xi )/4}}, \quad \forall i, j = 1, \dots, m.
\end{equation*}
\end{lemma}

\begin{lemma}\label{lemf:matrix_mean}
Assume $z \in \mathbb{R}^q$ is a $q$-dim sub-gaussian random vector with parameter $\sigma$, and $\mathbb{E}[z] = \mu$. Here are $n$ i.i.d.~samples $z_1, \dots, z_n$, which have the same distribution as $z$, then we can obtain that with probability at least $1 - 4 e^{- \sqrt{n}}$, 
\begin{equation*}
    \| \mathbb{E} zz^T - \frac{1}{n} \sum_{i=1}^n z_i z_i^T \|_2 \le \| \mathbb{E} zz^T \|_2 \max \{ \sqrt{\frac{\text{trace}(\mathbb{E} zz^T)}{n}}, \frac{\text{trace}(\mathbb{E} zz^T)}{n}, \frac{1}{n^{1/4}} \} + 2 \sqrt{2} \frac{\sigma \| \mu \|_2}{n^{1/4}}.
\end{equation*}
\end{lemma}

\begin{lemma}[Lemma 10 in \citealp{bartlett2020benign}]\label{lemf_eigen}
 There are constants $b, c'_1 \ge 1$ such that, for any $k \ge 0$, with probability at least $1 - 2 e^{- \frac{n}{c'_1}}$,
\begin{enumerate}
\item for all $i \ge 1$,
\begin{equation*}
    \mu_{k+1}(A_{-i}) \le \mu_{k+1}(A) \le \mu_{1}(A_{k}) \le c'_2 (\sum_{j > k} \lambda_j + \lambda_{k+1} n);
\end{equation*}
\item for all $1 \le i \le k$,
\begin{equation*}
    \mu_{n}(A) \ge \mu_{n}(A_{-i}) \ge \mu_{n}(A_{k}) \ge \frac{1}{c'_2} \sum_{j > k} \lambda_j - c'_2 \lambda_{k+1} n;
\end{equation*}
\item if $r_k \ge bn$, then
\begin{equation*}
    \frac{1}{c'_2} \lambda_{k+1} r_k \le \mu_n (A_k) \le \mu_1 (A_k) \le c'_2 \lambda_{k+1} r_k ,
\end{equation*}
\end{enumerate}
where $c'_2 > 1$ is a constant only depending on $b, \sigma_x$.
\end{lemma}

\begin{lemma}[Corollary 24 in \citealp{bartlett2020benign}]\label{lemf_subspacenorm}
 For any centered random vector $\bm{z} \in \mathbb{R}^n$ with independent $\sigma^2_x$ sub-Gaussian coordinates with unit variances, any $k$ dimensional random subspace $\mathscr{L}$ of $\mathbb{R}^n$ that is independent of $\bm{z}$, and any $t > 0$, with probability at least $1 - 3 e^{-t}$,
\begin{equation*}
\begin{aligned}
& \parallel \bm{z} \parallel^2 \le n + 2 (162e)^2 \sigma_x^2(t + \sqrt{nt}),\\
& \parallel \bm{\Pi}_{\mathscr{L}} \bm{z} \parallel^2 \ge n - 2 (162e)^2 \sigma_x^2 (k + t + \sqrt{nt}),
\end{aligned}
\end{equation*}
where $\bm{\Pi}_{\mathscr{L}}$ is the orthogonal projection on $\mathscr{L}$.
\end{lemma}

\subsection{Proof Sketch for Theorem~\ref{thm_idood}}

The proof mainly contains three steps as follows.

\paragraph{Step 1 : kernel matrix linearizatioin.}

With kernel estimation results of Lemmas~\ref{lem:ntk1} and \ref{lem:ntk2} in \citet{jacot2018neural}, we could approximate kernel matrix $K = \Phi_{W_r} \Phi_{W_r}^T \in \mathbb{R}^{n \times n}$ with $r = 1, 2$ as
\begin{align*}
 K_{s,t} &= \left( 1 + O_p(\frac{1}{\sqrt{m}}) \right) \frac{1}{p} \left[  \frac{x_s^T x_t }{2 \pi} \arccos\left( - \frac{x_s^T x_t}{\| x_s \|_2 \| x_t \|_2} \right) + \frac{ \| x_s \|_2 \| x_t \|_2 }{2 \pi} \sqrt{1 -\left( - \frac{x_s^T x_t}{\| x_s \|_2 \| x_t \|_2} \right)^2 } \right],  
\end{align*}
for any $s, t = 1, \dots, n$. Furthermore, with Assumption~\ref{cond:benign} and Assumption~\ref{cond:high-dim}, we could use kernel linearization techniques ( Lemma~\ref{lemf:spectrum}) to approximate $K$ by  $\tilde{K}$:
\begin{equation*}
     \tilde{K} = \frac{\tx}{p} (\frac{1}{2 \pi} + \frac{3 r_0(\varSigma^2)}{4 \pi \tx^2}) 11^T + \frac{1}{4 p} XX^T + \frac{\tx}{p} (\frac{1}{4} - \frac{1}{2 \pi}) I_n.
\end{equation*}

\paragraph{Step 2: Upper bound for ID excess risk.}
For simplicity, here we just take analysis on $f_{W_1}(\hat{\theta}(W_1), x)$, and the analysis on $f_{W_2}(\hat{\theta}(W_2), x)$ is similar. The excess ID risk could be decomposed as
\begin{align*}
& \quad \mathcal{L}_{\mathrm{id}}(f_{W_1}(\hat{\theta}(W_1), x))\\
&= \mathbb{E}_{x,y} [\frac{1}{\sqrt{m}} \phi(x^T W_1) (\hat{\theta}(W_1) - \theta^*(W_1))]^2 \\
 &=  \underbrace{ \frac{1}{m} \theta(W_1)^{*T} [I - \Phi_{W_1}^T (\Phi_{W_1} \Phi_{W_1}^T)^{-1} \Phi_{W_1}] \mathbb{E}_x \phi(W_1^T x) \phi(W_1^T x)^T [I - \Phi_{W_1}^T (\Phi_{W_1} \Phi_{W_1}^T)^{-1} \Phi_{W_1}] \theta^*(W_1) }_{\ib} \\
 & \quad + \underbrace{ \frac{1}{m} (\sigma^2 + o(1)) \text{trace} \{ (\Phi_{W_1} \Phi_{W_1}^T)^{-2} \Phi_{W_1} \mathbb{E}_x \phi(W_1^T x) \phi(W_1^T x)^T \Phi_{W_1}^T \} }_{\iv}.
\end{align*}
For the bias term $\ib$, it could be expressed as
\begin{align*}
 \ib 
 &=  \theta^*(W_1) [I - \Phi_{W_1}^T (\Phi_{W_1} \Phi_{W_1}^T)^{-1} \Phi_{W_1}]\\
 & \quad \quad \quad \quad \quad \quad \quad \quad \left( \frac{1}{m} \mathbb{E}_x \phi(W_1^T x) \phi(W_1^T x)^T - \frac{1}{n} \Phi_{W_1}^T \Phi_{W_1} \right) [I - \Phi_{W_1}^T (\Phi_{W_1} \Phi_{W_1}^T)^{-1} \Phi_{W_1}] \theta^*(W_1) \\
 &\le \| \theta^*(W_1) \|_2^2 \| \frac{1}{m} \mathbb{E}_x \phi(W_1^T x) \phi(W_1^T x)^T - \frac{1}{n} \Phi_{W_1}^T \Phi_{W_1} \|_2,   
\end{align*}
where the inequality is induced from $\| I - \Phi_{W_1}^T (\Phi_{W_1} \Phi_{W_1}^T)^{-1} \Phi_{W_1} \|_2 \le 1$ and $a^T B a \le \| a \|_2^2 \| B \|_2$ for any positive-defined matrix $B$. With the bounded Lipschitz of ReLU function $\phi(\cdot)$, we could verrify that with a high probability, the random vector $\frac{1}{\sqrt{m}} \phi(W_1^T x)$ is $\sigma_x \sqrt{\tx / p}$-subgaussian with respect to $x$. Consider Lemma~\ref{lemf:feature_ortho} and Lemma~\ref{lemf:matrix_mean}, with a high probability, we have
\begin{equation*}
      \ib \le c_1 \frac{\| \theta^*(W_1) \|_2^2 }{ n^{1/4}} \frac{\tx}{p},
\end{equation*}
with some constant $c_1 > 0$.

For the variance term $\iv$, it could be expressed as
\begin{align*}
 \iv &= (\sigma^2 + o(1)) \text{trace} \{ (\Phi_{W_1} \Phi_{W_1}^T )^{-2} \Phi_{W_1} \mathbb{E}_x \frac{1}{m} \phi(W_1^T x) \phi(W_1^T x)^T \Phi_{W_1} \}\\
 &= \frac{\sigma^2 + o(1)}{n m} \mathbb{E}_{x_1', \dotsm x_n'} \sum_{i=1}^n \text{trace}\{ (\Phi_{W_1} \Phi_{W_1}^T)^{-2} \Phi_{W_1} \mathbb{E}_x \phi(W_1^T x'_i) \phi(W_1^T x'_i)^T \Phi_{W_1}^T \}\\
 &= \frac{\sigma^2 + o(1)}{n} \mathbb{E}_{x'_1, \dots, x'_n} \text{trace} \{ K^{-2} \Phi_{W_1} \Phi_{W_1}^{'T} \Phi_{W_1}' \Phi_{W_1}^T \},
\end{align*}
where we denote $x'_1, \dots, x'_n$ are $n$ i.i.d. samples from the same distribution as $x_1, \dots, x_n$, and $\Phi_{W_1}' = [\phi(W_1^T x'_1), \dots, \phi(W_1^T x'_n)]^T$. Similar to the linearized approximation on $K$, we could approximate $\Phi_{W_1} \Phi_{W_1}^{'T}$ and $ \Phi_{W_1}' \Phi_{W_1}^T$ as:
\begin{align*}
& \| \Phi_{W_1} \Phi_{W_1}^{'T} -  \frac{\tx}{p} \left( \frac{1}{2 \pi}  + \frac{3 r_0(\varSigma^2)}{4 \pi \tx^2} \right) 11^T + \frac{1}{4p} XX^{'T} \|_2  \le \frac{4 \tx}{p n^{1/16}},\\
 & \| \Phi_{W_1}' \Phi_{W_1}^{T} -  \frac{\tx}{p} \left( \frac{1}{2 \pi}  + \frac{3 r_0(\varSigma^2)}{4 \pi \tx^2} \right) 11^T + \frac{1}{4p} X'X^{T} \|_2  \le \frac{4 \tx}{p n^{1/16}},
\end{align*}
where $X' = [x'_1, \dots, x'_n]^T \in \mathbb{R}^{n \times p}$, and it implies that
\begin{equation*}
Q := \frac{1}{n} \mathbb{E}_{x'_1, \dots, x'_n} \Phi_{W_1} \Phi_{W_1}^{'T} \Phi_{W_1}' \Phi_{W_1}^T \prec \frac{\tx^2}{2 \pi^2 p^2} (1 + o(1)) 11^T + \frac{1}{8 p^2} X \varSigma X^T + \frac{32 \tx^2}{p^2 n^{9/8}} I_n.
\end{equation*}
Then according to several inequalities of matrix trace calculation in  Lemma~\ref{lem:matrix_comp}, we could approximate $\iv$ by $\sigma^2 \mathrm{tr}\{ \tilde{K}^{-2} Q \}$, and the upper bound for this term contains three parts as follows:
\begin{align*}
\text{trace} \{ \tilde{K}^{-2} Q \} 
&\le \underbrace{  \frac{\tx^2}{2 \pi^2 p^2} (1 + o(1))  1^T \left( \frac{\tx}{p} (\frac{1}{2 \pi} + \frac{3 r_0(\varSigma^2)}{4 \pi \tx^2}) 11^T + \frac{1}{4 p} XX^T + \frac{\tx}{p} (\frac{1}{4} - \frac{1}{2 \pi}) I_n \right)^{-2} 1}_{\iv_1} \\
& \quad + \underbrace{ \text{trace} \left\{ \left( \frac{1}{4p} XX^T + \frac{\tx}{p} (\frac{1}{4} - \frac{1}{2 \pi}) I_n \right)^{-2} \left( \frac{1}{8 p^2} X \varSigma X^T  \right) \right\} }_{\iv_2}\\
 & \quad + \underbrace{ \text{trace} \left\{ \left( \frac{1}{4p} XX^T + \frac{\tx}{p} (\frac{1}{4} - \frac{1}{2 \pi}) I_n \right)^{-2} \left(  \frac{32 \tx^2}{p^{2} n^{9/8}} I_n \right) \right\} }_{\iv_3}.
\end{align*}
With Woodbury identity, we could upper bound the first term $\iv_1$ as
\begin{equation*}
    \iv_1 = \frac{\frac{\tx^2}{2 \pi^2 p^2} (1 + o(1)) 1^T \tilde{R}^{-2} 1}{(1 + \frac{\tx}{p} (\frac{1}{2 \pi} + \frac{3 r_0(\varSigma^2)}{4 \pi \tx^2}) 1^T \tilde{R}^{-1} 1)^2}  \le \frac{2 1^T \tilde{R}^{-2} 1}{(1^T \tilde{R}^{-1} 1)^2} \le \frac{2 n / \lambda_n(\tilde{R})^2}{n^2 / \lambda_1(\tilde{R})^2},
\end{equation*}
where we denote 
\begin{equation*}
    \tilde{R} := \frac{1}{4 p} XX^T + \frac{\tx}{p}(\frac{1}{4} - \frac{1}{2 \pi}) I_n.
\end{equation*}
With Assumption~\ref{cond:benign} and Assumption~\ref{cond:high-dim}, recalling the bounds for matrix eigenvalues ( Lemma~\ref{lemf_eigen}), with a high probability, we have
\begin{equation*}
\frac{c_2 \tx}{p} \le \lambda_n ( \tilde{R})  \le  \lambda_1 (\tilde{R}) \le \frac{c_3 (\tx + n)}{p},
\end{equation*}
with some constants $c_2, c_3 > 0$. Then
combining with the fact that
\begin{equation*}
    1^T \tilde{R}^{-2} 1 \le n \lambda_1(\tilde{R}^{-1})^2 = n / \lambda_n(\tilde{R})^2, \quad 1^T \tilde{R}^{-1} 1 \ge n \lambda_n(\tilde{R}^{-1}) = n / \lambda_1(\tilde{R}),
\end{equation*}
we could obtain that
\begin{equation*}
    \iv_1 \le \frac{c_4}{n^{1/2}},
\end{equation*}
with some positive constant $c_4$.

Then we turn to term $\iv_2$. Consider the bounds for matrix eigenvalues (Lemma~\ref{lemf_eigen}) and the results in Lemma~\ref{lem_bart} in \citet{bartlett2020benign}, with a high probability, we could obtain that
\begin{align*}
\iv_2 
&= 2 \text{trace} \{ (XX^T +  \tx (1 - 2 / \pi) I_n )^{-2} X \varSigma X^T \} \le  c_5 \frac{k^*}{n} + \frac{c_6 n \sum_{i > k^*} \lambda_i^2}{\tx^2},
\end{align*}
with some constants $c_5, c_6 > 0$.

And the last term $\iv_3$ could be upper bounded as
\begin{equation*}
    \iv_3 = \frac{c_7 \tx^2}{n^{9/8}} \text{trace} \{ (XX^T + \tx (1 - 2/\pi) I_n)^{-2} \} \le \frac{c_7 \tx^2}{n^{9/8}} \frac{n}{\mu_n (XX^T + \tx(1 - 2/\pi) I_n)^2} \le \frac{c_8}{n^{1/8}},
\end{equation*}
where the last inequality is due to $\mu_{n}(XX^T + \tx (1 - 2/\pi)) \ge \tx(1 - 2/\pi)$. And combing all of the estimation for $\ib, \iv_1, \iv_2$ and $\iv_3$, we could get the upper bound for ID excess risk.

\paragraph{Step 3: Lower bound for OOD risk.}
Similarly, for the OOD risk, we take the following decomposition first:
\begin{align*}
& \quad \mathcal{L}_{\mathrm{ood}}(f_{W_1}(\hat{\theta}(W_1), x)) \\
&= \max_{\Delta \in \Xi_{\mathrm{ood}}} \mathbb{E}_{x,\delta,y} \left[ \frac{1}{\sqrt{m}} \phi((x + \delta)^T W_1) (\hat{\theta}(W_1) - \theta^*(W_1))  \right]^2\\
&= \mathcal{L}_{\mathrm{id}}(f_{W_1}(\hat{\theta}(W_1), x)) + \frac{1}{m} \max_{\Delta \in \Xi_{\mathrm{ood}}} \mathbb{E}_{x,\delta, \epsilon} [\delta^T\nabla_x \phi(W_1^T x) (\hat{\theta}(W_1) - \theta^*(W_1))]^2\\
&= \mathcal{L}_{\mathrm{id}}(f_{W_1}(\hat{\theta}(W_1), x)) + \max_{\Delta \in \Xi_{\mathrm{ood}}} \left\{ \ob + \ov \right\},
 \end{align*}
where we take first-order Taylor expansion with respect to $x$ and:
\small{
\begin{align*}
& \ob = \frac{1}{m} \theta^{*T}(W_1)[I - \Phi_{W_1}^T (\Phi_{W_1} \Phi_{W_1}^T)^{-1} \Phi_{W_1}] \mathbb{E}_x \nabla_x \phi(W_1^T x)^T \varSigma_{\delta} \nabla_x \phi(W_1^T x) [I - \Phi_{W_1}^T (\Phi_{W_1} \Phi_{W_1}^T)^{-1} \Phi_{W_1}] \theta^*(W_1), \\
& \ov = (\sigma^2 + o(1)) \frac{1}{m} \text{trace} \{ (\Phi_{W_1} \Phi_{W_1}^T)^{-2} \Phi_{W_1} \mathbb{E}_x \nabla_x \phi(W_1^T x)^T \varSigma_{\delta} \nabla_x \phi(W_1^T x) \Phi_{W_1}^T\}.
 \end{align*}
 }
The ID risk $ \mathcal{L}_{\mathrm{id}}(f_{W_1}(\hat{\theta}(W_1), x)) = o(1)$, and $\ob$ is related to the relationship between ground truth model $g(x)$ and $\delta$. So to focus on the impact of overfitting process, we could just focus on $\ov$ to obtain a lower bound.
With Assumption~\ref{cond:benign} and Assumption~\ref{cond:high-dim}, we could approximate $\ov$ as
\begin{equation*}
\sigma^2 \frac{\mathrm{tr}\{ \varSigma_{\delta}\}}{2 p m } \mathrm{tr} \{ K^{-1} \} + \sigma^2 \frac{1}{16 p^2} \mathrm{tr} \{ K^{-2} X \varSigma_{\delta} X^T \},
\end{equation*}
which could be further estimated as
\begin{equation*}
    \sigma^2 \frac{\mathrm{tr}\{ \varSigma_{\delta}\}}{2 p m } \mathrm{tr} \{ \tilde{K}^{-1} \} + \sigma^2 \frac{1}{16 p^2} \mathrm{tr} \{ \tilde{K}^{-2} X \varSigma_{\delta} X^T \}.
\end{equation*}
For the term $\mathrm{tr}\{ \tilde{K}^{-1} \}$, we could approximate it as $4p \mathrm{tr}\{ (XX^T + l (1 - 2/\pi) 
 I_n)^{-1} \}$ and
use Woodbury identity to take lower bound
\begin{align*}
     \mathrm{tr}\{ (XX^T + \tx(1 - 2/\pi) I)^{-1} \} &= \mathrm{tr} \{ (A_{-1} + \tx(1 - 2/\pi) I)^{-1} \} - \frac{\lambda_1 z_1^T (A_{-1} + \tx(1 - 2/\pi I))^{-2} z_1}{1 + \lambda_1 z_1^T (A_{-1} + \tx(1 - 2/\pi) I)^{-1} z_1}\\
 &= \mathrm{tr}\{ A_{k^*} + \tx(1 - 2/\pi) I)^{-1} \} - \sum_{i=1}^{k^*} \frac{\lambda_i z_i^T (A_{i} + \tx(1 - 2/\pi) I)^{-2} z_i}{1 + \lambda_i z_i^T (A_{i} + \tx(1 - 2/\pi) I)^{-1} z_i}\\
 &\ge \mathrm{tr}\{ A_{k^*} + \tx(1 - 2/\pi) I)^{-1} \} - \sum_{i=1}^{k^*} \frac{z_i^T (A_{i} + \tx(1 - 2/\pi) I)^{-2} z_i}{ z_i^T (A_{i} + \tx(1 - 2/\pi)I)^{-1} z_i},
 \end{align*}
then with the bounds of eigenvalues and random vectors (Lemma~\ref{lemf_eigen} and Lemma~\ref{lemf_subspacenorm}), we could control the norm of $z_i$ and the eigenvalues of $A_i$, which induces the lower bound as
\begin{equation*}
    \mathrm{tr}\{ \tilde{K}^{-1} \} \ge c_8 \frac{np}{\tx},
\end{equation*}
with some constant $c_8 > 0$.
And its upper bound could be estimated as
\begin{equation*}
    \mathrm{tr}\{ \tilde{K}^{-1} \} \le \frac{n}{\mu_n(\tilde{K})} \le c_9 \frac{np}{\tx}.
\end{equation*}
Then we turn to the approximation for the term 
\begin{equation*}
    \mathrm{tr}\{ \tilde{K}^{-2} X \varSigma_{\delta} X^T \} \approx 16 p^2 \mathrm{tr} \{ (XX^T + \tx(1 - 2/\pi) I_n)^{-2} X \varSigma_{\delta} X^T \},
\end{equation*}
and the analysis is similar to the process on $\iv_2$. To be specific, with Woodbury identity, we could express this term as
\begin{align*}
\mathrm{tr} \{ (XX^T + \tx(1 - 2/\pi) I)^{-2} X \varSigma_{\delta} X^T \}
&= \sum_i \lambda_i \alpha_i z_i^T (XX^T + \tx(1 - 2/\pi) I)^{-2} z_i \\
&= \sum_i \frac{\lambda_i \alpha_i z_i^T(A_{-i} + \tx(1 - 2/\pi) I)^{-2} z_i}{(1 + \lambda_i z_i^T (A_{-i} + \tx(1 - 2/\pi) I)^{-1} z_i)^2}.
\end{align*}
With Lemma~\ref{lemf_eigen} and Lemma~\ref{lemf_subspacenorm}, we control the norm of $z_i$, as well as the eigenvalues of $A_{-i}$, then consider Assumption~\ref{cond:benign} and Assumption~\ref{cond:high-dim}, with a high probability, the whole term $\mathrm{tr}\{ \tilde{K}^{-2} X \varSigma_{\delta} X^T \}$ could be bounded as
\begin{align*}
&  \mathrm{tr} \{ \tilde{K}^{-2} X \varSigma_{\delta} X^T \} 
  \ge c_{10}  p^2 \left( \sum_{\lambda_i > \tx / (bn)} \frac{\alpha_i}{n \lambda_i} +  \frac{n \sum_{\lambda_j \le \tx/ (b n)} \lambda_j \alpha_j}{\tx^2} \right) ,\\
  & \mathrm{tr} \{ \tilde{K}^{-2} X \varSigma_{\delta} X^T \} 
   \le c_{11} p^2 \left(  \sum_{\lambda_i > \tx/(bn)}  \frac{\alpha_i}{n \lambda_i} +  \frac{n \sum_{\lambda_i \le \tx / (bn)} \lambda_i \alpha_i }{\tx^2 } \right).
\end{align*}
After obtaining the upper and lower bounds for $\ob$, $\mathrm{tr}\{ \tilde{K}^{-1} \}$ and $\mathrm{tr} \{ \tilde{K}^{-2} X \varSigma_{\delta} X^T \}$, to further estimate OOD risk, we provide bounds for $\mathrm{tr}\{ \varSigma_{\delta} \}$ here:
\begin{align*}
& \mathrm{tr} \{ \varSigma_{\delta} \} = \sum_i  \alpha_i \le \tau k^* + \tau \frac{p\tx}{n} \le \tau \lambda_1 n + 
\tau \frac{p\tx}{n} \le 2 \tau \frac{p\tx}{n},\\
& \mathrm{tr} \{ \varSigma_{\delta} \} \ge \sum_{\lambda_i \le \tx/(bn)} \alpha_i \ge \tau \frac{\tx (p - k^*)}{n} \ge \tau \frac{p \tx }{2n}, \quad \text{while} \quad \alpha_i = \tau \frac{\tx}{n} \quad  \forall i \quad \text{s.t.} \quad \lambda_i \le \frac{\tx}{bn}.
\end{align*}
Summarizing all of the results above, we could finish the proof for OOD risk. 

\subsection{Proof Sketch for Theorem~\ref{thm_relu}}
The proofs for Theorem~\ref{thm_relu} are similar to the analysis in Theorem~\ref{thm_idood}, which contains the following two steps:

\paragraph{Step 1 : Upper bound for OOD risk.}

First, with Eq.~\eqref{eq:est_ability}, we could upper bound the term $\ob$ as
\begin{align*}
 \ob & \le \frac{1}{m} \mathbb{E} [\phi(W_1^T x)^T \theta^*(W_1) - \phi(W_1^T (x + \delta))^T \theta^*(W_1) ]^2 \\
&= \mathbb{E} [\frac{1}{\sqrt{m}} \phi(W_1^T x)^T \theta^*(W_1) - g(x) + g(x) - g(x + \delta) + g(x + \delta) - \frac{1}{\sqrt{m}} \phi(W_1^T(x+\delta))^T \theta^*(W_1)]^2 \\
 &\le 4 \mathbb{E}[\frac{1}{\sqrt{m}} \phi(W_1^T x) \theta^*(W_1) - g(x)]^2 + 4 \mathbb{E} [g(x+\delta) - \frac{1}{\sqrt{m}} \phi(W_1^T(x+\delta))^T \theta^*(W_1)]^2 + 4 \mathbb{E}[g(x) - g(x + \delta)]^2 \\
 &= 8 \varrho^2 + 4 \mathbb{E}[\nabla_x g(x)^T \delta]^2 \le 8 \varrho^2 + 4 \tau \mathbb{E}\| \nabla_x g(x)^T \|_2^2,
 \end{align*}
which is up to $O(1)$, due to the assumptions on $\| \nabla_x g(x) \|_2$ and $\delta$. And the upper bound for $\ov$, as well as $\mathrm{tr}\{ K^{-1} \}$ and $\mathrm{tr} \{ K^{-2} X \varSigma_{\delta} X^T\}$, has been established in the proof of Theorem~\ref{thm_idood}. Summarie all of these estimations and Assumption~\ref{ass:ood1}, \ref{ass:ood2} about $\Xi_{\mathrm{ood}}$, we could obtain an upper bound for $\mathcal{L}_{\mathrm{ood}}(f_r(\hat{\theta}_r, x))$.

\paragraph{Step 2: Proof sketch for ensemble model.} The OOD risk on ensemble model could be decomposed as
\begin{align*}
& \quad \mathcal{L}_{\mathrm{ood}}(f_{\mathrm{ens}}(\hat{\theta}(W_1), \hat{\theta}(W_2),x))\\
& = \max_{\Delta \in \Xi_{\mathrm{ood}}}  \mathbb{E}_{x,y, \delta} \left[\frac{1}{2 \sqrt{m}} \left( \phi( (x + \delta)^T W_1) ( \hat{\theta}(W_1) -  \theta^*(W_1)) + \phi( (x + \delta)^T W_2) ( \hat{\theta}(W_2) -  \theta^*(W_2)) \right) \right]^2 \\
& = \max_{\Delta \in \Xi_{\mathrm{ood}}}  \mathbb{E}_{x,y, \delta} \left\{ \text{term 1} + \text{term 2} \right\},
\end{align*}
where term 1 is corresponding to ID risk, so we could obtain
\begin{equation*}
\begin{aligned}
\text{term 1} & = \mathbb{E}_{x, y} \left[ \frac{1}{2 \sqrt{m}} \phi(x^T W_1) (\hat{\theta}(W_1) - \theta^*(W_1)) + \frac{1}{2} \phi(x^T W_2) (\hat{\theta}(W_2) - \theta^*(W_2)) \right]^2 \\
&\le \frac{1}{2 m} \mathbb{E}_{x, y} \left[ \phi(x^T W_1) (\hat{\theta}(W_1) - \theta^*(W_1) ) \right]^2 + \frac{1}{2m} \mathbb{E}_{x, y} \left[ \phi(x^T W_2) (\hat{\theta}(W_2) - \theta^*(W_2) ) \right]^2\\
& = \frac{1}{2} \left( \mathcal{L}_{\mathrm{id}} (f_{W_1}(\hat{\theta}(W_1), x)) + \mathcal{L}_{\mathrm{id}} (f_{W_2}(\hat{\theta}(W_2), x)) \right) \to 0,    
\end{aligned}
\end{equation*}
and term 2 can be approximated by
\begin{align*}
\text{term 2} &= \mathbb{E}_{x, y, \delta} \left[ \frac{1}{2 \sqrt{m}} (\delta^T\nabla_x \phi(W_1^T x) (\hat{\theta}(W_1) - \theta^*(W_1)) + \delta^T\nabla_x \phi(W_2^T x) (\hat{\theta}(W_2) - \theta^*(W_2))) \right]^2 \\
&= \underbrace{ \frac{1}{4} \mathbb{E}_{\delta, x} \left[ \delta^T \nabla_x \phi(W_1^T x) [I - \Phi_{W_1}^T (\Phi_{W_1} \Phi_{W_1}^T)^{-1} \Phi_{W_1}] \theta^*(W_1) + \delta^T \nabla_x \phi(W_2^T x) [I - \Phi_{W_2}^T (\Phi_{W_2} \Phi_{W_2}^T)^{-1} \Phi_{W_2}] \theta^*(W_2) \right]^2 }_{\text{term 2.1}} \\
& \quad + \underbrace{ (\sigma^2 + o(1)) \frac{1}{4m} \text{trace} \{ (\Phi_{W_1} \Phi_{W_1}^T)^{-2} \Phi_{W_1} \mathbb{E}_x \nabla_x \phi(W_1^T x)^T \varSigma_{\delta} \nabla_x \phi(W_1^T x) \Phi_{W_1}^T\}}_{\text{term 2.2}} \\
& \quad + \underbrace{ (\sigma^2 + o(1)) \frac{1}{4m}  \text{trace} \{ (\Phi_{W_2} \Phi_{W_2}^T)^{-2} \Phi_{W_2} \mathbb{E}_x \nabla_x \phi(W_2^T x)^T \varSigma_{\delta} \nabla_x \phi(W_2^T x) \Phi_{W_2}^T\}}_{\text{term 2.3}} \\
& \quad + \underbrace{ (\sigma^2 + o(1)) \frac{1}{2m}  \text{trace} \{ (\Phi_{W_1} \Phi_{W_1}^T)^{-1} (\Phi_{W_2} \Phi_{W_2}^T)^{-1} \Phi_{W_1} \mathbb{E}_x \nabla_x \phi(W_1^T x)^T \varSigma_{\delta} \nabla_x \phi(W_2^T x) \Phi_{W_2}^T\}}_{\text{term 2.4}}.
\end{align*}
Term 2.1 is related to the average of bias term $\ob$, to be specific,
\begin{align*}
    \text{term 2.1} &\le \frac{1}{2} \left(\frac{1}{m} \mathbb{E} [\phi(W_1^T x)^T \theta^*(W_1) - \phi(W_1^T (x + \delta))^T \theta^*(W_1) ]^2 + \frac{1}{m} \mathbb{E} [\phi(W_2^T x)^T \theta^*(W_2) - \phi(W_2^T (x + \delta))^T \theta^*(W_2) ]^2\right)  \\
    & = \frac{1}{2} \left( \ob(f_{W_1}) + \ob(f_{W_2}) \right) \le c_{12} \tau \mathbb{E}_x \| \nabla_x g(x)^T \|_2^2 + o(1),
\end{align*}
and for the other three terms, similar to the analysis on $\ov$, with Assumption~\ref{cond:benign} and Assumption~\ref{cond:high-dim}, we could approximate term 2.2 $+$ term 2.3 as
\begin{equation*}
    \sigma^2 \frac{\mathrm{tr}\{ \varSigma_{\delta}\}}{4 p m } \mathrm{tr} \{ K^{-1} \} + \sigma^2 \frac{1}{32 p^2} \mathrm{tr} \{ K^{-2} X \varSigma_{\delta} X^T \},
\end{equation*}
and term 2.4 could be approximated as
\begin{equation*}
    \sigma^2 \frac{1}{32 p^2} \mathrm{tr} \{ K^{-2} X \varSigma_{\delta} X^T \}.
\end{equation*}
The analysis above shows that we could estimate the difference between ensemble model OOD risk and single model OOD risk as
\begin{align*}
 & \quad \mathcal{L}_{\mathrm{ood}} (f_{W_1}(\hat{\theta}(W_1), x)) + \mathcal{L}_{\mathrm{ood}} (f_{W_2}(\hat{\theta}(W_2), x)))/2 - \mathcal{L}_{\mathrm{ood}} (f_{\mathrm{ens}} (\hat{\theta}(W_1), \hat{\theta}(W_2), x) \\
 & \approx \max_{\Delta \in \Xi_{\mathrm{ood}}} \sigma^2 \frac{\mathrm{tr}\{ \varSigma_{\delta}\}}{4 p m } \mathrm{tr} \{ K^{-1} \} \ge \frac{\tau \tx}{8 m n} \mathrm{tr}\{ K^{-1} \},  
\end{align*}
then with the upper and lower bounds for $\ob$, $\mathrm{tr}\{ \tilde{K}^{-1} \}$ and $\mathrm{tr} \{ \tilde{K}^{-2} X \varSigma_{\delta} X^T \} $, we could finish the proof for Theorem~\ref{thm_relu}.

\section{Conclusion}\label{sec:conclu}
In this study, we investigate the impact of over-parameterization on OOD loss. Surprisingly, we find that increasing over-parameterization can actually improve generalization under significant shifts in natural distributions. This discovery is unexpected because it demonstrates that the impact of over-parameterization on natural shifts differs significantly from its impact on adversarial examples. While increased over-parameterization exacerbates a model's susceptibility to adversarial attacks, it actually proves beneficial in the context of natural shifts.

\bibliography{ref.bib}
\bibliographystyle{apalike}

\newpage

\appendix

\section{Details on Example~\ref{eg1} and Example~\ref{eg2}}\label{pf:eg}
Here we restate the two examples and provide detailed calculations.
\subsection{Details on Example~\ref{eg1}}
 Suppose the eigenvalues as
\begin{equation*}
\lambda_k = \left\{
\begin{aligned}
& 1, k = 1, \\
& \frac{1}{n^{21/5}} \frac{1 + s^2 - 2 s \cos(k \pi / (p + 1))}{1 + s^2 - 2 s \cos(\pi / (p + 1))}, \quad 2 \le k \le p, \\
& 0, \text{otherwise},
\end{aligned}
\right.
\end{equation*}
where $p = n^5$. As it is easy to verified that $k^* = 1$, we could obtain that
\begin{align*}
& \frac{k^*}{n} = \frac{1}{n} \to 0,\\
& \frac{\tx}{n} \le \frac{1 + p (1 + s)^2 / n^{21/5}}{n} \le \frac{2 (1 + s)^2}{n^{1/5}} \to 0,\\
& \frac{n^{3/4}}{\tx} \le \frac{n^{3/4}}{1 + p(1 - s)^2 / n^{21/5}} \le \frac{1}{2 (1 - s)^2 n^{1/20}} \to 0,\\
& \frac{n^{1 + \xi} \sum_i \lambda_i^2}{\tx^2} \le \frac{n^{1 + \xi} (1 + p (1 + s)^4 / n^{21/5} )}{(1 + p (1 - s)^2 / n^{21/5})^2} \le \frac{2}{n^{4 - \xi}} \to 0, \quad \forall 0 < \xi < 4.
\end{align*}

\subsection{Details on Example~\ref{eg2}}

 Suppose the eigenvalues as
 \begin{equation*}
     \lambda_k = k^{- 5/6}, \quad 1 \le k < p,
 \end{equation*}
where $p = n^5$. As we could verify that $k^* = n^{1/5}$, we could obtain that
\begin{align*}
& \frac{k^*}{n} = \frac{1}{n^{4/5}} \to 0,\\
& \frac{\tx}{n} \le \frac{12 n^{5/6}}{n} \le \frac{12}{n^{1/6}} \to 0,\\
& \frac{n^{3/4}}{\tx} \le \frac{n^{3/4}}{3 n^{5/6}} \le \frac{1}{3 n^{1/12}} \to 0,\\
& \frac{n^{1 + \xi} \sum_i \lambda_i^2}{\tx^2} \le \frac{4 n^{1 + \xi} / 3 }{(3 n^{5/6})^2} \le \frac{4}{27 n^{2/3 - \xi}} \to 0, \quad \forall 0 < \xi < \frac{2}{3}.
\end{align*}

\section{Proof for Theorem~\ref{thm_idood}}\label{pf:idood}
Recalling the decomposition $\varSigma = \sum_i \lambda_i e_i e_i^T$, we have
\begin{equation}\label{eq:note1}
    XX^T = \sum_i \lambda_i z_i z_i^T, \quad X \varSigma X^T = \sum_i \lambda_i^2 z_i z_i^T,
\end{equation}
in which
\begin{equation}
  z_i := \frac1{\sqrt{\lambda_i}} X e_i
  \label{eq:z_i}
\end{equation}
are independent $\sigma_x$-subgaussian random vectors in $\rR^n$ with mean $0$ and covariance $I$. Then we will take the following notations in further analysis: 
\begin{equation}\label{eq:note2}
    A = XX^T, \quad A_k = \sum_{i > k} \lambda_i z_i z_i^T, \quad A_{-k} = \sum_{i \ne k} \lambda_i z_i z_i^T.
\end{equation}
\subsection{Kernel matrix linearization}
Here the first step is to estimate the kernel matrix $K = \Phi_r \Phi_r^T \in \mathbb{R}^{n \times n}$ for $r = 1, 2$ properly. With Lemma~\ref{lem:ntk1} and \ref{lem:ntk2} in \citep{jacot2018neural}, we could approximate each element $K_{s,t}$ as
\begin{align}\label{eq:kernel1}
 K_{s,t} &= \left( 1 + O_p(\frac{1}{\sqrt{m}}) \right) \frac{1}{p} \left[  \frac{x_s^T x_t }{2 \pi} \arccos\left( - \frac{x_s^T x_t}{\| x_s \|_2 \| x_t \|_2} \right) + \frac{ \| x_s \|_2 \| x_t \|_2 }{2 \pi} \sqrt{1 -\left( - \frac{x_s^T x_t}{\| x_s \|_2 \| x_t \|_2} \right)^2 } \right],  
\end{align}
here we define a temporary function $h_{s,t}(z)$ as :
\begin{equation*}
    h_{s,t}(z) := \frac{x_s^T x_t}{2 \pi \tx} \arccos \left( - \frac{x_s^T x_t}{\tx z} \right) + \frac{z}{2 \pi} \sqrt{1 - \left( \frac{x_s^T x_t}{\tx z} \right)^2},
\end{equation*}
which has an uniformal bounded Lipschitz as:
\begin{equation*}
    | h'_{s,t}(z) | = | \frac{1}{2 \pi} \sqrt{1 - \left( \frac{x_s^T x_t}{\tx z} \right)^2} | \le \frac{1}{2 \pi},
\end{equation*}
and the kernel matrix $K$ could be approximated by a new kernel $K'$ which has components $K'_{s,t} = (\tx / p) h_{s,t}(1)$, due to the following fact 
\begin{align*}
  & \quad  \| p / \tx K - p / \tx K' \|_2 = p/\tx \max_{\beta \in \mathbb{S}^{n-1}} \left| \beta^T (K - K') \beta \right| \\
  &= \max_{\beta \in \mathbb{S}^{n-1}} \left| \sum_{s,t} \beta_s \beta_t \left( \left( 1 + O_p(\frac{1}{\sqrt{m}}) \right) h_{s,t}(\| x_s \|_2 \| x_t \|_2 / \tx) - h_{s,t}(1) \right) \right|  \\
    & \le \frac{1}{2 \pi} \max_{\beta \in \mathbb{S}^{n-1}} \left| \sum_{s,t} \beta_s \beta_t \mid \frac{\| x_s \|_2 \| x_t \|_2 }{\tx}  - 1 \mid  \right| + O_p(\frac{1}{\sqrt{m}}) \max_{\beta \in \mathbb{S}^{n-1}} \left| \sum_{s,t} \beta_s \beta_t h_{s,t}( \frac{\| x_s \|_2 \| x_t \|_2}{\tx} ) \right| \\
    &\le \frac{1}{2 \pi} \max_{s,t} \mid \frac{\| x_s \|_2 \| x_t \|_2 }{\tx}  - 1 \mid \cdot \max_{\beta \in \mathbb{S}^{n-1}} \sum_{s,t} \beta_s \beta_t + O_p(\frac{p}{ \tx \sqrt{m}}) \| K \|_2 \\
    & = \frac{1}{2 \pi} \max_{s} \mid \frac{\| x_s \|_2^2 }{\tx} - 1 \mid \cdot \max_{\beta \in \mathbb{S}^{n-1}} \sum_{s,t} \beta_s \beta_t + O_p(\frac{p}{ \tx \sqrt{m}}) \| K \|_2 \\
    &\le \frac{n }{2 \pi} \max_s \mid \frac{\| x_s \|_2^2}{\tx}  - 1 \mid + O_p(\frac{p}{\tx \sqrt{m}}) \| K \|_2,
\end{align*}
where the first inequality is due to the bounded Lipschitz norm of $h_{s,t}(z)$, the second inequality is from the fact that $\beta \in \mathbb{S}^{n-1}$, and the last inequality is from Cauthy-Schwarz inequality:
\begin{equation*}
    \sum_{i,j} \beta_i \beta_j \le \sqrt{\sum_{i,j} \beta_i^2 } \sqrt{\sum_{i,j} \beta_j^2} = n \sum_i \beta_i^2 = n. 
\end{equation*}
Then with Assumption~\ref{cond:benign} and \ref{cond:high-dim}, consider the settings on input data, with probability at least $1 - 2n e^{- t^2 \tx^2 / 2 r_0(\varSigma^2)}$, we could obtain that
\begin{equation*}
    \max_{s = 1, \dots, n} \mid \frac{\| x_s \|_2^2}{\tx}  - 1 \mid \le t,
\end{equation*}
under Assumption~\ref{cond:benign}, as $r_0(\varSigma^2) \le r_0(\varSigma) = \tx$, choosing $t = n^{- 5/16}$, we have $t^2 \tx^2/ r_0(\varSigma^2) \ge \tx n^{-5/8} \ge n^{1/8}$, so with probability at least $1 - 2n e^{- n^{1/8} / 2}$, we can get
\begin{equation}\label{eq:kernelest1}
    \| K - K' \|_2 \le \frac{2n^{11/16} }{p \pi} + o_p(\frac{n \tx}{ p \sqrt{m}}) = o(\frac{\tx}{p}),
\end{equation}
where the last inequality is from Assumption~\ref{cond:high-dim}.
Further, if we denote a function $g(\cdot): \mathbb{R} \to \mathbb{R}$ as:
\begin{equation*}
    g(z) := \frac{z}{2 \pi \tx} \arccos (- \frac{z}{\tx}) + \frac{1}{2 \pi} \sqrt{1 - (\frac{z}{\tx})^2},
\end{equation*}
the components of matrix $K'$ could be expressed as $K'_{s,t} = 
\frac{\tx}{p} g(x_s^T x_t)$. Then with a refinement of \citet{el2010spectrum} in Lemma \ref{lem:spectrum}, with a probability at least $1 - 4 n^2 e^{- n^{1/8} / 2}$, we have the following approximation:
\begin{equation}\label{eq:kernelest2}
    \| K' - \tilde{K} \|_2 \le o(\frac{\tx}{p n^{1/16} }),
\end{equation}
in which
\begin{equation}\label{eq:est_kernel}
    \tilde{K} = \frac{\tx}{p} (\frac{1}{2 \pi} + \frac{3 r_0(\varSigma^2)}{4 \pi \tx^2}) 11^T + \frac{1}{4 p} XX^T + \frac{\tx}{p} (\frac{1}{4} - \frac{1}{2 \pi}) I_n,
\end{equation}
and according to Lemma \ref{lem_ridgeeigen}, with probability $1 - 2 e^{-n/c}$, we have
\begin{equation}\label{eq:mu1mun}
 \begin{aligned}
\mu_1( \tilde{K} ) &\le \frac{\tx}{p} (\frac{1}{2 \pi} + \frac{3 r_0(\varSigma^2)}{4 \pi \tx^2})\| 1 1^T \|_2 + \frac{1}{4p} \| XX^T \|_2 + \frac{\tx}{p} (\frac{1}{4} - \frac{1}{2 \pi}) \| I_n \|_2\\
&\le \frac{n \tx}{\pi p} + \frac{c_1}{4 p} (\tx + n \lambda_1) + \frac{\tx}{p} (\frac{1}{4} - \frac{1}{2 \pi}) \le \frac{(1 + c_1 \lambda_1)n\tx}{p},\\ 
\mu_n (\tilde{K}) &\ge \frac{1}{4p} \mu_n(XX^T) + \frac{\tx}{p} (\frac{1}{4} - \frac{1}{2 \pi}) \mu_n(I_n) \ge \frac{\tx}{p} (\frac{1}{4} - \frac{1}{2 \pi}),
 \end{aligned}   
\end{equation}
then combining with Eq.~\eqref{eq:kernelest1} and \eqref{eq:kernelest2}, we could obtain
\begin{equation*}
    \| K - \tilde{K} \|_2 \le O(\frac{n^{11/16}}{p}) + O(\frac{n \tx}{ p m^{1/2}}) + O(\frac{\tx}{p n^{1/16}}) = o(\frac{\tx}{p}),
\end{equation*}
where the last equality is from Assumption~\ref{cond:benign} and \ref{cond:high-dim}. And we could approximate $K$ by $\tilde{K}$ in further analysis.

\subsection{Proof for ID risk upper bound}
For the single model $f_{W_1}(\hat{\theta}(W_1), x)$, we could express the ID excess risk as
\begin{equation}\label{eq:idloss}
\begin{aligned}
 \mathcal{L}_{\mathrm{id}}(f_{W_1}(\hat{\theta}(W_1), x)) &=   \mathbb{E}_{x,y} [\frac{1}{\sqrt{m}} \phi(x^T W_1) (\hat{\theta}(W_1) - \theta^*(W_1))]^2 \\
 &=  \underbrace{ \frac{1}{m} \theta(W_1)^{*T} [I - \Phi_{W_1}^T (\Phi_{W_1} \Phi_{W_1}^T)^{-1} \Phi_{W_1}] \mathbb{E}_x \phi(W_1^T x) \phi(W_1^T x)^T [I - \Phi_{W_1}^T (\Phi_{W_1} \Phi_{W_1}^T)^{-1} \Phi_{W_1}] \theta^*(W_1) }_{\ib} \\
 & \quad + \underbrace{ \frac{1}{m} (\sigma^2 + o(1)) \text{trace} \{ (\Phi_{W_1} \Phi_{W_1}^T)^{-2} \Phi_{W_1} \mathbb{E}_x \phi(W_1^T x) \phi(W_1^T x)^T \Phi_{W_1}^T \} }_{\iv}.
\end{aligned}
\end{equation}

\subsubsection{Proof for bias term}

For the bias part $\ib$, we could consider that
\begin{equation}\label{eq:idbiaspre}
 \begin{aligned}
   \ib &=  \theta^*(W_1) [I - \Phi_{W_1}^T (\Phi_{W_1} \Phi_{W_1}^T)^{-1} \Phi_{W_1}] \left( \frac{1}{m} \mathbb{E}_x \phi(W_1^T x) \phi(W_1^T x)^T - \frac{1}{n} \Phi_{W_1}^T \Phi_{W_1} \right) [I - \Phi_{W_1}^T (\Phi_{W_1} \Phi_{W_1}^T)^{-1} \Phi_{W_1}] \theta^*(W_1) \\
 &\le \| \theta^*(W_1) \|_2^2 \| \frac{1}{m} \mathbb{E}_x \phi(W_1^T x) \phi(W_1^T x)^T - \frac{1}{n} \Phi_{W_1}^T \Phi_{W_1} \|_2,   
 \end{aligned}   
\end{equation}
 where the inequality is induced from $\| I - \Phi_{W_1}^T (\Phi_{W_1} \Phi_{W_1}^T)^{-1} \Phi_{W_1} \|_2 \le 1$ and $a^T B a \le \| a \|_2^2 \| B \|_2$ for any positive-defined matrix $B$. 
 And the next step is to prove the random vector $\Phi_{W_1}(W_1^T x)$ is sub-gaussian with respect to $x$. To be specific, based on Lemma~\ref{lem:feature_ortho}, with probability at least $1 - 2 e^{- n^{\xi / 2} / 4}$, we have
 \begin{equation}\label{eq:biaspre00}
    \mathbb{E}_x [ w_{1,i}^T x ]^2 \le 2 \tx / p, \quad i = 1, \dots, m,
 \end{equation}
 which implies that $\{ w_{1.i}^T x \}$ are $\sigma_x \sqrt{2 \tx / p}$-sub gaussian random variables.
Take derivative for $\phi(W_1^T x)$ with respect to $w_{1,i}^T x$ on each dimension, we have
\begin{equation*}
  \left|  \frac{\partial \phi(w_{1,i}^T x)}{\partial (w_{1,i}^T x)} \right| \le 1,
\end{equation*}
so for any vector $\gamma \in \mathbb{R}^m$, the function $\gamma^T \frac{1}{\sqrt{m}} \phi(W_1^T x)$ has a bounded Lipschitz such that
 \begin{equation*}
    \frac{1}{\sqrt{m}} | \gamma^T  \phi(W_1^T x) | \le \frac{1}{\sqrt{m}} \sum_{i=1}^m | \gamma_i | | w_{1,i}^T x |,
 \end{equation*}
which implies that
\begin{equation*}
 \frac{1}{m} \mathbb{E}_x [ \gamma^T \phi(W_1^T x) ]^2 \le \frac{1}{m} \mathbb{E}_x [ \sum_i | \gamma_i | | w_{1,i}^T x | ]^2 \le \| \gamma \|_2^2 \frac{1}{m} \mathbb{E}_x \mathrm{tr} \{ W_1^T xx^T W_1 \} \le \frac{2 \tx}{p} \| \gamma \|_2^2,
\end{equation*}
where the second inequality is due to Cauthy-Schwarz inequality, and the last inequality is from \eqref{eq:biaspre00}.
So we could obtain that $\frac{1}{\sqrt{m}} \phi(W_1^T x)$ is a sub-gaussian random vector satisfying
\begin{equation}\label{eq:subgaussian}
    \mathbb{E} e^{\lambda \gamma^T \phi(W_1^Tx) / \sqrt{m}} \le e^{ \lambda^2 \sigma_x^2 \| \gamma \|_2^2 \tx/ p  }.
\end{equation}
The next step is to consider the positive-defined matrix $M_1 = 
\frac{1}{m} \mathbb{E}_x \phi(W_1^T x) \phi(W_1^T x)^T$, it has elements as:
\begin{align*}
M_{1,i,j} &=  \frac{w_{1,i}^T \varSigma w_{1,j}}{2 \pi m} \arccos \left( - \frac{w_{1,i}^T \varSigma w_{1,j}}{\| \varSigma^{1/2} w_{1,i} \|_2 \| \varSigma^{1/2} w_{1,j} \|_2} \right) \\
& \quad + \frac{\| \varSigma^{1/2} w_{1,i} \|_2 \| 
\varSigma^{1/2} w_{1,j} \|_2}{2 \pi m} \sqrt{1 - \left( \frac{w_{1,i}^T \varSigma w_{1,j}}{\| \varSigma^{1/2} w_{1,i} \|_2 \| \varSigma^{1/2} w_{1,j} \|_2} \right)^2}, \\
M_{1,i,i} &= \frac{1}{2 m} w_{1,i}^T \varSigma w_{1,i},
\end{align*}
combing with Lemma~\ref{lem:feature_ortho}, we could obtain that with probability at least $1 - 2 e^{- n^{\sqrt{\xi}} / 4}$, 
\begin{equation}\label{eq:biastrace}
 \begin{aligned}
 & \| M_1 \|_2 \le \text{trace}(M_1) = \frac{1}{2 m} \sum_{i=1}^m w_{1,i}^T \varSigma w_{1,i} \le \frac{1}{2} \frac{\tx}{p} \left( 1 + \frac{1}{n^{(2 + \xi)/4}} \right), \\
& \| \mathbb{E}_x \frac{1}{\sqrt{m}} \phi(W_1^T x) \|_2 = \sqrt{\sum_{i=1}^m \frac{1}{2 \pi m}  w_{1,i}^T \varSigma w_{1,i} } \le  \sqrt{ \frac{\tx}{2 \pi p}}  \left( 1 + \frac{1}{n^{(2 + \xi)/8}} \right),    
 \end{aligned}   
\end{equation}
Consider Lemma \ref{lem:matrix_mean} and Eq.~\eqref{eq:subgaussian} \eqref{eq:biastrace}, we have
\begin{equation*}
    \| \mathbb{E}_x \frac{1}{m} \phi(W_1^T x) \phi(W_1^T x)^T - \frac{1}{n} \Phi_{W_1}^T \Phi_{W_1} \|_2 \le \frac{1}{n^{1/4} } \frac{\tx}{p}  \left( 1 + \frac{4 \sigma_x }{\sqrt{\pi}}  \right),
\end{equation*}
with probability at least $1 - 6 e^{- n^{\xi'}}$ where $\xi' = \min\{ 1/2, \xi/2 \}$. Further combing with Eq.~\eqref{eq:idbiaspre}, we have
\begin{equation}\label{eq:idbias}
    \ib \le \frac{\| \theta^*(W_1) \|_2^2 }{ n^{1/4}} \frac{\tx}{p} \left( 1 + \frac{4 \sigma_x }{\sqrt{\pi}}  \right).
\end{equation}

\subsubsection{Proof for variance term}

Now we turn to the variance term $\iv$, 
\begin{align*}
 \iv &= (\sigma^2 + o(1)) \text{trace} \{ (\Phi_{W_1} \Phi_{W_1}^T )^{-2} \Phi_{W_1} \mathbb{E}_x \frac{1}{m} \phi(W_1^T x) \phi(W_1^T x)^T \Phi_{W_1} \}\\
 &= \frac{\sigma^2 + o(1)}{n m} \mathbb{E}_{x_1', \dotsm x_n'} \sum_{i=1}^n \text{trace}\{ (\Phi_{W_1} \Phi_{W_1}^T)^{-2} \Phi_{W_1} \mathbb{E}_x \phi(W_1^T x'_i) \phi(W_1^T x'_i)^T \Phi_{W_1}^T \}\\
 &= \frac{\sigma^2 + o(1)}{n} \mathbb{E}_{x'_1, \dots, x'_n} \text{trace} \{ (\Phi_{W_1} \Phi_{W_1}^T)^{-2} \Phi_{W_1} \Phi_{W_1}^{'T} \Phi'_1 \Phi_{W_1}^T \},
\end{align*}
where we denote $x'_1, \dots, x'_n$ are i.i.d. samples from the same distribution as $x_1, \dots, x_n$, and $\Phi_{W_1}' = [\phi(W_1^T x'_1), \dots, \phi(W_1^T x'_n)]^T$. For the matrix $\Phi_{W_1} \Phi_{W_1}^{T}$ and $\Phi_{W_1}' \Phi_{W_1}^T$, with probability at least $1 - 4 n^2 e^{- n^{1/4}/ 2}$, we could take the similar linearizing procedure as on $K = \Phi_{W_1} \Phi_{W_1}^T$ to obtain:
\begin{align*}
 & \| \Phi_{W_1} \Phi_{W_1}^{'T} -  \frac{\tx}{p} \left( \frac{1}{2 \pi}  + \frac{3 r_0(\varSigma^2)}{4 \pi \tx^2} \right) 11^T + \frac{1}{4p} XX^{'T} \|_2  \le \frac{4 \tx}{p n^{1/16}},\\
 & \| \Phi_{W_1}' \Phi_{W_1}^{T} -  \frac{\tx}{p} \left( \frac{1}{2 \pi}  + \frac{3 r_0(\varSigma^2)}{4 \pi \tx^2} \right) 11^T + \frac{1}{4p} X'X^{T} \|_2  \le \frac{4 \tx}{p n^{1/16}},
 \end{align*} 
where $X' = [x'_1, \dots, x'_n]^T \in \mathbb{R}^{n \times p}$, and it implies that
\begin{equation*}
Q := \frac{1}{n} \mathbb{E}_{x'_1, \dots, x'_n} \Phi_{W_1} \Phi_{W_1}^{'T} \Phi_{W_1}' \Phi_{W_1}^T \prec \frac{\tx^2}{2 \pi^2 p^2} (1 + o(1)) 11^T + \frac{1}{8 p^2} X \varSigma X^T + \frac{32 \tx^2}{p^2 n^{9/8}} I_n.
\end{equation*}
Then according to Lemma \ref{lem:matrix_comp}, as $\mu_n(\tilde{K}) \ge \frac{\tx}{p}(\frac{1}{4} - \frac{1}{2 \pi})$ shown in Eq.~\eqref{eq:mu1mun}, we have
\begin{align}\label{eq:idvarpre1}
\mid \iv - (\sigma^2 + o(1)) \text{trace} \{ \tilde{K}^{-2} Q \} \mid \le   O(\frac{1}{n^{1/16}}) (\sigma^2 + o(1)) \text{trace} \{ \tilde{K}^{-2} Q \}. 
\end{align}
And as we could express $\text{trace} \{ \tilde{K}^{-2} Q \}$ as :
\begin{equation}\label{eq:idvarpre2}
 \begin{aligned}
 & \quad \text{trace} \{ \tilde{K}^{-2} Q \} \\
 &\le \text{trace} \left\{ \left( \frac{\tx}{p} (\frac{1}{2 \pi} + \frac{3 r_0(\varSigma^2)}{4 \pi \tx^2}) 11^T + \frac{1}{4 p} XX^T + \frac{\tx}{p} (\frac{1}{4} - \frac{1}{2 \pi}) I_n \right)^{-2} \left( \frac{\tx^2}{2 \pi^2 p^2} (1 + o(1)) 11^T + \frac{1}{8 p^2} X \varSigma X^T + \frac{32 \tx^2}{p^2 n^{9/8}} I_n \right) \right\}\\
&\le \underbrace{ \frac{\tx^2}{2 \pi^2 p^2} (1 + o(1)) 1^T \left( \frac{\tx}{p} (\frac{1}{2 \pi} + \frac{3 r_0(\varSigma^2)}{4 \pi \tx^2}) 11^T + \frac{1}{4 p} XX^T + \frac{\tx}{p} (\frac{1}{4} - \frac{1}{2 \pi}) I_n \right)^{-2} 1}_{\iv_1} \\
& \quad + \underbrace{ \text{trace} \left\{ \left( \frac{1}{4p} XX^T + \frac{\tx}{p} (\frac{1}{4} - \frac{1}{2 \pi}) I_n \right)^{-2} \left( \frac{1}{8 p^2} X \varSigma X^T  \right) \right\} }_{\iv_2}\\
 & \quad + \underbrace{ \text{trace} \left\{ \left( \frac{1}{4p} XX^T + \frac{\tx}{p} (\frac{1}{4} - \frac{1}{2 \pi}) I_n \right)^{-2} \left(  \frac{32 \tx^2}{p^{2} n^{9/8}} I_n \right) \right\} }_{\iv_3},
\end{aligned}   
\end{equation}
where the last inequality is from the fact that
\begin{equation*}
    \frac{\tx}{p} (\frac{1}{2 \pi} + \frac{3 r_0(\varSigma^2)}{4 \pi \tx^2}) 11^T + \frac{1}{4 p} XX^T + \frac{\tx}{p} (\frac{1}{4} - \frac{1}{2 \pi}) I_n \succ  \frac{1}{4 p} XX^T + \frac{\tx}{p} (\frac{1}{4} - \frac{1}{2 \pi}) I_n.
\end{equation*}
By Woodbury identity, denoting
\begin{equation*}
    \tilde{R} := \frac{1}{4 p} XX^T + \frac{\tx}{p}(\frac{1}{4} - \frac{1}{2 \pi}) I_n,
\end{equation*}
we can get
\begin{equation*}
\begin{aligned}
\iv_1 &= \frac{\tx^2}{ 2 \pi^2 p^2} (1 + o(1)) 1^T \tilde{K}^{-2} 1 = \frac{\tx^2}{2 \pi^2 p^2} (1 + o(1)) 1^T (\frac{\tx}{p} (\frac{1}{2 \pi} + \frac{3 r_0(\varSigma^2)}{4 \pi \tx^2}) 11^T + \tilde{R})^{-2} 1\\
& = \frac{\frac{\tx^2}{2 \pi^2 p^2} (1 + o(1)) 1^T \tilde{R}^{-2} 1}{(1 + \frac{\tx}{p} (\frac{1}{2 \pi} + \frac{3 r_0(\varSigma^2)}{4 \pi \tx^2}) 1^T \tilde{R}^{-1} 1)^2}  \le \frac{2 1^T \tilde{R}^{-2} 1}{(1^T \tilde{R}^{-1} 1)^2} \le \frac{2 n / \lambda_n(\tilde{R})^2}{n^2 / \lambda_1(\tilde{R})^2},   
\end{aligned}    
\end{equation*}
where the first inequality is from ignoring the constant term on denominator, and the second inequality is due to the fact
\begin{align*}
    & 1^T \tilde{R}^{-2} 1 \le n \lambda_1(\tilde{R}^{-1})^2 = n / \lambda_n(\tilde{R})^2,\\
    & 1^T \tilde{R}^{-1} 1 \ge n \lambda_n(\tilde{R}^{-1}) = n / \lambda_1(\tilde{R}),
\end{align*}
as recalling Lemma \ref{lem_ridgeeigen}, with a high probability, we have
\begin{align*}
& \lambda_n(\tilde{R}) \ge \frac{\tx}{p}(\frac{1}{4} - \frac{1}{2 \pi}) +  \frac{1}{c_1 p} \lambda_{k^* + 1} r_{k^*} \ge \frac{\tx}{8 p}, \\
& \lambda_1 (\tilde{R}) \le \frac{\tx}{p}(\frac{1}{4} - \frac{1}{2 \pi}) + \frac{c_1}{p}(n \lambda_1 + \tx) \le \frac{2 \tx (1 + c_1)}{p} \le \frac{(1 + c_1)(\tx + n)}{p},   
\end{align*}
we can further obtain that
\begin{equation*}
\frac{\tx^2}{2 \pi^2 p^2} (1 + o(1)) 1^T \tilde{K}^{-2} 1 \le \frac{n}{n^2} \frac{2 (\tx + n)^2 (1 + c_1)^2 / p^2}{\tx^2 / (8 p^2)} \le 256(1 + c_1)^2 \left( \frac{1}{n} + \frac{n}{\tx^2} \right),
\end{equation*}
further due to Assumption \ref{cond:high-dim}, we have
\begin{equation}\label{eq:idvarpre3}
\iv_1 \le \frac{\tx^2}{2 \pi^2 p^2} (1 + o(1)) 1^T \tilde{K}^{-2} 1 \le  256 (1 + c_1)^2 \left( \frac{1}{n} + \frac{n}{\tx^2} \right) \le \frac{512 (1 + c_1)^2}{n^{1/2}}.
\end{equation}
For the second term, based on Lemma~\ref{lem_ridgeeigen}, with probability at least $1 - c e^{- n /c}$, we can obtain that
\begin{equation}\label{eq:idvarpre4}
\begin{aligned}
\iv_2 &= \frac{1}{8 p^2} \text{trace} \{ (\frac{1}{4p} XX^T + \frac{\tx}{p} (\frac{1}{4} - \frac{1}{ 2 \pi}) I_n )^{-2} X \varSigma X^T \}\\
&= 2 \text{trace} \{ (XX^T +  \tx (1 - 2 / \pi) I_n )^{-2} X \varSigma X^T \}\\
& \le 2 \left(  \frac{k^*}{n} \frac{(c_1 \lambda_{k^* + 1} r_{k^*} + \tx (1 - 2/\pi))^2}{( 1/c_1 \lambda_{k^* + 1} r_{k^*} + \tx (1 - 2/\pi))^2} + \frac{n \sum_{i > k^*} \lambda_i^2}{(\lambda_{k^* + 1} r_{k^*} + \tx(1 - 2/\pi) )^2} \right)\\
&\le  2 c_1^4 \frac{k^*}{n} + \frac{2 n \sum_{i > k^*} \lambda_i^2}{(\tx(1 - 2/\pi) )^2},
\end{aligned}    
\end{equation}
where the first inequality is based on Lemma \ref{lem_bart}, and the second inequality is from the fact that
\begin{align*}
 & \frac{(c_1 \lambda_{k^* + 1} r_{k^*} + \tx (1 - 2/\pi))^2}{( 1/c_1 \lambda_{k^* + 1} r_{k^*} + \tx (1 - 2/\pi))^2} \le c_1^4,\\
 & \lambda_{k^* + 1} r_{k^*} + \tx(1 - 2/\pi) \ge \tx (1 - 2/\pi).
 \end{align*}
And for the third term,
\begin{equation}\label{eq:idvarpre5}
 \begin{aligned}
\iv_3 &= \frac{32 \tx^2}{p^2 n^{9/8}}  \text{trace} \left\{ \left( \frac{1}{4 p} XX^T + \frac{\tx}{p}(\frac{1}{4} - \frac{1}{2 \pi}) I_n \right)^{-2} \right\}\\
&= \frac{512 \tx^2}{n^{9/8}} \text{trace} \{ (XX^T + \tx (1 - 2/\pi) I_n)^{-2} \}\\
& \le \frac{512}{(1 - 2/\pi)^2} \frac{1}{n^{1/8}},
\end{aligned}   
\end{equation}
where the inequality is from the fact that $\mu_{n}(XX^T + \tx (1 - 2/\pi)) \ge \tx(1 - 2/\pi)$.
Combing Eq~\eqref{eq:idvarpre1}, \eqref{eq:idvarpre2}, \eqref{eq:idvarpre3}, \eqref{eq:idvarpre4} and \eqref{eq:idvarpre5}, with a high probability, $\srisk(\hat{w})$ can be upper bounded as
\begin{equation}\label{eq:idvar}
    \iv \le  \sigma^2 \left( \frac{1024}{(1 - 2 / \pi)^2} \frac{1}{n^{1/8}} + 2 c_1^4 \frac{k^*}{n} + \frac{2 n \sum_{i> k^*} \lambda_i^2}{\tx^2 (1 - 2/ \pi)^2} \right).
\end{equation}
And combing Eq. \eqref{eq:idbias} and \eqref{eq:idvar}, we have
\begin{equation*}
    \mathcal{L}_{\mathrm{id}}(f_{W_1}(\hat{\theta}(W_1), x)) \le \frac{\| \theta^*(W_1) \|_2^2 }{ n^{1/4}} \sqrt{\frac{\tx}{p}} \left( 1 + \frac{4 \sigma_x }{\sqrt{\pi}}  \right) + \sigma^2 \left( \frac{1024}{(1 - 2 / \pi)^2} \frac{1}{n^{1/8}} + 2 c_1^4 \frac{k^*}{n} + \frac{2 n \sum_{i> k^*} \lambda_i^2}{\tx^2 (1 - 2/ \pi)^2} \right).
\end{equation*}

\subsection{Proof for OOD risk lower bound}

Then we turn to the OOD situation. Due to the definition of OOD excess risk, we could obtain that
\begin{align*}
& \quad \mathcal{L}_{\mathrm{ood}}(f_{W_1}(\hat{\theta}(W_1), x)) \\
&= \max_{\Delta \in \Xi_{\mathrm{ood}}} \mathbb{E}_{x,\delta,y} \left[ \frac{1}{\sqrt{m}} \phi((x + \delta)^T W_1) (\hat{\theta}(W_1) - \theta^*(W_1))  \right]^2\\
&= \mathcal{L}_{\mathrm{id}}(f_{W_1}(\hat{\theta}(W_1), x)) + \frac{1}{m} \max_{\Delta \in \Xi_{\mathrm{ood}}} \mathbb{E}_{x,\delta, \epsilon} [\delta^T\nabla_x \phi(W_1^T x) (\hat{\theta}(W_1) - \theta^*(W_1))]^2\\
&= \mathcal{L}_{\mathrm{id}}(f_{W_1}(\hat{\theta}(W_1), x)) + \max_{\Delta \in \Xi_{\mathrm{ood}}} \underbrace{\frac{1}{m} \theta_1^{*T}[I - \Phi_{W_1}^T (\Phi_{W_1} \Phi_{W_1}^T)^{-1} \Phi_{W_1}] \mathbb{E}_x \nabla_x \phi(W_1^T x)^T \varSigma_{\delta} \nabla_x \phi(W_1^T x) [I - \Phi_{W_1}^T (\Phi_{W_1} \Phi_{W_1}^T)^{-1} \Phi_{W_1}] \theta^*(W_1)}_{\ob} \\
& \quad +\max_{\Delta \in \Xi_{\mathrm{ood}}}\underbrace{ (\sigma^2 + o(1)) \frac{1}{m} \text{trace} \{ (\Phi_{W_1} \Phi_{W_1}^T)^{-2} \Phi_{W_1} \mathbb{E}_x \nabla_x \phi(W_1^T x)^T \varSigma_{\delta} \nabla_x \phi(W_1^T x) \Phi_{W_1}^T\} }_{\ov},
 \end{align*}
where we take first-order Taylor expansion with respect to $x$. For the lower bound for $\mathcal{L}_{\mathrm{ood}}(f_{W_1}(\hat{\theta}(W_1), x))$, we just focus on the variance term and ignore the impact on bias.
Then for the variance term $\ov$, similar to the analysis above, we need to deal with the matrix
\begin{equation*}
    D := \frac{1}{m} \Phi_{W_1} \mathbb{E}_x \nabla_x \phi(W_1^T x)^T \varSigma_{\delta} \nabla_x \phi(W_1^T x) \Phi_{W_1}^T \in \mathbb{R}^{n \times n}
\end{equation*}
Taking expectation with respect to $x$, we have
\begin{align*}
 \left( \mathbb{E}_x \nabla_x \phi(W_1^T x)^T \varSigma_{\delta} \nabla_x \phi(W_1^T x) \right)_{i,j} &= \mathbb{E}_x \left(\frac{\partial \phi(w_{1,i}^T x)}{\partial x} \right)^T \varSigma_{\delta} \frac{\partial \phi(w_{1,j}^T x)}{\partial x}\\
 &= \frac{1}{2 \pi} \arccos \left( - \frac{w_{1,i}^T \varSigma w_{1,j}}{\| \varSigma^{1/2} w_{1,i} \|_2 \| \varSigma^{1/2} w_{1,j} \|_2} \right) w_{1,i}^T \varSigma_{\delta} w_{1,j},\\
 \left( \mathbb{E}_x \nabla_x \phi(W_1^T x)^T \varSigma_{\delta} \nabla_x \phi(W_1^T x) \right)_{i,i} &= \mathbb{E}_x \left(\frac{\partial \phi(w_{1,i}^T x)}{\partial x} \right)^T \varSigma_{\delta} \frac{\partial \phi(w_{1,i}^T x)}{\partial x}\\
 &= \frac{1}{2} w_{1,i}^T \varSigma_{\delta} w_{1,i},
\end{align*}
furthermore, according to Lemma \ref{lem:feature_ortho}, with probability at least $1 - 2 e^{- n^{\sqrt{\xi}} / 4}$, for any $i \ne j$, we have
\begin{equation*}
\left( \mathbb{E}_x \nabla_x \phi(W_1^T x)^T \varSigma_{\delta} \nabla_x \phi(W_1^T x) \right)_{i,j} = \frac{1}{2 \pi} w_{1,i}^T \varSigma_{\delta} w_{1,j} \left( \frac{\pi}{2} + O(\frac{1}{n^{(2 + \xi)/4}}) \right),
\end{equation*}
the equality is from Lemma~\ref{lem:feature_ortho} and the fact that function $\arccos(-z)$ has a constant Lipschitz bound around $0$:
\begin{equation*}
    \left| \frac{w_{1,i}^T \varSigma w_{1,j}}{\| \varSigma^{1/2} w_{1,i} \|_2 \| \varSigma^{1/2} w_{1,j} \|_2} \right| \le \frac{O(n^{ - (2 + \xi)/4} )}{1 - O(n^{- (2 + \xi)/4})} = O(\frac{1}{n^{(2 + \xi)/4}}).
\end{equation*}
While $\mathrm{tr}\{ \varSigma_{\delta} \} / \mu_1(\varSigma_{\delta}) \ge n^2$, the components of $D$ could be expressed as:
\begin{equation}\label{eq:d}
 \begin{aligned}
 D_{s,t} &= \frac{1}{m^2} \sum_{i,j} \left( \mathbb{E}_x \nabla_x \phi(W_1^T x)^T \varSigma_{\delta} \nabla_x \phi(W_1^T x) \right)_{i,j} \phi(w_{1,i}^T x_s) \phi(w_{1,j}^T x_t)\\
&= \frac{1}{2m^2} \sum_i w_{1,i}^T \varSigma_{\delta} w_{1,i} \phi(w_{1,i}^T x_s) \phi(w_{1,i}^T x_t) + \frac{1}{2 \pi m^2} \sum_{i \ne j} w_{1,i}^T \varSigma_{\delta} w_{1,j} \phi(w_{1,i}^T x_s) \phi(w_{1,j}^T x_t) \left( \frac{\pi}{2} + O(\frac{1}{n^{(2 + \xi)/4}}) \right)\\
&= \left( 1 + O(\frac{1}{\sqrt{m}}) \right) \mathbb{E}_{w \sim \mathcal{N}(0, 1/p I_p)} \frac{1}{2m} w^T \varSigma_{\delta} w \phi(w^T x_s) \phi(w^T x_t)\\
& \quad + \left( 1 + O(\frac{1}{\sqrt{m}}) \right) \frac{m-1}{m}  \mathbb{E}_{w, w' \sim \mathcal{N}(0, 1/p I_p)} \frac{1}{2 \pi} w^T \varSigma_{\delta} w' \phi(w^T x_s) \phi(w^{'T} x_t) \left( \frac{\pi}{2} + O(\frac{1}{n^{(2 + \xi)/4}}) \right) \\
&= \left( 1 + O(\frac{1}{\sqrt{m}}) + O(\frac{1}{n^2}) \right) \frac{1}{2 pm} \text{trace}( \varSigma_{\delta} ) K_{s, t} + \left(1 + O(\frac{1}{\sqrt{m}}) \right) \left( 1 + O(\frac{1}{n^{(2 + \xi)/4}}) \right) \frac{m-1}{16 m p^2} x_s^T \varSigma_{\delta} x_t,    
 \end{aligned}   
\end{equation}
where the last equality is induced from Lemma~\ref{lem:ood} and the definition of kernel matrix $K$.  Then from Lemma~\ref{lem:trace_fnorm}, we have
\begin{align*}
  & \quad  \left| \mathrm{tr} \left\{ K^{-2} \left( D - \frac{1}{2 p m} \mathrm{tr}\{ \varSigma_{\delta} \} K - \frac{m-1}{16 m p^2} X \varSigma_{\delta} X \right) \right\} \right| \\
  &\le \left( O(\frac{1}{\sqrt{m}}) + O(\frac{1}{n^2}) \right) \frac{\mathrm{tr}\{ \varSigma_{\delta} \}}{2pm} \| K^{-2} \|_F \| K \|_F + \left( O(\frac{1}{\sqrt{m}}) + O(\frac{1}{n^{(2+\xi)/4}}) \right) \frac{m-1}{16mp^2} \| K^{-2} \|_F \| X \varSigma_{\delta} X^T \|_F \\
  &\le \left( O(\frac{1}{\sqrt{m}}) + O(\frac{1}{n^2}) \right) \frac{n \mathrm{tr}\{ \varSigma_{\delta} \}}{2pm} \mu_1(K^{-2}) \mu_1(K) + \left( O(\frac{1}{\sqrt{m}}) + O(\frac{1}{n^{(2+\xi)/4}}) \right) \frac{n(m-1)}{16mp^2} \mu_1(K^{-2}) \mu_1(X\varSigma_{\delta} X^T)\\
  &\le \left( O(\frac{1}{\sqrt{m}}) + O(\frac{1}{n^2}) \right) \frac{\tx}{m} \frac{\mu_1(K)}{\mu_n(K)^2} + \left( O(\frac{1}{\sqrt{m}}) + O(\frac{1}{n^{(2+\xi)/4}}) \right) \frac{n}{p^2} \frac{\mu_1(XX^T) }{\mu_n(K)^2}, 
  \end{align*}
  where the second inequality is from $\| A \|_F \le \sqrt{n} \| A \|_2$ for any $A \in \mathbb{R}^{n \times n}$, and the last inequality is due to Assumption~\ref{ass:ood1} and \ref{ass:ood2}, in which
  \begin{align*}
  & \mathrm{tr} \{ \varSigma_{\delta} \} = \sum_i  \alpha_i \le \tau k^* + \tau \frac{p\tx}{n} \le \tau \lambda_1 n + 
\tau \frac{p\tx}{n} \le 2 \tau \frac{p\tx}{n}, \quad \| \varSigma_{\delta} \|_2 \le \tau.
  \end{align*}
  Then with Eq.~\eqref{eq:mu1mun} and Lemma~\ref{lem_eigen}, we have
  \begin{equation*}
      \mu_1(K) \le \frac{(1 + c_1 \lambda_1)nl}{p}, \quad \mu_n(K) \ge \frac{\tx}{p}(\frac{1}{4} - \frac{1}{2 \pi}), \quad \mu_1(XX^T) \le \tx + n \lambda_1,
  \end{equation*}
consider Assumption~\ref{cond:high-dim}, we will obtain that
\begin{equation}\label{eq:dgap}
     \zeta:= \left| \mathrm{tr} \left\{ K^{-2} \left( D - \frac{1}{2 p m} \mathrm{tr}\{ \varSigma_{\delta} \} K - \frac{m-1}{16 m p^2} X \varSigma_{\delta} X \right) \right\} \right| \le O(\frac{np}{m^{3/2}}) + O(\frac{p}{mn}) + O(\frac{1}{n^{\xi/4}}) = o(1),
\end{equation}
it implies that we can estimate $\ov$ as
and 
\begin{equation}\label{eq:odvarexpress}
\begin{aligned}
 \ov &= (\sigma^2 + o(1)) \text{trace} \{ (\Phi_{W_1} \Phi_{W_1}^T)^{-2} D \} \\
&\ge \frac{(\sigma^2 + o(1))}{2 p m} \mathrm{tr}( \varSigma_{\delta} ) \mathrm{tr} \{ K^{-1} \}  + (\sigma^2 + o(1)) \frac{m-1}{16 m p^2} \text{trace} \{ K^{-2} X \varSigma_{\delta} X^T \} - \zeta \\
\ov &\le \frac{(\sigma^2 + o(1))}{2 p m} \mathrm{tr}( \varSigma_{\delta} ) \mathrm{tr} \{ K^{-1} \}  + (\sigma^2 + o(1)) \frac{m-1}{16 m p^2} \text{trace} \{ K^{-2} X \varSigma_{\delta} X^T \} + \zeta.
\end{aligned}    
\end{equation}
Then according to Lemma \ref{lem:matrix_comp}, as $\mu_n(\tilde{K}) \ge \frac{\tx}{p}(\frac{1}{4} - \frac{1}{2 \pi})$ shown in Eq.~\eqref{eq:mu1mun}, we have
\begin{equation}\label{eq:odvarpre1}
\begin{aligned}
  & \mid \text{trace}(K^{-1}) - \text{trace} (\tilde{K}^{-1}) \mid \le O(\frac{1}{n^{1/16}}) \text{trace} (\tilde{K}^{-1}),\\
 & \mid \text{trace}(K^{-2} X \varSigma_{\delta} X^T) - \text{trace} (\tilde{K}^{-2} X \varSigma_{\delta} X^T) \mid \le O(\frac{1}{n^{1/16}}) \text{trace} (\tilde{K}^{-2} X \varSigma_{\delta} X^T). 
\end{aligned}
\end{equation}
Then for the first term, 
\begin{equation*}
\begin{aligned}
 \mathrm{tr}\{\tilde{K}^{-1}\} &= \mathrm{tr} \left\{ \left(\frac{\tx}{p} (\frac{1}{2 \pi} + \frac{3 r_0(\varSigma^2)}{4 \pi \tx^2}) 11^T + \frac{1}{4 p} XX^T + \frac{\tx}{p} (\frac{1}{4} - \frac{1}{2 \pi}) I_n\right)^{-1} \right\} \\
 &= (1 + o(1)) \mathrm{tr} \{ \left( \frac{1}{4p} XX^T + \frac{\tx}{p} (\frac{1}{4} - \frac{1}{2 \pi}) I_n \right)^{-1} \} \\
 &= (1 + o(1)) 4p \mathrm{tr} \{ (XX^T + \tx (1 - 2 / \pi) I_n)^{-1} \},
 \end{aligned}    
\end{equation*}
where the second equality is from relaxing the unimportant term $11^T$ in trace calculation (see Lemma 2.2 in \citet{bai2008methodologies}). Then using Woodbury identity, we have
\begin{align*}
 \mathrm{tr}\{ (XX^T + \tx(1 - 2/\pi) I)^{-1} \} &= \mathrm{tr} \{ (A_{-1} + \tx(1 - 2/\pi) I)^{-1} \} - \frac{\lambda_1 z_1^T (A_{-1} + \tx(1 - 2/\pi I))^{-2} z_1}{1 + \lambda_1 z_1^T (A_{-1} + \tx(1 - 2/\pi) I)^{-1} z_1}\\
 &= \mathrm{tr}\{ A_{k^*} + \tx(1 - 2/\pi) I)^{-1} \} - \sum_{i=1}^{k^*} \frac{\lambda_i z_i^T (A_{i} + \tx(1 - 2/\pi) I)^{-2} z_i}{1 + \lambda_i z_i^T (A_{i} + \tx(1 - 2/\pi) I)^{-1} z_i}\\
 &\ge \mathrm{tr}\{ A_{k^*} + \tx(1 - 2/\pi) I)^{-1} \} - \sum_{i=1}^{k^*} \frac{z_i^T (A_{i} + \tx(1 - 2/\pi) I)^{-2} z_i}{ z_i^T (A_{i} + \tx(1 - 2/\pi)I)^{-1} z_i}.
 \end{align*}
Due to Lemma~\ref{lem_subspacenorm}, with probability at least $1 - 3 e^{-n / c}$, for any $1 \le i \le k^*$, we have
\begin{equation}\label{eq:znorm}
 \begin{aligned}
  & \| z_i \|^2 \le n + 2 (162e)^2 \sigma_x^2(\frac{n}{c} + \ln{k^*} + \sqrt{n(\frac{n}{c} + \ln{k^*})}) \le c_2 n,\\
& \| \bm{\Pi}_{\mathscr{L}_i} z \|^2 \ge n - 2 (162e)^2 \sigma_x^2 (k^* + \frac{n}{c} + \ln{k^*} + \sqrt{n(\frac{n}{c} + \ln{k^*})}) \ge n / c_3,  
 \end{aligned}   
\end{equation}
where $\mathscr{L}_i$ is the span of the $n - k^*$ eigenvectors of $A_i$ corresponding to its smallest $n - k^*$ eigenvalues, $\bm{\Pi}_{\mathscr{L}_i}$ is the orthogonal projection on $\mathscr{L}_i$, and $c_2 = 8(162e)^2 \sigma_x^2$, $c_3 = 2$, (in our assumptions, $c > 1$ is a large enough constant to make $\sqrt{c} > 16 (162 e)^2 \sigma_x^2$, which leads to a positive $c_3$). Then according to Assumption~\ref{cond:benign} and Lemma~\ref{lem_eigen}, with probability at least $1 - 5 e^{- n / c}$, we have
\begin{align*}
& \quad z_i^T (A_{i} + \tx(1 - 2/\pi) I)^{-1} z_i \ge (\bm{\Pi}_{\mathscr{L}_i}z_i)^T (A_{i} + \tx(1 - 2/\pi) I)^{-1} (\bm{\Pi}_{\mathscr{L}_i} z_i) \\
& \ge \frac{\parallel \bm{\Pi}_{\mathscr{L}_i} z_i \parallel^2}{\mu_{k^{*} + 1} (A_{i} + \tx(1 - 2/\pi) I)} \ge \frac{n}{c_3  (c_1 \sum_{j > k^{*}} \lambda_j + \tx(1 - 2/\pi))}\\
&\ge \frac{n}{c_3 (c_1 + 1 - 2 / \pi) \tx},
\end{align*}
and we could also obtain that
\begin{equation*}
z_i^T (A_{i} + \tx(1 - 2/\pi) I)^{-2} z_i \le \frac{\| z_i \|_2^2}{\mu_n(A_i + \tx(1 - 2/\pi) I)^2} \le \frac{c_2 n}{\tx^2 (1 - 2/\pi)^2},
\end{equation*}
and they imply that
\begin{equation}\label{eq:tem}
    \sum_{i=1}^{k^*} \frac{z_i^T (A_{i} + \tx(1 - 2/\pi) I)^{-2} z_i}{ z_i^T (A_{i} + \tx(1 - 2/\pi)I)^{-1} z_i} \le \sum_{i=1}^{k^*} \frac{c_2 n}{\tx^2 (1 - 2/\pi)^2} \left( \frac{n}{c_3 (c_1 + 1 - 2 / \pi) \tx} \right)^{-1} = \frac{c_2 c_3 (c_1 + 1 - 2/\pi)}{(1 - 2/\pi)^2} \frac{k^*}{\tx},
\end{equation}
and combing it with the fact that
\begin{equation*}
    \mathrm{tr}\{ (A_{k^*} + \tx(1 - 2/\pi)I)^{-1} \} \ge \frac{n}{\mu_1(A_{k^*} + \tx(1 - 2/\pi) I)} \ge \frac{n}{c_1 \sum_{j > k^*} \lambda_j + \tx(1 - 2/\pi)} \ge \frac{n}{\tx(c_1 + 1 - 2/\pi)},
\end{equation*}
we will obtain 
\begin{equation}\label{eq:ovup1}
 \begin{aligned}
  & \quad \mathrm{tr}\{ (XX^T + \tx(1 - 2/\pi) I)^{-1} \}  \\
  & \ge \frac{n}{\tx(c_1 + 1 - 2/\pi)} - \frac{c_2 c_3 (c_1 + 1 - 2/\pi)}{(1 - 2/\pi)^2} \frac{k^*}{\tx}\\
  & \ge \frac{n}{2 \tx(c_1 + 1 - 2/\pi)},
\end{aligned}   
\end{equation}
where the last inequality is induced by Assumption~\ref{cond:benign}. And from the other side, we have
\begin{equation}\label{eq:ovlow1}
    \mathrm{tr}\{ (XX^T + \tx(1 - 2/\pi) I)^{-1} \} \le \frac{n}{\mu_n(XX^T + \tx(1 - 2/\pi) I)} \le \frac{n}{\tx(1 - 2/\pi)},
\end{equation}
so combing Eq.~\eqref{eq:ovup1} and \eqref{eq:ovlow1}, with probability at least $1 - 7 e^{- n / c}$, we have
\begin{equation}\label{eq:odvarpre2}
\begin{aligned}
& \mathrm{tr}\{ \tilde{K}^{-1} \} \ge (1 + o(1)) \frac{2 n p}{ \tx(c_1 + 1 - 2/\pi)}, \\
& \mathrm{tr}\{ \tilde{K}^{-1} \} \le (1 + o(1)) \frac{4 n p}{\tx(1 - 2/\pi)}.
\end{aligned}
\end{equation}
Similarly, for the second term, we have
\begin{align*}
\mathrm{tr}(\tilde{K}^{-2} X \varSigma_{\delta} X^T) &= \mathrm{tr} \left\{ \left(\frac{\tx}{p} (\frac{1}{2 \pi} + \frac{3 r_0(\varSigma^2)}{4 \pi \tx^2}) 11^T + \frac{1}{4 p} XX^T + \frac{\tx}{p} (\frac{1}{4} - \frac{1}{2 \pi}) I_n \right)^{-2} X \varSigma_{\delta} X^T \right\} \\
 &= (1 + o(1)) \mathrm{tr} \{ \left( \frac{1}{4p} XX^T + \frac{\tx}{p} (\frac{1}{4} - \frac{1}{2 \pi}) I_n \right)^{-2} X \varSigma_{\delta} X^T \}\\
 &= 16p^2 (1 + o(1)) \mathrm{tr} \{ (XX^T + \tx(1 - 2/\pi) I)^{-2} X \varSigma_{\delta} X^T \}, 
\end{align*}
where the second equality is from rom relaxing the unimportant term $11^T$ in trace calculation (see Lemma 2.2 in \citet{bai2008methodologies}). Similar to the analysis above, we could decompose the matrix $X \varSigma_{\delta} X^T$ as
\begin{equation*}
    X \varSigma_{\delta} X^T = \sum_i \lambda_i \alpha_i z_i z_i^T,
\end{equation*}
by Woodbury identity, we could further obtain that
\begin{align*}
\mathrm{tr} \{ (XX^T + \tx(1 - 2/\pi) I)^{-2} X \varSigma_{\delta} X^T \} &= \sum_i \lambda_i \alpha_i z_i^T (XX^T + \tx(1 - 2/\pi) I)^{-2} z_i \\
&= \sum_i \frac{\lambda_i \alpha_i z_i^T(A_{-i} + \tx(1 - 2/\pi) I)^{-2} z_i}{(1 + \lambda_i z_i^T (A_{-i} + \tx(1 - 2/\pi) I)^{-1} z_i)^2},
\end{align*}
for each index $i = 1, \dots, p$, according to Lemma~\ref{lem_subspacenorm} and \ref{lem_ridgeeigen}, with probability at least $1 -5 n^{- n / c}$, we have
\begin{align*}
    \frac{\lambda_i \alpha_i z_i^T(A_{-i} + \tx(1 - 2/\pi) I)^{-2} z_i}{(1 + \lambda_i z_i^T (A_{-i} + \tx(1 - 2/\pi) I)^{-1} z_i)^2} &\ge \frac{\lambda_i \alpha_i (z_i^T (A_{-i} + \tx(1 - 2/\pi) I)^{-1} z_i)^2}{\| z_i \|_2^2(1 + \lambda_i z_i^T (A_{-i} + \tx(1 - 2/\pi) I)^{-1} z_i)^2}\\
    &= \frac{\alpha_i }{\lambda_i \| z_i \|_2^2} \left( 1 + \frac{1}{\lambda_i z_i^T (A_{-i} + \tx(1 - 2/\pi) I)^{-1} z_i} \right)^{-2}\\
    &\ge \frac{\alpha_i}{c_2 \lambda_i n} \left( 1 + \frac{\tx(1 - 2/\pi)}{c_2 n \lambda_i} \right)^{-2},
\end{align*}
where the first inequality is from Cauthy-Schwarz inequality, and the second inequality is from the fact that
\begin{equation*}
    \| z_i \|_2^2 \le c_2 n, \quad z_i^T(A_{-i} + \tx(1 - 2/\pi) I)^{-1} z_i \le \frac{\| z_i \|_2^2}{\mu_n(A_{-i} + \tx(1 - 2/\pi) I)} \le \frac{c_2 n}{\tx(1 - 2/\pi)}.
\end{equation*}
So according to Lemma~\ref{lem_sum}, with probability at least $1 - 10 e^{-n/c}$, we could obtain that
\begin{equation}\label{eq:ovup2}
 \begin{aligned}
& \quad \mathrm{tr} \{ (XX^T + l(1 - 2/\pi) I)^{-2} X \varSigma_{\delta} X^T \}\\
&\ge \frac{1}{2 c_2 n} \sum_i \frac{\alpha_i}{\lambda_i} \left( 1 + \frac{\tx(1 - 2/\pi)}{c_2 n \lambda_i} \right)^{-2} \\
&\ge \frac{1}{8 c_2 n} \sum_i \frac{\alpha_i}{\lambda_i} \min\{ 1, \frac{n^2 \lambda_i^2}{\tx^2} \} \\
&\ge \frac{1}{8 c_2 n} \sum_i \frac{\alpha_i}{\lambda_i} \min\{ 1, \frac{n^2 \lambda_i^2}{b^2 \tx^2} \} \\
&= \frac{1}{8 c_2  n} \left( \sum_{\lambda_i > \tx / (bn)} \frac{\alpha_i}{\lambda_i} + \sum_{\lambda_i \le \tx/ (b n)}  \frac{\alpha_i}{\lambda_i} \frac{n^2 \lambda_i^2}{\tx^2} \right) \\
&= \frac{1}{8 c_2} \sum_{\lambda_i > \tx / (bn)} \frac{\alpha_i}{n \lambda_i} + \frac{1}{8 c_2 b^2}  \frac{n \sum_{\lambda_j \le \tx/ (b n)} \lambda_j \alpha_j}{\tx^2},
 \end{aligned}   
\end{equation}
where the second inequality is from the fact $c_2/(1 - 2/\pi) > 1$ and
\begin{equation*}
    (a + b)^{-2} \ge (2 \max\{ a, b\})^2 = \frac{1}{4} \min\{ a^{-2}, b^{-2} \},
\end{equation*}
and on the last equality, while for any $j > k^*$, we have 
\begin{equation}\label{eq:index}
    n \lambda_i \le \frac{1}{b} \sum_{j > k^*} \lambda_j \le l/b,
\end{equation}
the term $\sum_{\lambda_j \le \tx / bn} \lambda_j \alpha_j$ must contains the index $\{ k^*, \dots, p\}$.
And from another side, we could also obtain that
\begin{align*}
& \quad \mathrm{tr} \{ (XX^T + \tx(1 - 2/\pi) I)^{-2} X \varSigma_{\delta} X^T \} = \sum_i \lambda_i \alpha_i z_i^T (XX^T + \tx(1 - 2/\pi) I)^{-2} z_i \\
&= \sum_{i=1}^{k^*} \frac{\lambda_i \alpha_i z_i^T(A_{-i} + \tx(1 - 2/\pi) I)^{-2} z_i}{(1 + \lambda_i z_i^T (A_{-i} + \tx(1 - 2/\pi) I)^{-1} z_i)^2} + \sum_{i > k^*} \lambda_i \alpha_i z_i^T (XX^T + \tx(1 - 2/\pi) I)^{-2} z_i\\
&\le \sum_{\lambda_i > \tx / (bn)} \frac{\lambda_i \alpha_i z_i^T(A_{-i} + \tx(1 - 2/\pi) I)^{-2} z_i}{(\lambda_i z_i^T (A_{-i} + \tx(1 - 2/\pi) I)^{-1} z_i)^2} + \sum_{\lambda_i \le \tx / (bn)} \lambda_i \alpha_i z_i^T (XX^T + \tx(1 - 2/\pi) I)^{-2} z_i,
\end{align*}
where the inequality is from relaxing the constant $1$ on the denominator. For the first term, with probability at least $1 - 5 e^{- n/c}$, we could obtain that
\begin{equation}\label{eq:teml1}
 \begin{aligned}
  & \quad \sum_{\lambda_i > \tx /(bn)} \frac{\lambda_i \alpha_i z_i^T(A_{-i} + \tx(1 - 2/\pi) I)^{-2} z_i}{(\lambda_i z_i^T (A_{-i} + \tx(1 - 2/\pi) I)^{-1} z_i)^2} \\
 & \le \sum_{\lambda_i > \tx /(bn)} \frac{\alpha_i}{\lambda_i} \frac{\| z_i \|_2^2}{\mu_n(A_{-i} + \tx(1 - 2/\pi) I)^2} \cdot \frac{\mu_{k^*+1}(A_{-i} + \tx(1 - 2/\pi) I)^2}{\| \bm{\Pi}_{\mathscr{L}_i} z_i \|_2^4}\\
 &\le \sum_{\lambda_i > \tx/(bn)} \frac{\alpha_i}{\lambda_i} \frac{ c_2 n}{\tx^2 (1 - 2/\pi)^2} \cdot \frac{c_3^2 (c_1 \sum_{j > k^*} \lambda_i + \tx(1 - 2/\pi))^2}{n^2}\\
 &\le \frac{c_2 c_3^2 (c_1 + 1 - 2/\pi)}{(1 - 2/\pi)^2} \sum_{\lambda_i > \tx/(bn)}  \frac{\alpha_i}{n \lambda_i},
 \end{aligned}   
\end{equation}
where the second inequality is from Eq.~\eqref{eq:znorm} and Lemma~\ref{lem_eigen}, and the third inequality is due to $\sum_{j > k^*} \lambda_j < \tx$. Then for the second term, we could obtain that
\begin{equation*}
    \sum_{\lambda_i \le \tx / (bn)} \lambda_i \alpha_i z_i^T (XX^T + \tx(1 - 2/\pi) I)^{-2} z_i \le \frac{\sum_{\lambda_i \le \tx / (bn)} \lambda_i \alpha_i \| z_i \|_2^2}{\mu_n(XX^T + \tx (1 - 2/\pi) I)^2} \le \frac{\sum_{\lambda_i \le \tx / (bn)} \lambda_i \alpha_i \| z_i \|_2^2}{\tx^2 (1 - 2/\pi)^2},
\end{equation*}
then from Lemma~\ref{lem_stnorm}, with probability at least $1 - 2 e^{- n / c}$, we have
\begin{equation}\label{eq:teml2}
    \sum_{\lambda_i \le \tx / (bn)} \lambda_i \alpha_i z_i^T (XX^T + \tx(1 - 2/\pi) I)^{-2} z_i \le \frac{\sum_{\lambda_i \le \tx / (bn)} \lambda_i \alpha_i \| z_i \|_2^2}{\tx^2 (1 - 2/\pi)^2} \le (1 + 324 e \sigma_x^2 /c)^2 \frac{n \sum_{\lambda_i \le \tx / (bn)} \lambda_i \alpha_i }{\tx^2 (1 - 2/\pi)^2},
\end{equation}
then combing Eq.~\eqref{eq:teml1} and \eqref{eq:teml2}, we could obtain that
\begin{equation}\label{eq:ovlow2}
 \begin{aligned}
  & \quad \mathrm{tr} \{ (XX^T + \tx(1 - 2/\pi) I)^{-2} X \varSigma_{\delta} X^T \}\\
  & \le \frac{c_2 c_3^2 b (c_1 + 1 - 2/\pi)}{(1 - 2/\pi)^2} \sum_{\lambda_i > \tx/(bn)}  \frac{\alpha_i}{n \lambda_i} + (1 + 324 e \sigma_x^2 /c)^2 \frac{n \sum_{\lambda_i \le \tx / (bn)} \lambda_i \alpha_i }{\tx^2 (1 - 2/\pi)^2}.
\end{aligned}   
\end{equation}
Summarizing the results in Eq.~\eqref{eq:ovup2} and \eqref{eq:ovlow2}, we could obtain that
 \begin{equation}\label{eq:odvarpre3}
  \begin{aligned}
  & \quad \mathrm{tr} \{ \tilde{K}^{-2} X \varSigma_{\delta} X^T \} \\
  &\ge \frac{2 p^2 (1 + o(1))}{ c_2} \sum_{\lambda_i > \tx / (bn)} \frac{\alpha_i}{n \lambda_i} + \frac{2 p^2 (1 + o(1))}{ c_2 b^2}  \frac{n \sum_{\lambda_j \le \tx/ (b n)} \lambda_j \alpha_j}{\tx^2},\\
  & \quad \mathrm{tr} \{ \tilde{K}^{-2} X \varSigma_{\delta} X^T \} \\
  & \le 16p^2 (1 + o(1)) \left( \frac{c_2 c_3^2 b (c_1 + 1 - 2/\pi)}{(1 - 2/\pi)^2} \sum_{\lambda_i > \tx/(bn)}  \frac{\alpha_i}{n \lambda_i} + (1 + 324 e \sigma_x^2 /c)^2 \frac{n \sum_{\lambda_i \le \tx / (bn)} \lambda_i \alpha_i }{\tx^2 (1 - 2/\pi)^2} \right).
  \end{aligned}   
 \end{equation}
Then based on Eq.~\eqref{eq:odvarexpress}, \eqref{eq:odvarpre1}, \eqref{eq:odvarpre2} and \eqref{eq:odvarpre3}, with probability at least $1 - 10 e^{- n^{\xi'} / c}$ where $\xi' = \min\{ 1/2, \xi/2 \}$, we have 
\begin{equation}\label{eq:odvarl}
 \begin{aligned}
  & \quad \mathcal{L}_{\mathrm{ood}}(f_1(\hat{\theta}_1, x)) \ge \max_{\Delta \in \Xi_{\mathrm{ood}}} \ov \\
  & \ge \max_{\Delta \in \Xi_{\mathrm{ood}}} \left\{ \frac{\sigma^2 n \mathrm{tr}\{ \varSigma_{\delta} \}}{ m\tx (c_1 + 1 - 2/\pi)} + \frac{\sigma^2}{8 c_2} \sum_{\lambda_i > \tx / (bn)} \frac{\alpha_i}{n \lambda_i} + \frac{\sigma^2}{8 c_2 b^2}  \frac{n \sum_{\lambda_j \le \tx/ (b n)} \lambda_j \alpha_j}{\tx^2} - o(1) \right\} \\
  & \ge \frac{\sigma^2 \tau p}{ 2 m  (c_1 + 1 - 2/\pi)} +  \frac{\sigma^2}{16 c_2 b^2}  \frac{\tau p \sum_{\lambda_j \le \tx/ (b n)} \lambda_j }{\tx} - o(1),
\end{aligned}   
\end{equation}
where the last inequality is from taking $\alpha_i = \tau \tx / n$ for any $i$ satisfying $\lambda_i \le \tx/(bn)$.

\section{Proof for Theorem~\ref{thm_relu}}\label{pf:relu}

\subsection{Proof for OOD risk upper bound}
Recalling the expression for OOD excess risk, we have
Then we turn to the OOD situation. Due to the definition of OOD excess risk, we could obtain that
\begin{align*}
& \quad \mathcal{L}_{\mathrm{ood}}(f_{W_1}(\hat{\theta}(W_1), x)) \\
&= \mathcal{L}_{\mathrm{id}}(f_{W_1}(\hat{\theta}(W_1), x)) + \max_{\Delta \in \Xi_{\mathrm{ood}}} \underbrace{\frac{1}{m} \theta_1^{*T}[I - \Phi_{W_1}^T (\Phi_{W_1} \Phi_{W_1}^T)^{-1} \Phi_{W_1}] \mathbb{E}_x \nabla_x \phi(W_1^T x)^T \varSigma_{\delta} \nabla_x \phi(W_1^T x) [I - \Phi_{W_1}^T (\Phi_{W_1} \Phi_{W_1}^T)^{-1} \Phi_{W_1}] \theta^*(W_1)}_{\ob} \\
& \quad +\max_{\Delta \in \Xi_{\mathrm{ood}}}\underbrace{ (\sigma^2 + o(1)) \frac{1}{m} \text{trace} \{ (\Phi_{W_1} \Phi_{W_1}^T)^{-2} \Phi_{W_1} \mathbb{E}_x \nabla_x \phi(W_1^T x)^T \varSigma_{\delta} \nabla_x \phi(W_1^T x) \Phi_{W_1}^T\} }_{\ov},
 \end{align*}
For the bias term $\ob$, similar to the analysis above, with the fact that $\| \Phi_{W_1}^T (\Phi_{W_1} \Phi_{W_1}^T)^{-1} \Phi_{W_1} \|_2 \le 1$, we have
\begin{align*}
& \quad  \theta(W_1)^{*T} [\Phi_{W_1}^T (\Phi_{W_1} \Phi_{W_1}^T)^{-1} \Phi_{W_1}] \mathbb{E}_x \nabla_x \phi(W_1^T x)^T \varSigma_{\delta} \nabla_x \phi(W_1^T x) [2 I - \Phi_{W_1}^T (\Phi_{W_1} \Phi_{W_1}^T)^{-1} \Phi-1] \theta^*(W_1) \ge 0 \\
& \Rightarrow  \ob \le \frac{1}{m} \theta(W_1)^{*T} \mathbb{E}_x \nabla_x \phi(W_1^T x)^T \varSigma_{\delta} \nabla_x \phi(W_1^T x) \theta^*(W_1) = 
\frac{1}{m} \mathbb{E} [\phi(W_1^T x)^T \theta^*(W_1) - \phi(W_1^T (x + \delta))^T \theta^*(W_1) ]^2,
\end{align*}
and consider the assumptions on observer $x$ and perturbation $\delta$, we could further obtain that
\begin{align*}
& \quad \frac{1}{m} \mathbb{E} [\phi(W_1^T x)^T \theta^*(W_1) - \phi(W_1^T (x + \delta))^T \theta^*(W_1) ]^2 \\
&= \mathbb{E} [\frac{1}{\sqrt{m}} \phi(W_1^T x)^T \theta^*(W_1) - g(x) + g(x) - g(x + \delta) + g(x + \delta) - \frac{1}{\sqrt{m}} \phi(W_1^T(x+\delta))^T \theta^*(W_1)]^2 \\
 &\le 4 \mathbb{E}[\frac{1}{\sqrt{m}} \phi(W_1^T x) \theta^*(W_1) - g(x)]^2 + 4 \mathbb{E} [g(x+\delta) - \frac{1}{\sqrt{m}} \phi(W_1^T(x+\delta))^T \theta^*(W_1)]^2 + 4 \mathbb{E}[g(x) - g(x + \delta)]^2 \\
 &= 8 \varrho^2 + 4 \mathbb{E}[\nabla_x g(x)^T \delta]^2 .
 \end{align*}
As we have the facts:
\begin{align*}
 & \varrho = o(1), \quad \mathbb{E}_x \| \nabla_x g(x) \|_2 \le c'', \quad \| \varSigma_{\delta} \|_2 \le \tau ,
\end{align*}
we could obtain that 
\begin{equation}\label{eq:oodbias}
    \ob \le 8 \varrho^2 + 4 \mathbb{E}[\nabla_x g(x)^T \delta]^2 \le 8 \varrho^2 + 4 \tau \mathbb{E} \| \nabla_x g(x) \|_2^2 \le 8 \varrho^2 + 4 c'' \tau  = O(1)
\end{equation}
And for the variance term $\ov$, based on Eq.~\eqref{eq:oodbias}, \eqref{eq:odvarexpress}, \eqref{eq:odvarpre1}, \eqref{eq:odvarpre2} and \eqref{eq:odvarpre3}, with probability at least $1 - 10 e^{- n^{\xi'} / c}$ where $\xi' = \min\{ 1/2, \xi/2 \}$, we have
\begin{equation}\label{eq:odvaru}
 \begin{aligned}
  & \quad \mathcal{L}_{\mathrm{ood}}(f_1(\hat{\theta}_1, x)) \le \max_{\Delta \in \Xi_{\mathrm{ood}}} \left\{ \ob + \ov \right\} \\
  & \le 4 \tau \mathbb{E}\| \nabla_x g(x)\|_2^2 \\
  & \quad + \max_{\Delta \in \Xi_{\mathrm{ood}}} \left\{ \frac{2 \sigma^2 n \mathrm{tr} \{ \varSigma_{\delta} \}}{m\tx(1 - 2/\pi)} + \frac{c_2 c_3^2 b (c_1 + 1 - 2/\pi) \sigma^2}{(1 - 2/\pi)^2} \sum_{\lambda_i > \tx/(bn)}  \frac{\alpha_i}{n \lambda_i} + \frac{\sigma^2 (1 + 324 e \sigma_x^2 /c)^2 }{(1 - 2/\pi)^2 }  \frac{n \sum_{\lambda_i \le \tx / (bn)} \lambda_i \alpha_i }{\tx^2} + o(1) \right\}\\
  &\le 4 \tau \mathbb{E}\| \nabla_x g(x)\|_2^2 + \frac{4 \sigma^2 \tau p}{m\tx(1 - 2/\pi)} + \frac{c_2 c_3^2 b (c_1 + 1 - 2/\pi) \tau \sigma^2}{(1 - 2/\pi)^2} + \frac{\sigma^2 (1 + 324 e \sigma_x^2 /c)^2 }{(1 - 2/\pi)^2 }  \frac{2 \tau p \sum_{\lambda_i \le \tx / (bn)} \lambda_i}{\tx} + o(1),
  \end{aligned}   
\end{equation}
where the last inequality is from Assumption~\ref{ass:ood1}, \ref{ass:ood2} and the fact:
\begin{equation*}
\mathrm{tr} \{ \varSigma_{\delta} \} = \sum_i  \alpha_i \le \tau k^* + \tau \frac{p\tx}{n} \le \tau \lambda_1 n + 
\tau \frac{p\tx}{n} \le 2 \tau \frac{p\tx}{n}.
\end{equation*}
\subsection{Proof for ensemble model}

If we turn to the ensemble model, we could obtain that
\begin{align*}
& \quad \mathcal{L}_{\mathrm{ood}}(f_{\mathrm{ens}}(\hat{\theta}(W_1), \hat{\theta}(W_2),x))\\
& = \max_{\Delta \in \Xi_{\mathrm{ood}}}  \mathbb{E}_{x,y, \delta} \left[\frac{1}{2 \sqrt{m}} \left( \phi( (x + \delta)^T W_1) ( \hat{\theta}(W_1) -  \theta^*(W_1)) + \phi( (x + \delta)^T W_2) ( \hat{\theta}(W_2) -  \theta^*(W_2)) \right) \right]^2 \\
& = \max_{\Delta \in \Xi_{\mathrm{ood}}} \underbrace{ \mathbb{E}_{x, y, \delta} \left[ \frac{1}{2 \sqrt{m}} \phi(x^T W_1) (\hat{\theta}(W_1) - \theta^*(W_1)) + \frac{1}{2} \phi(x^T W_2) (\hat{\theta}(W_2) - \theta^*(W_2)) \right]^2 }_{\text{term} 1}\\
& \quad + \max_{\Delta \in \Xi_{\mathrm{ood}}} \underbrace{ \mathbb{E}_{x, y, \delta} \left[ \frac{1}{2 \sqrt{m}} (\delta^T\nabla_x \phi(W_1^T x) (\hat{\theta}(W_1) - \theta^*(W_1)) + \delta^T\nabla_x \phi(W_2^T x) (\hat{\theta}(W_2) - \theta^*(W_2))) \right]^2}_{\text{term} 2},
\end{align*}
similar to the analysis on single model, the equality is from the independence of $x$ and $\delta$. And for term 1, we have
\begin{equation*}
\begin{aligned}
\text{term 1} & = \mathbb{E}_{x, y} \left[ \frac{1}{2 \sqrt{m}} \phi(x^T W_1) (\hat{\theta}(W_1) - \theta^*(W_1)) + \frac{1}{2} \phi(x^T W_2) (\hat{\theta}(W_2) - \theta^*(W_2)) \right]^2 \\
&\le \frac{1}{2 m} \mathbb{E}_{x, y} \left[ \phi(x^T W_1) (\hat{\theta}(W_1) - \theta^*(W_1) ) \right]^2 + \frac{1}{2m} \mathbb{E}_{x, y} \left[ \phi(x^T W_2) (\hat{\theta}(W_2) - \theta^*(W_2) ) \right]^2\\
& = \frac{1}{2} \left( \mathcal{L}_{\mathrm{id}} (f_{W_1}(\hat{\theta}(W_1), x)) + \mathcal{L}_{\mathrm{id}} (f_{W_2}(\hat{\theta}(W_2), x)) \right) \to 0,    
\end{aligned}
\end{equation*}
for term 2, we could express it as
\begin{align*}
\text{term 2} &= \mathbb{E}_{x, y, \delta} \left[ \frac{1}{2 \sqrt{m}} (\delta^T\nabla_x \phi(W_1^T x) (\hat{\theta}(W_1) - \theta^*(W_1)) + \delta^T\nabla_x \phi(W_2^T x) (\hat{\theta}(W_2) - \theta^*(W_2))) \right]^2 \\
&= \underbrace{ \frac{1}{4} \mathbb{E}_{\delta, x} \left[ \delta^T \nabla_x \phi(W_1^T x) [I - \Phi_{W_1}^T (\Phi_{W_1} \Phi_{W_1}^T)^{-1} \Phi_{W_1}] \theta^*(W_1) + \delta^T \nabla_x \phi(W_2^T x) [I - \Phi_{W_2}^T (\Phi_{W_2} \Phi_{W_2}^T)^{-1} \Phi_{W_2}] \theta^*(W_2) \right]^2 }_{\text{term 2.1}} \\
& \quad + \underbrace{ (\sigma^2 + o(1)) \frac{1}{4m} \text{trace} \{ (\Phi_{W_1} \Phi_{W_1}^T)^{-2} \Phi_{W_1} \mathbb{E}_x \nabla_x \phi(W_1^T x)^T \varSigma_{\delta} \nabla_x \phi(W_1^T x) \Phi_{W_1}^T\}}_{\text{term 2.2}} \\
& \quad + \underbrace{ (\sigma^2 + o(1)) \frac{1}{4m}  \text{trace} \{ (\Phi_{W_2} \Phi_{W_2}^T)^{-2} \Phi_{W_2} \mathbb{E}_x \nabla_x \phi(W_2^T x)^T \varSigma_{\delta} \nabla_x \phi(W_2^T x) \Phi_{W_2}^T\}}_{\text{term 2.3}} \\
& \quad + \underbrace{ (\sigma^2 + o(1)) \frac{1}{2m}  \text{trace} \{ (\Phi_{W_1} \Phi_{W_1}^T)^{-1} (\Phi_{W_2} \Phi_{W_2}^T)^{-1} \Phi_{W_1} \mathbb{E}_x \nabla_x \phi(W_1^T x)^T \varSigma_{\delta} \nabla_x \phi(W_2^T x) \Phi_{W_2}^T\}}_{\text{term 2.4}}.
\end{align*}
according to Eq.~\eqref{eq:oodbias} and \eqref{eq:odvarexpress}, with probability at least $1 - 8 e^{- \sqrt{n}}$, we could obtain that
\begin{equation}\label{eq:ensembledecom}
\begin{aligned}
 \text{term 2.1} &\le \frac{1}{2} \left(\frac{1}{m} \mathbb{E} [\phi(W_1^T x)^T \theta^*(W_1) - \phi(W_1^T (x + \delta))^T \theta^*(W_1) ]^2 + \frac{1}{m} \mathbb{E} [\phi(W_2^T x)^T \theta^*(W_2) - \phi(W_2^T (x + \delta))^T \theta^*(W_2) ]^2\right)\\
 &\le \frac{1}{2} \left( \ob(f_1) + \ob(f_2) \right)  \le 4 \mathbb{E}[\nabla_x g(x)^T \delta]^2 + o(1) \le 4 \tau \mathbb{E}\| \nabla_x g(x)\|_2^2 + o(1), \\
 \text{term 2.2} + \text{term 2.3} &= \frac{\sigma^2 + o(1) }{4 pm} \text{trace}( \varSigma_{\delta} ) \text{trace} \{ K^{-1} \} + \frac{\sigma^2 + o(1) }{2} \frac{m-1}{16 m p^2} \text{trace} \{ K^{-2} X \varSigma_{\delta} X^T \} \pm \zeta.     
\end{aligned}
\end{equation}
For the last term, similar to the procedure before, we denote $D' = \frac{1}{m}\Phi_{W_1} \mathbb{E}_x \nabla_x \phi(W_1^T x)^T \varSigma_{\delta} \nabla_x \phi(W_2^T x) \Phi_{W_2}^T \in \mathbb{R}^{n \times n}$. Taking expectation with respect to $x$, we have
\begin{align*}
& \left( \mathbb{E}_x \nabla_x \phi(W_2^T x)^T \varSigma_{\delta} \nabla_x \phi(W_1^T x) \right)_{i,j} = \mathbb{E}_x \left(\frac{\partial \phi(w_{1,i}^T x)}{\partial x} \right)^T \varSigma_{\delta} \frac{\partial \phi(w_{2,j}^T x)}{\partial x} = \frac{1}{2 \pi} \arccos \left( - \frac{w_{1,i}^T w_{2,j}}{\| w_{1,i} \|_2 \| w_{2,j} \|_2} \right) w_{1,i}^T \varSigma_{\delta} w_{1,j},
\end{align*}
furthermore, according to Lemma \ref{lem:feature_ortho}, with probability at least $1 - 2 e^{- n^{\sqrt{\xi}} / 4}$, for any $i,j$, we have
\begin{equation*}
\left( \mathbb{E}_x \nabla_x \phi(W_1^T x)^T \varSigma_{\delta} \nabla_x \phi(W_2^T x) \right)_{i,j} = \frac{1}{2 \pi} w_{1,i}^T \varSigma_{\delta} w_{2,j} \left( \frac{\pi}{2} + O(\frac{1}{n^{(2 + \xi) / 4}}) \right),
\end{equation*}
the equality is from Lemma \ref{lem:feature_ortho} and the fact that function $\arccos(-z)$ has a constant Lipschitz bound around $0$, and
\begin{equation*}
    \left| \frac{w_{1,i}^T w_{2,j}}{\| w_{1,i} \|_2 \| w_{2,j} \|_2} \right| \le \frac{O(n^{ - (2 + \xi)/4} )}{1 - O(n^{- (2 + \xi)/4})} = O(\frac{1}{n^{(2 + \xi)/4}}).
\end{equation*}
Then the components of $D'$ could be expressed as:
\begin{equation}\label{eq:d'}
 \begin{aligned}
 D'_{s,t} &= \frac{1}{m^2} \sum_{i,j} \left( \mathbb{E}_x \nabla_x \phi(W_1^T x)^T \varSigma_{\delta} \nabla_x \phi(W_2^T x) \right)_{i,j} \phi(w_{1,i}^T x_s) \phi(w_{2,j}^T x_t)\\
&= \frac{1}{2 \pi m^2} \sum_{i, j} w_{1,i}^T \varSigma_{\delta} w_{2,j} \phi(w_{1,i}^T x_s) \phi(w_{2,j}^T x_t) \left( \frac{\pi}{2} + O(\frac{1}{n^{(2 + \xi) / 4}}) \right)\\
&= \left( 1 + O(\frac{1}{\sqrt{m}}) \right)  \mathbb{E}_{w, w' \sim \mathcal{N}(0, 1/p I_p)} \frac{1}{2 \pi} w^T \varSigma_{\delta} w' \phi(w^T x_s) \phi(w^{'T} x_t) \left( \frac{\pi}{2} + O(\frac{1}{n^{(2 + \xi)/4}}) \right) \\
&= \left(1 + O(\frac{1}{\sqrt{m}})) \right) \left( 1 + O(\frac{1}{n^{(2 + \xi)/4}}) \right) \frac{1}{16 p^2} x_s^T \varSigma_{\delta} x_t,    
 \end{aligned}   
\end{equation}
so we could repeat the analysis for matrix $D$, and further obtain that
\begin{equation}\label{eq:dgape}
     \zeta':= \left| \mathrm{tr} \left\{ K^{-2} \left( D' - \frac{1}{16p^2} X\varSigma_{\delta} X^T \right) \right\} \right| \le  O(\frac{1}{n^{\xi/4}}) = o(1),
\end{equation}
it implies that
\begin{equation}\label{eq:ensembledecom2}
    \text{term 2.4} = \frac{\sigma^2 + o(1) }{2} \frac{1}{16 p^2} \text{trace} \{ K^{-2} X \varSigma_{\delta} X^T \} \pm \zeta'.
\end{equation}
Combing Eq.~\eqref{eq:odvarpre1} \eqref{eq:odvarpre2} \eqref{eq:odvarpre3}    \eqref{eq:ensembledecom} and \eqref{eq:ensembledecom2}, with probability at least $1 - 20 e^{- n^{\xi'} / 4}$, where $\xi' = \min\{ 1/2, \xi/2 \}$, the improvement of ensemble model could be approximated as 
\begin{align*}
 & \quad \mathcal{L}_{\mathrm{ood}} (f_{W_1}(\hat{\theta}(W_1), x)) + \mathcal{L}_{\mathrm{ood}} (f_{W_2}(\hat{\theta}(W_2), x)))/2 - \mathcal{L}_{\mathrm{ood}} (f_{\mathrm{ens}} (\hat{\theta}(W_1), \hat{\theta}(W_2), x) \\
& \ge \max_{\Delta \in \Xi_{\mathrm{ood}}}  \frac{\sigma^2 n \mathrm{tr}\{ \varSigma_{\delta} \}}{2 m\tx (c_1 + 1 - 2/\pi)} - o(1) \ge \frac{\sigma^2 \tau p}{4 m\tx (c_1 + 1 - 2/\pi)} - o(1),
 \end{align*}
and combine it with Eq.~\eqref{eq:odvaru}, we could finish the proof.

\section{Auxiliary Lemmas}\label{lem}

\begin{lemma}[Refinement of Theorem 2.1 in \citealp{el2010spectrum}]\label{lem:spectrum}
Let we assume that we observe n i.i.d.~random vectors, $x_i \in \rR^p$. Let us consider the kernel matrix $K$ with entries
\begin{equation*}
    K_{i,j} = f(\frac{x_i^T x_j}{\tx}).
\end{equation*}
We assume that:
\begin{enumerate}
\item  $n,p, \tx$ satisfy Assumption~\ref{cond:benign} and \ref{cond:high-dim};
\item  $\varSigma$ is a positive-define $p \times p$ matrix, and $\| \varSigma \|_2$ remains bounded;
\item $x_i = \varSigma^{1/2} \eta_i$, in which $\eta_i, i = 1, \dots, n$ are $\sigma$-subgaussian i.i.d.~random vectors with $\mathbb{E} \eta_i = 0$ and $\mathbb{E} \eta_i \eta_i^T = I_p$;
\item $f$ is a $C^1$ function in a neighborhood of $\tau = \lim_{p \to \infty} \text{trace}(\varSigma) / \tx $ and a $C^3$ function in a neighborhood of $0$.
\end{enumerate}
Under these assumptions, the kernel matrix $K$ can in probability be approximated consistently in operator norm, when $p$ and $n$ tend to $\infty$, by the kernel $\tilde{k}$, where
\begin{align*}
& \tilde{K} = \left( f(0) + f''(0) \frac{\text{trace}(\varSigma^2)}{2 \tx^2} \right) 1 1^T + f'(0) \frac{XX^T}{\tx} + v_p I_n,\\
& v_p = f(1) - f(0) - f'(0).
\end{align*}
In other words, with probability at least $1 - 4 n^2 e^{- n^{1/8} / (2 \tau)}$,
\begin{equation*}
    \| K - \tilde{K} \|_2 \le o(n^{- 1 / 16}). 
\end{equation*}
\end{lemma}

\begin{proof}
The proof is quite similar to Theorem 2.1 in \citet{el2010spectrum}, and the only difference is we change the bounded $4 + \epsilon$ absolute moment assumption to sub-gaussian assumption on data $x_i$, so obtain a faster convergence rate. 

First, using Taylor expansions, we can rewrite the kernel matrix $K$ sa
\begin{align*}
& f(x_i^T x_j / \tx) = f(0) + f'(0) \frac{x_i^T x_j}{\tx} + \frac{f''(0)}{2} \left( \frac{x_i^T x_j}{\tx} \right)^2 + \frac{f^{(3)}(\xi_{i,j})}{6} \left( \frac{x_i^T x_j}{\tx} \right)^3, i \ne j,\\
& f(\| x_i \|_2^2 / \tx) = f(1) + f'(\xi_{i,i}) \left( \frac{\| x_i \|_2^2}{\tx} - 1 \right), \text{on the diagonal.}
\end{align*}
Then we could deal with these terms separately.

For the second-order off-diagonal term, as the concentration inequality shows that 
\begin{equation}\label{eq:c1}
    \mathbb{P} \left( \max_{i,j} | \frac{x_i^T x_j}{\tx} - \delta_{i,j} \frac{\text{trace}(\varSigma)}{\tx} | \le t \right) \ge 1 - 2 n^2 e^{- \frac{\tx^2 t^2}{2 r_0(\varSigma^2)}},
\end{equation}
with Lemma \ref{lem_sg_se}, we can obtain that
\begin{equation}\label{eq:c2}
    \mathbb{P} \left( \max_{i \ne j} | \frac{(x_i^T x_j)^2}{\tx^2} - \mathbb{E} \frac{(x_i^T x_j)^2}{\tx^2} | \le t \right) \ge 1 - 2 n^2 e^{- \frac{\tx^4 t^2}{2 (162e)^2 r_0(\varSigma^4)}},
\end{equation}
in which 
\begin{equation*}
    \mathbb{E} \left( \frac{x_i^T x_j}{\tx} \right)^2 = \frac{1}{\tx^2} \mathbb{E} [x_i^T x_j x_j^T x_i] = \frac{1}{\tx^2} \mathbb{E} \text{trace} \{ x_j x_j^T x_i x_i^T \} = \frac{\text{trace}(\varSigma^2)}{\tx^2}. 
\end{equation*}
Denoting a new matrix $W$ as
\begin{equation*}
 W_{i,j} = \left\{
 \begin{aligned}
    & \frac{(x_i^T x_j)^2}{\tx^2}, i \ne j,\\
    & 0, i = j,
 \end{aligned}
 \right.
\end{equation*}
then considering $r_0(\varSigma^4) / \tx \le r_0(\varSigma) / \tx = \tau$ is bounded, choosing $t = n^{- 17 / 16}$, under Assumption~\ref{cond:high-dim}, we have $\tx^3 n^{-17/8} \ge n^{21/32}$, so with probability at least $1 - 2 n^2 e^{- \frac{n^{1/8}}{2 (162e)^2 }}$, we have
\begin{equation*}
    \| W - \frac{\text{trace}(\varSigma^2)}{\tx^2} (11^T - I_n) \|_2 \le \| W - \frac{\text{trace}(\varSigma^2)}{\tx^2} (11^T - I_n) \|_F \le \frac{1}{n^{1/16}}.
\end{equation*}

For the third-order off-diagonal term, as is mentioned in Eq.\eqref{eq:c1}, choosing $t = n^{- 1 / 4}$, with probability at least $1 - 2 n^2 e^{- \frac{n^{1/4}}{2 }}$, we have
\begin{equation*}
    \max_{i \ne j} | \frac{x_i^T x_j}{\tx} | \le \frac{1}{n^{1/4}}.
\end{equation*}
Denote the matrix $E$ has entries $E_{i,j} = f^{(3)}(\xi_{i,j}) x_i^T x_j / \tx$ off the diagonal and $0$ on the diagonal, the third-order off-diagonal term can be upper bounded as
\begin{equation*}
    \| E \circ W \|_2 \le \max_{i,j} | E_{i,j} | \| W \|_2 \le o(n^{- 1 / 4}), 
\end{equation*}
where the last inequality is from the bounded norm of $W$.

For the diagonal term, still recalling Eq.\eqref{eq:c1}, while we have
\begin{equation*}
    \max_i | \frac{\| x_i \|_2^2}{\tx} - 1 | \le \frac{1}{n^{1 / 4}}, 
\end{equation*}
with probability at least $1 - 2 n^2 e^{- \frac{n^{1/4}}{2 }}$, we can further get 
\begin{equation*}
    \max_i | f(\frac{\| x_i \|_2^2}{\tx}) - f(1) | \le o(n^{- 1 / 4}), 
\end{equation*}
which implies that 
\begin{equation*}
    \| \text{diag}[f(\frac{\| x_i \|_2^2}{\tx}), i = 1, \dots, n] - f(\tau) I_n \|_2 \le o(n^{- 1 / 4}).
\end{equation*} 
Combing all the results above, we can obtain that 
\begin{equation*}
    \| K - \tilde{K} \|_2 \le o(n^{- 1 / 16}),
\end{equation*}
with probability at least $1 - 4 n^2 e^{- n^{1/8} / 2}$.
\end{proof}

\begin{lemma}\label{lem:feature_ortho}
Assume $w_1, \dots, w_m$ are sampled i.i.d. from $\mathcal{N}(1, 1/p I_p)$, then with probability at least $1 - 2 e^{- n^{\xi/2} / 4}$, we have
\begin{equation*}
    \mathbb{P} \left( \mid w_i^T \varSigma w_j - \mathbb{E}(w_i^T \varSigma w_j) \mid \right) \le  \frac{\tx}{p n^{(2 + \xi )/4}}, \quad \forall i, j = 1, \dots, m.
\end{equation*}
\end{lemma}

\begin{proof}
First, according to the corresponding concentration inequality, for any $i,j = 1, \dots,m$  with probability at least $1 - 2 e^{- p^2 t^2 / (2 \sum_i \lambda_i^2)}$, we have
\begin{equation*}
    \mathbb{P} \left( \mid w_i^T \varSigma w_j - \mathbb{E}(w_i^T \varSigma w_j) \mid  \right) \le t.
\end{equation*}
Then by choosing $t = \frac{\tx}{p n^{(2 + \xi )/4}}$ and considering all of the pairwise data in $(w_i, w_j), i,j = 1, \dots, m$, under Assumption~\ref{cond:benign}, we have
\begin{equation*}
\mathbb{P} \left( \mid w_i^T \varSigma w_j - \mathbb{E}(w_i^T \varSigma w_j) \mid \right) \le \frac{\tx}{p n^{1/4}}, \quad \forall i, j = 1, \dots, m,
\end{equation*}
with probability at least $1 - 2 m^2 e^{- n^{\xi / 2} / 2}$. 
\end{proof}

\begin{lemma}\label{lem:lip_gaussian}
If $x$ is a $\sigma_x$-sub-gaussian random vector with zero mean and $\mathbb{E} xx^T = I$, and function $f: \mathbb{R}^d \to \mathbb{R}$ is L-Lipschitz, the random variable $f(x)$ is still sub-gaussian with parameter $L \sigma_x$. To be specific,
\begin{equation*}
    \mathbb{E} e^{ \lambda f(x)} \le e^{\frac{\lambda^2 L^2 \sigma_x^2}{2}}.
\end{equation*}
\end{lemma}

\begin{lemma}\label{lem:matrix_mean}
Assume $x \in \mathbb{R}^q$ is a $q$-dim sub-gaussian random vector with parameter $\sigma$, and $\mathbb{E}[x] = \mu$. Here are $n$ i.i.d.~samples $x_1, \dots, x_n$, which have the same distribution as $x$, then we can obtain that with probability at least $1 - 4 e^{- \sqrt{n}}$, 
\begin{equation*}
    \| \mathbb{E} xx^T - \frac{1}{n} \sum_{i=1}^n x_i x_i^T \|_2 \le \| \mathbb{E} zz^T \|_2 \max \{ \sqrt{\frac{\text{trace}(\mathbb{E} zz^T)}{n}}, \frac{\text{trace}(\mathbb{E} zz^T)}{n}, \frac{1}{n^{1/4}} \} + 2 \sqrt{2} \frac{\sigma \| \mu \|_2}{n^{1/4}},
\end{equation*}
where we denote $z := x - \mu$.
\end{lemma}

\begin{proof}
First, we denote $z = x - \mu$ is a ramdom vector with zero mean, correspondingly, there are $n$ i.i.d.~samples, $z_1, \dots, z_n$. Then we can obtain that
\begin{equation*}
    \mathbb{E} xx^T = \mathbb{E} (z + \mu) (z + \mu)^T = \mathbb{E} zz^T + \mu \mu^T,
\end{equation*}
and for the samples, 
\begin{equation*}
    \frac{1}{n} \sum_{i=1}^n x_i x_i^T = \frac{1}{n} \sum_{i=1}^n z_i z_i^T + \frac{2}{n} \sum_{i=1}^n \mu z_i^T + \mu \mu^T,
\end{equation*}
which implies that
\begin{align*}
\| \mathbb{E} xx^T - \frac{1}{n} \sum_{i=1}^n x_i x_i^T \|_2  &= \| \mathbb{E} zz^T + \mu \mu^T - \frac{1}{n} \sum_{i=1}^n z_i z_i^T - \mu \mu^T - \frac{2}{n} \sum_{i=1}^n \mu z_i^T \|_2 \\
&=  \| \mathbb{E} zz^T - \frac{1}{n} \sum_{i=1}^n z_i z_i^T - \frac{2}{n} \sum_{i=1}^n \mu z_i^T \|_2 \\
&\le \| \mathbb{E} z z^T - \frac{1}{n} \sum_{i=1}^n z_i z_i^T \|_2 + 2 \| \frac{1}{n} \sum_{i=1}^n \mu z_i^T \|_2\\
&= \| \mathbb{E} z z^T - \frac{1}{n} \sum_{i=1}^n z_i z_i^T \|_2 + 2 | \frac{1}{n} \sum_{i=1}^n \mu^T z_i |,
\end{align*}
where the inequality is from triangular inequality. So we can estimate the two terms respectively.

For the first term, as $z$ is $\sigma$-subgaussian random variable, by Theorem 9 in \citet{koltchinskii2017concentration}, with probability at least $1 - 2 e^{-t}$,
\begin{equation}\label{eq:lem1}
    \| \mathbb{E} z z^T - \frac{1}{n} \sum_{i=1}^n z_i z_i^T \|_2 \le  \| \mathbb{E} zz^T \|_2 \max \{ \sqrt{\frac{\text{trace}(\mathbb{E} zz^T)}{n}}, \frac{\text{trace}(\mathbb{E} zz^T)}{n}, \sqrt{\frac{t}{n}}, \frac{t}{n} \},
\end{equation}
And for the second term, by general concentration inequality, we can obtain that with probability at least $1 - 2 e^{- n t^2 / (2 \sigma^2 \| \mu \|_2^2 )}$,
\begin{equation}\label{eq:lem2}
    | \frac{1}{n} \sum_{i=1}^n z_i^T \mu | \le t.
\end{equation}
Choosing $t = \sqrt{n}$ in Eq.\eqref{eq:lem1} and $t = \sqrt{2} \sigma \| \mu \|_2 n^{- 1/4}$ in Eq.\eqref{eq:lem2}, with probability at least $1 - 4 e^{- \sqrt{n}}$,
\begin{align*}
\| \mathbb{E} xx^T - \frac{1}{n} \sum_{i=1}^n x_i x_i^T \|_2 &\le \| \mathbb{E} zz^T - \frac{1}{n} \sum_{i=1}^n z_i z_i^T \|_2 + 2 \| \frac{1}{n} \sum_{i=1}^n z_i^T \mu \|_2\\
& \le \| \mathbb{E} zz^T \|_2 \max \{ \sqrt{\frac{\text{trace}(\mathbb{E} zz^T)}{n}}, \frac{\text{trace}(\mathbb{E} zz^T)}{n}, \frac{1}{n^{1/4}} \} + 2 \sqrt{2} \frac{\sigma \| \mu \|_2}{n^{1/4}}.
\end{align*} 
\end{proof}

\begin{lemma}\label{lem:matrix_comp}
Consider two positive defined matrix $A, B \in \bm{R}^{n \times n}$ satisfying that $\mu_n(A) > \mu_1(B)$, for any $C \in \bm{R}^{p \times p}$, we could obtain that
\begin{align*}
& \text{trace} \{ A C \} \le \text{trace} \{ (A + B) C \} \le \left( 1 + \frac{\mu_1(B)}{\mu_n(A)} \right) \text{trace} \{ A C \},\\
& \left( 1 - \frac{\mu_1(B)}{\mu_n(A)} \right) \text{trace} \{ A^{-1} C \} \le \text{trace} \{ (A + B)^{-1} C \} \le \text{trace} \{ A^{-1} C \}.
\end{align*}
\end{lemma}

\begin{proof}
The proof procedures are just related to some linear algebra calculations. For the first inequality, it is intuitive that
\begin{equation*}
    \text{trace} \{ A C \} \le \text{trace} \{ (A + B) C \},
\end{equation*}
and considering the right hand side, we have
\begin{align*}
& \text{trace} \{ (A + B) C \} - \text{trace} \{ A C \} = \text{B C} \le \mu_1(B) \text{trace} \{ C \} = \frac{\mu_1(B)}{\mu_n (A)} \mu_n(A) \text{trace} \{ C \} \le \frac{\mu_1(B)}{\mu_n (A)} \text{trace} \{ A C \}.
\end{align*}
And similarly, to prove the second inequality, what we need to prove is just
\begin{equation*}
    \text{trace} \{ A^{-1} C \} - \text{trace} \{ (A + B)^{-1} C \} \le \frac{\mu_1(B)}{\mu_1(B) + \mu_n(A)} \text{trace} \{ A^{-1} C \},
\end{equation*}
it could be verified by
\begin{align*}
\text{trace} \{ A^{-1} C \} - \text{trace} \{ (A + B)^{-1} C \}
&= \text{trace} \{ (A + B)^{-1} B A^{-1} C \} \\
& \le \mu_1((A + B)^{-1} B) \text{trace} \{ A^{-1} C \} \\
& \le \frac{\mu_1(B)}{\mu_1(B) + \mu_n(A)} \text{trace} \{ A^{-1} C \}.
\end{align*}

\end{proof}

\begin{lemma}\label{lem:trace_fnorm}
 For any matrix $A, B \in \mathbb{R}^{n \times n}$, we hava
 \begin{equation*}
    | \mathrm{tr}\{ AB\} | \le \| A \|_F \| B \|_F.
 \end{equation*}
\end{lemma}

\begin{proof}
From Cauthy-Schwarz inequality, we have
 \begin{align*}
   | \mathrm{tr}\{ AB \} | &= \sum_{i,j = 1}^n a_{i,j} b_{i,j} \le \sqrt{\sum_{i,j = 1}^n a_{i,j}^2} \sqrt{\sum_{i,j=1}^n b_{i,j}^2} = \| A \|_F \| B \|_F.
 \end{align*}   
\end{proof}

\begin{lemma}\label{lem:ood}
For random vector $w \sim \mathcal{N}(0, I_p)$, $a, b \in \mathbb{R}^p$ and semi-positive defined diagonal matrix
\begin{equation*}
  H := \text{diag}[h_1, \dots, h_p] \in \mathbb{R}^{p \times p},  
\end{equation*}
while $\mathrm{tr}\{ H \} / \mu_1(H) \ge n^2$, we have
\begin{align*}
& \quad  \mathbb{E}_w w^T H w a^Tw b^Tw 1(a^Tw \ge 0) 1(b^T w \ge 0) \\
&= \left( \| a \|_2 \| b \|_2 \frac{\mathrm{tr}\{ H \}}{2 \pi} \frac{a^T b}{\| a \|_2 \| b \|_2} + \frac{a^T H b}{\pi}  \right) \arccos\left( - \frac{a^T b^T}{\| a \|_2 \| b \|_2} \right) \\
& \quad + \left( 1 + O(\frac{1}{n^2}) \right) \mathrm{tr}\{ H \}  \frac{ \| a \|_2 \| b \|_2 }{2 \pi} \sqrt{1 - \left( \frac{a^T b}{\| a \|_2 \| b \|_2} \right)^2}.
\end{align*}
\end{lemma}

\begin{proof}
As $H$ is a diagonal matrix we can express the term as
\begin{align*}
  & \quad \mathbb{E}_w w^T H w a^Tw b^Tw 1(a^Tw \ge 0) 1(b^T w \ge 0)\\
  &= \sum_{i=1}^p h_i \mathbb{E} w_i^2 a^T w b^T w 1(a^Tw \ge 0, b^T w \ge 0) \\
  &= \sum_{i=1}^p h_i \mathbb{E} (e_i^T w)^2 a^T w b^T w 1 (a^T w \ge 0, b^T w \ge 0),  
\end{align*}
so we could just focus on each single term $\mathbb{E} (e_i^T w)^2 a^T w b^T w 1 (a^T w \ge 0, b^T w \ge 0)$ firstly. By Gram-Schmidt orthogonalization, we could denote
\begin{align*}
 & y_1 = \frac{a^T w}{\| a \|_2}, \quad y_2 = \frac{\frac{b^T w}{\| b \|_2} - \rho_{ab} y_1}{\sqrt{1 - \rho_{ab}^2}}, \\
 & y_3 = \frac{\frac{ e_i^T w}{\| e_i \|_2} - \rho_{ae} y_1 - \frac{\rho_{be} - \rho_{ab} \rho_{ae}}{\sqrt{1 - \rho_{ab}^2}}y_2}{\sqrt{1 - \rho_{ae}^2 - \frac{(\rho_{be} - \rho_{ab} \rho_{ac})^2}{1 - \rho_{ab}^2}}},
\end{align*}
where 
\begin{equation*}
    \rho_{ab} = \frac{a^T b}{\| a \|_2 \| b \|_2}, \quad 
    \rho_{ae} = \frac{a^T e_i}{\| a \|_2 \| e_i \|_2} = \frac{a_i}{\| a \|_2}, \quad
    \rho_{be} = \frac{b^T e_i}{\| b \|_2 \| e_i \|_2} = \frac{b_i}{\| b \|_2},
\end{equation*}
then $[y_1, y_2, y_3]^T \sim \mathcal{N}(0,I_3)$, and the single term cold be expressed as
\begin{align*}
& \quad \mathbb{E} (e_i^T w)^2 a^T w b^T w 1 (a^T w \ge 0, b^T w \ge 0) \\
&= \| a \|_2 \| b \|_2 \| e_i \|_2 \mathbb{E} \left( \sqrt{1 - \rho_{ae}^2 - \frac{(\rho_{be} - \rho_{ab} \rho_{ac})^2}{1 - \rho_{ab}^2}} y_3 + \frac{\rho_{be} - \rho_{ab} \rho_{ae}}{\sqrt{1 - \rho_{ab}^2}}y_2 + \rho_{ae} y_1 \right) \\
& \quad \quad \quad \quad \quad \quad \quad \quad y_1 \left( \sqrt{1 - \rho_{ab}^2} y_2 + \rho_{ab} y_1 \right)  1(y_1 \ge 0, \rho_{ab}y_1 + \sqrt{1 - \rho_{ab}^2} y_2 \ge 0).
\end{align*}
With detailed calculations, we could further obtain that
\begin{align*}
& \quad \mathbb{E} (e_i^T w)^2 a^T w b^T w 1 (a^T w \ge 0, b^T w \ge 0) \\
&= \frac{1}{2 \pi} \left( \rho_{ab} + 2 \rho_{ae} \rho_{be} \right) \arccos(- \rho_{ab}) + \frac{1}{2 \pi} \sqrt{1 - \rho_{ab}^2} \left( 1 + \rho_{ae}^2 + \rho_{be}^2 \right),
\end{align*}
then sum over all single terms, we could obtain that
\begin{align*}
& \quad \mathbb{E}_w w^T H w a^Tw b^Tw 1(a^Tw \ge 0) 1(b^T w \ge 0)\\
&= \sum_{i=1}^p h_i \mathbb{E} (e_i^T w)^2 a^T w b^T w 1 (a^T w \ge 0, b^T w \ge 0),\\
&= \sum_i h_i \| a \|_2 \| b \|_2 \frac{1}{2 \pi} \left( \frac{a^T b}{\| a \|_2 \| b \|_2} + 2 \frac{a_ib_i}{\| a \|_2 \| b \|_2} \right) \arccos\left( - \frac{a^T b^T}{\| a \|_2 \| b \|_2} \right) \\
& \quad + \sum_i h_i \| a \|_2 \| b \|_2 \frac{1}{2 \pi} \sqrt{1 - \left( \frac{a^T b}{\| a \|_2 \| b \|_2} \right)^2} \left( 1 + \frac{a_i^2}{\| a \|_2^2} + \frac{b_i^2}{\| b \|_2^2} \right)\\
&= \left( \| a \|_2 \| b \|_2 \frac{\mathrm{tr}\{ H \}}{2 \pi} \frac{a^T b}{\| a \|_2 \| b \|_2} + \frac{a^T H b}{\pi}  \right) \arccos\left( - \frac{a^T b^T}{\| a \|_2 \| b \|_2} \right) \\
& \quad + \left( \mathrm{tr}\{ H \} + \frac{a^T H a}{\| a \|_2^2} + \frac{b^T H b}{\| b \|_2^2} \right)  \frac{ \| a \|_2 \| b \|_2 }{2 \pi} \sqrt{1 - \left( \frac{a^T b}{\| a \|_2 \| b \|_2} \right)^2}, 
\end{align*}
as $\mathrm{tr}\{ H \} / \mu_1(H) \ge n^2$ and 
\begin{equation*}
    \frac{a^T H a}{\| a \|_2^2} \le \mu_1(H), \quad \frac{b^T H b}{\| b \|_2^2} \le \mu_1(H),
\end{equation*}
we could further get that
\begin{align*}
& \quad \mathbb{E}_w w^T H w a^Tw b^Tw 1(a^Tw \ge 0) 1(b^T w \ge 0)\\
&= \left( \| a \|_2 \| b \|_2 \frac{\mathrm{tr}\{ H \}}{2 \pi} \frac{a^T b}{\| a \|_2 \| b \|_2} + \frac{a^T H b}{\pi}  \right) \arccos\left( - \frac{a^T b^T}{\| a \|_2 \| b \|_2} \right) \\
& \quad + \left( 1 + O(\frac{1}{n^2}) \right) \mathrm{tr}\{ H \}  \frac{ \| a \|_2 \| b \|_2 }{2 \pi} \sqrt{1 - \left( \frac{a^T b}{\| a \|_2 \| b \|_2} \right)^2},
\end{align*}
which finishe the proof.
 
\end{proof}

\section{Technical Lemmas from Prior Works}

\begin{lemma}[Lemma 10 in \citealp{bartlett2020benign}]\label{lem_eigen}
 There are constants $b, c \ge 1$ such that, for any $k \ge 0$, with probability at least $1 - 2 e^{- \frac{n}{c}}$,
\begin{enumerate}
\item for all $i \ge 1$,
\begin{equation*}
    \mu_{k+1}(A_{-i}) \le \mu_{k+1}(A) \le \mu_{1}(A_{k}) \le c_1 (\sum_{j > k} \lambda_j + \lambda_{k+1} n);
\end{equation*}
\item for all $1 \le i \le k$,
\begin{equation*}
    \mu_{n}(A) \ge \mu_{n}(A_{-i}) \ge \mu_{n}(A_{k}) \ge \frac{1}{c_1} \sum_{j > k} \lambda_j - c_1 \lambda_{k+1} n;
\end{equation*}
\item if $r_k \ge bn$, then
\begin{equation*}
    \frac{1}{c_1} \lambda_{k+1} r_k \le \mu_n (A_k) \le \mu_1 (A_k) \le c_1 \lambda_{k+1} r_k ,
\end{equation*}
\end{enumerate}
where $c_1 > 1$ is a constant only depending on $b, \sigma_x$.
\end{lemma}

\begin{lemma}[Corollary 24 in \citealp{bartlett2020benign}]\label{lem_subspacenorm}
 For any centered random vector $\bm{z} \in \mathbb{R}^n$ with independent $\sigma^2_x$ sub-Gaussian coordinates with unit variances, any $k$ dimensional random subspace $\mathscr{L}$ of $\mathbb{R}^n$ that is independent of $\bm{z}$, and any $t > 0$, with probability at least $1 - 3 e^{-t}$,
\begin{equation*}
\begin{aligned}
& \parallel \bm{z} \parallel^2 \le n + 2 (162e)^2 \sigma_x^2(t + \sqrt{nt}),\\
& \parallel \bm{\Pi}_{\mathscr{L}} \bm{z} \parallel^2 \ge n - 2 (162e)^2 \sigma_x^2 (k + t + \sqrt{nt}),
\end{aligned}
\end{equation*}
where $\bm{\Pi}_{\mathscr{L}}$ is the orthogonal projection on $\mathscr{L}$.
\end{lemma}

\begin{lemma}\label{lem_ridgeeigen}
 There are constants $b, c \ge 1$ such that, for any $k \ge 0$, with probability at least $1 - 2 e^{- \frac{n}{c}}$:
 \begin{enumerate}
\item for all $i \ge 1$,
\begin{equation*}
    \mu_{k+1}(A_{-i} + \lambda  I) \le \mu_{k+1}(A + \lambda  I) \le \mu_{1}(A_{k} + \lambda  I) \le c_1 (\sum_{j > k} \lambda_j + \lambda_{k+1} n) + \lambda ;
\end{equation*}
\item for all $1 \le i \le k$,
\begin{equation*}
    \mu_{n}(A + \lambda  I) \ge \mu_{n}(A_{-i} + \lambda  I) \ge \mu_{n}(A_{k} + \lambda  I) \ge \frac{1}{c_1} \sum_{j > k} \lambda_j - c_1 \lambda_{k+1} n + \lambda ;
\end{equation*}
\item if $r_k \ge bn$, then
\begin{equation*}
    \frac{1}{c_1}  \lambda_{k+1} r_k + \lambda \le \mu_n (A_k + \lambda  I) \le \mu_1 (A_k + \lambda  I) \le c_1 \lambda_{k+1} r_k + \lambda.
\end{equation*}   
\end{enumerate}
\end{lemma}

\begin{proof}
With Lemma \ref{lem_eigen}, the first two claims follow immediately. For the third claim: if $r_k \ge bn$, we have that $bn \lambda_{k+1} \le \sum_{j > k} \lambda_j$, so
\begin{equation*}
 \begin{aligned}
& \mu_1 (A_k + \lambda  I) \le c_1 \lambda_{k+1} r_k (\Sigma) + \lambda  \le  \lambda  + c_1 \lambda_{k+1} r_k, \\
& \mu_n (A_k + \lambda  I) \ge \frac{1}{c_1} \lambda_{k+1} r_k  + \lambda  \ge \frac{1}{c_1}  \lambda_{k+1} r_k (\Sigma) + \lambda,
 \end{aligned}   
\end{equation*}
for the same constant $c_1 > 1$ as in Lemma \ref{lem_eigen}.
\end{proof}

\begin{lemma}[Proposition 2.7.1 in \citealp{vershynin2018high}]\label{lem_sg_se}
 For any random variable $\xi$ that is centered, $\sigma^2$-subgaussian, and unit variance, $\xi^2 - 1$ is a centered $162e\sigma^2$-subexponential random variable, that is,
\begin{equation*}
    \mathbb{E}\exp(\lambda (\xi^2 - 1)) \le \exp((162e\lambda \sigma^2)^2), 
\end{equation*}
for all such $\lambda$ that $| \lambda | \le 1 / (162 e \sigma^2)$.
\end{lemma}

\begin{lemma}[Lemma 15 in \citealp{bartlett2020benign}]\label{lem_sum}
Suppose that $\{ \eta_i \}$ is a sequence of non-negative random variables, and that $\{ t_i \}$ is a sequence of non-negative real numbers (at least one of which is strictly positive) such that, for some $\delta \in (0, 1)$ and any $i \ge 1$, $\Pr(\eta_i > t_i) \ge 1 - \delta$. Then,
\begin{equation*}
    \Pr\left(\sum_i \eta_i \ge \frac{1}{2} \sum_i t_i\right) \ge 1 - 2 \delta.
\end{equation*}
\end{lemma}

\begin{lemma}[Lemma 2.7.6 in \citealp{vershynin2018high}]\label{lem_stnorm}
 For any non-increasing sequence $\{ \lambda_i \}_{i=1}^{\infty}$ of non-negative numbers such that $\sum_i \lambda_i < \infty$, and any independent, centered, $\sigma-$subexponential random variables $\{ \xi_i \}_{i=1}^{\infty}$, and any $x > 0$, with probability at least $1 - 2 e^{-x}$
\begin{equation*}
    | \sum_i\lambda_i \xi_i | \le 2 \sigma \max \left( x \lambda_1, \sqrt{x \sum_i \lambda_i^2} \right).
\end{equation*}
\end{lemma}

\begin{lemma}[Theorem 9 in \citet{koltchinskii2017concentration}]\label{lem_eigenx}
Let $z_1, \dots, z_n$ be i.i.d. sub-gaussian random variables with zero mean, then with probability at least $1 - 2 e^{- t}$,
 \begin{equation*}
     \| \mathbb{E} z z^T - \frac{1}{n} \sum_{i=1}^n z_i z_i^T \|_2 \le  \| \mathbb{E} zz^T \|_2 \max \{ \sqrt{\frac{\text{trace}(\mathbb{E} zz^T)}{n}}, \frac{\text{trace}(\mathbb{E} zz^T)}{n}, \sqrt{\frac{t}{n}}, \frac{t}{n} \}.
\end{equation*}
\end{lemma}

\begin{lemma}[Consequence of Theorem 5 in \citet{tsigler2020benign}]\label{lem_bart}
There is an absolute constant $c>1$ such that the following holds.
For any $k < \frac{n}{c}$, with probability at least $1 - c e^{- \frac{n}{c}}$, if $A_k$ is positive definite, then
\begin{equation*}
\begin{aligned}
{\trace \{ \varSigma [I - X^T (XX^T + \lambda n I)^{-1} X]^2 \}}
  & \le
  \left( \sum_{i>k} \lambda_i \right)
  \left( 1 + \frac{\mu_1 (A_k + \lambda n I)^2}{\mu_n (A_k + \lambda n I)^2} +  \frac{n \lambda_{k+1}}{\mu_n (A_k + \lambda n I)} \right)\\
& +
\left( \sum_{i\leq k} \frac1{\lambda_i} \right)
\left( \frac{\mu_1 (A_k + \lambda n I)^2}{n^2} + \frac{\lambda_{k+1}}{n} \cdot \frac{\mu_1 (A_k + \lambda n I)^2}{\mu_n (A_k + \lambda n I)} \right) ,
\\
{ \trace\{ X \varSigma X^T (XX^T + \lambda n I)^{-2} \}}
& \le
\frac{\mu_1(A_k+\lambda nI)^2}{\mu_n(A_k +\lambda nI)^2} \cdot \frac{k}{n}
+ \frac{n}{\mu_n(A_k + \lambda nI)^2}
\left( \sum_{i>k} \lambda_i^2 \right) .
\end{aligned}
\end{equation*}    
\end{lemma}

\begin{lemma}[Proposition 1 in \citealp{jacot2018neural}]\label{lem:ntk1}
For a network of depth L at initialization, with a Lipschitz nonlinearity $\sigma$, and in the limit as $n_1, \dots, n_{L-1} \to \infty$, the output functions $f_{\theta, k}$, for $k = 1, \dots, n_L$, tend (in law) to iid centered Gaussian processes of covariance $\Sigma^{(L)}$ is defined recursively by:
\begin{equation*}
\begin{aligned}
& \Sigma^{(1)} (x,x^{\prime}) = \frac{1}{n_0} x^T x^{\prime} + \beta^2,\\
& \Sigma^{(L+1)} (x,x^{\prime}) = \mathbb{E}_{f \sim N(0,\Sigma^{(L)})} [\sigma (f(x)) \sigma (f(x^{\prime}))] + \beta^2,
\end{aligned}
\end{equation*}
taking the expectation with respect to a centered Gaussian process f of covariance $\Sigma^{(L)}$.    
\end{lemma}

\begin{lemma}[Theorem 1 in \citealp{jacot2018neural}]\label{lem:ntk2}
For a network of depth L at initialization, with a Lipschitz nonlinearity $\sigma$, and in the limit as the layers width $n_1, \dots, n_{L-1} \to \infty$, the NTK $\Theta^{(L)}$ converges in probability to a deterministic limiting kernel:
\begin{equation*}
\Theta^{(L)} \to \Theta_{\infty}^{(L)} \otimes Id_{n_L}.
\end{equation*}
The scalar kernel $\Theta_{\infty}^{(L)}$ : $\mathbb{R}^{n_0} \times \mathbb{R}^{n_0} \to \mathbb{R}$ is defined recursively by
\begin{equation*}
\begin{aligned}
& \Theta_{\infty}^{(1)} (x,x^{\prime}) = \Sigma^{(1)} (x,x^{\prime}),\\
& \Theta_{\infty^{(L+1)} (x,x^{\prime})} = \Theta_{\infty}^{(L)}(x, x^{\prime}) \dot{\Sigma}^{(L+1)}(x,x^{\prime}) + \Sigma^{(L+1)}(x,x^{\prime}),
\end{aligned}
\end{equation*}
where
\begin{equation*}
\dot{\Sigma}^{(L+1)}(x,x^{\prime}) = \mathbb{E}_{f \sim N(0,\Sigma^{(L)})} [\dot{\sigma} (f(x)) \dot{\sigma} (f(x^{\prime}))]
\end{equation*}
taking the expectation with respect to a centered Gaussian process f of covariance $\Sigma^{(L)}$, and where $\dot{\sigma}$ denotes the derivative of $\sigma$.    
\end{lemma}

\end{document}